\documentclass[11pt]{article}
\usepackage[utf8]{inputenc}

\usepackage[margin=1in]{geometry}

\usepackage{microtype}
\usepackage{graphicx}
\usepackage{subfigure}
\usepackage{booktabs} %

\usepackage[colorlinks=true,allcolors=blue,hyperfootnotes=false]{hyperref}

\newenvironment{ack}{\subsection*{Acknowledgements}}

\title{Rate-Optimal Online Convex Optimization \\ in Adaptive Linear Control}

\author{%
Asaf Cassel%
\thanks{Blavatnik School of Computer Science, Tel Aviv University; \texttt{acassel@mail.tau.ac.il}.}
\and
Alon Cohen%
\thanks{School of Electrical Engineering, Tel Aviv University, and Google Research; \texttt{alonco@tauex.tau.ac.il}.}
\and
Tomer Koren%
\thanks{Blavatnik School of Computer Science, Tel Aviv University, and Google Research; \texttt{tkoren@tauex.tau.ac.il}.}
}

\usepackage[normalem]{ulem}

\RequirePackage{mathtools}
\RequirePackage{dsfont}
\RequirePackage{delimset}

\newcommand{\floor}[2][*]{\delim\lfloor\rfloor#1{#2}}
\newcommand{\indEvent}[2][*]{\mathds{1}_{\brk[c]#1{#2}}}

\newcommand{\tr}[2][*]{\mathrm{Tr}\brk*{#2}}

\newcommand{\tran}{^{\mkern-1.5mu\mathsf{T}}}

\newcommand{\EE}[1][]{\mathbb{E}_{#1}}

\newcommand{\RR}[1][]{\mathbb{R}^{#1}}

\DeclareMathOperator*{\argmax}{arg\,max}
\DeclareMathOperator*{\argmin}{arg\,min}

\DeclarePairedDelimiterX\setDef[1]\lbrace\rbrace{#1}

\newcommand{\rootT}{\ensuremath{\smash{\sqrt{T}}}\xspace} %
\usepackage{amsthm}
\usepackage{amssymb}
\usepackage{thmtools}

\declaretheoremstyle[
	    spaceabove=\topsep, 
	    spacebelow=\topsep, 
	    headfont=\normalfont\bfseries,
	    bodyfont=\normalfont\itshape,
	    notefont=\normalfont\bfseries,
	    notebraces={(}{)},
	    postheadspace=0.33em, 
	    headpunct={.},
    ]{theorem}
\declaretheorem[style=theorem]{theorem}

\declaretheoremstyle[
	    spaceabove=\topsep, 
	    spacebelow=\topsep, 
	    headfont=\normalfont\bfseries,
	    bodyfont=\normalfont,
	    notefont=\normalfont\bfseries,
	    notebraces={(}{)},
	    postheadspace=0.33em, 
	    headpunct={.},
    ]{definition}

\declaretheoremstyle[
        spaceabove=\topsep, 
        spacebelow=\topsep, 
        headfont=\normalfont\bfseries,
        bodyfont=\normalfont,
        notefont=\normalfont\bfseries,
        notebraces={}{},
        postheadspace=0.33em, 
        qed=$\blacksquare$, 
        headpunct={.},
    ]{proofstyle}
\declaretheorem[style=proofstyle,numbered=no,name=Proof]{proof}

\declaretheorem[style=theorem,sibling=theorem,name=Lemma]{lemma}

\declaretheorem[style=theorem,sibling=theorem,name=Proposition]{proposition}

\declaretheorem[style=theorem,numbered=no,name=Theorem]{theorem*}
\declaretheorem[style=theorem,numbered=no,name=Lemma]{lemma*}
\declaretheorem[style=theorem,numbered=no,name=Corollary]{corollary*}
\declaretheorem[style=theorem,numbered=no,name=Proposition]{proposition*}
\declaretheorem[style=theorem,numbered=no,name=Claim]{claim*}
\declaretheorem[style=theorem,numbered=no,name=Fact]{fact*}
\declaretheorem[style=theorem,numbered=no,name=Observation]{observation*}
\declaretheorem[style=theorem,numbered=no,name=Conjecture]{conjecture*}

\declaretheorem[style=definition,sibling=theorem,name=Definition]{definition}

\declaretheorem[style=definition,numbered=no,name=Definition]{definition*}
\declaretheorem[style=definition,numbered=no,name=Remark]{remark*}
\declaretheorem[style=definition,numbered=no,name=Example]{example*}
\declaretheorem[style=definition,numbered=no,name=Question]{question*}

\newcommand{\pt}[1][t]{p_{#1}}
\newcommand{\vSign}{\chi}
\newcommand{\kt}[1][t]{k_{#1}}
\newcommand{\gdlr}{\eta_G}
\newcommand{\mwlr}{\eta_M}
\newcommand{\ltofu}[1][t]{\bar{\ell}_{#1}}

\newcommand{\pOCOoptimism}{\alpha}
\newcommand{\at}[1][t]{a_{#1}}
\newcommand{\yt}[1][t]{y_{#1}}
\newcommand{\lt}[1][t]{\ell_{#1}}
\newcommand{\ocoModel}{Q}
\newcommand{\ocoEstModel}{\smash{\widehat{Q}}}
\newcommand{\ocoTrueModel}{Q_{\star}}

\newcommand{\ocoSet}{\mathcal{S}}
\newcommand{\ocoDiam}{R_a}
\newcommand{\din}{d_a}
\newcommand{\dout}{d_y}
\newcommand{\ocoMaxNoise}{W}
\newcommand{\ocoModelDiam}{R_Q}

\newcommand{\ocoR}[1]{R_{#1}}

\newcommand{\ocowt}{\epsilon_t}

\newcommand{\ut}[1][t]{u_{#1}}
\newcommand{\xt}{x_t}
\newcommand{\xtpi}{x_t^{\pi}}
\newcommand{\utpi}{u_t^{\pi}}

\newcommand{\obs}{\rho}
\newcommand{\obsOp}{P}

\newcommand{\pLipF}{G_{f}}
\newcommand{\maxF}{C_{f}}

\newcommand{\tauI}[1][i]{\tau_{#1,1}}
\newcommand{\tauIJ}[1][j]{\tau_{i,#1}}
\newcommand{\pRegTheta}{\lambda_{\model}}
\newcommand{\pRegW}{\lambda_w}
\newcommand{\pOptimism}{\alpha}

\newcommand{\dx}{d_x}
\newcommand{\du}{d_u}
\newcommand{\dmodel}{d_{\model}}

\newcommand{\Astar}{A_{\star}}
\newcommand{\Bstar}{B_{\star}}

\newcommand{\model}{\Psi}

\newcommand{\RxuMax}{R_{\max}}
\newcommand{\nEpochs}{N}

\newcommand{\wErr}{C_w}

\newcommand{\maxNoise}{W}

\newcommand{\RM}{R_{\mathcal{M}}}
\newcommand{\Bbound}{R_B}

\newcommand{\wCov}{\Sigma}

\newcommand{\ww}{{w}}
\newcommand{\wwhat}{\hat{w}}
\newcommand{\wwTilde}{\tilde{w}}

\newcommand{\lseMax}{\bar{\Delta}}

\newcommand{\alg}[1][]{\mathcal{A}}

\newcommand{\taut}[1][t]{\tau_{#1}}

\newcommand{\ftildeti}[1][t]{\tilde{f}_{#1}^{(i)}}
\newcommand{\Mti}[1][t]{M_{#1}^{(i)}}

\newcommand{\lossScale}{C_M}

\newcommand{\pti}[1][t,i]{p_{#1}}
\newcommand{\lti}[1][t,i]{\ell_{#1}} 

\usepackage{enumitem}
\usepackage[numbers]{natbib}
\usepackage{crossreftools}
\usepackage{algorithm}
\usepackage[noend]{algpseudocode}

\algnewcommand{\IfThenElse}[3]{%
  \State \algorithmicif\ #1\ \algorithmicthen\ #2\ \algorithmicelse\ #3}

\usepackage[capitalise]{cleveref}

\usepackage{xspace}

\newif\ifneurips

\begin{document}

\maketitle

\begin{abstract}%
We consider the problem of controlling an unknown linear dynamical system under adversarially changing convex costs and full feedback of both the state and cost function.
We present the first computationally-efficient algorithm that attains an optimal \rootT-regret rate compared to the best stabilizing linear controller in hindsight, 
while avoiding stringent assumptions on the costs such as strong convexity.
Our approach is based on a careful design of non-convex lower confidence bounds for the online costs, and uses a novel technique for computationally-efficient regret minimization of these bounds that leverages their particular non-convex structure.
\end{abstract}

\section{Introduction}

We study a general setting of online adaptive linear control, where a learner attempts to stabilize an initially unknown discrete-time linear dynamical system while minimizing its cumulative cost with respect to an arbitrary sequence of convex loss functions.
The system dynamics evolve according to 
\[
    x_{t+1} = \Astar x_t + \Bstar u_t + w_t,
\]
where $x_t \in \RR[\dx]$, $u_t \in \RR[\du]$ are the (fully observable) system's state and learner's control at time step $t$, and $w_t \in \mathbb{R}^{\dx}$ is the system noise added at step $t$ which is a zero-mean i.i.d.\ Gaussian random variable.
The matrices $\Astar \in \mathbb{R}^{\dx \times \dx}$ and $\Bstar \in \mathbb{R}^{\dx \times \du}$ are the system parameters, which are assumed to be unknown ahead of time and need to be learned adaptively.
The goal is to minimize regret with respect to a sequence of convex loss functions $c_1,\ldots,c_T$ over $T$ time steps, namely, the difference between the learner's cumulative control cost 
$    
    \smash{\sum_{t=1}^{T}} c_t(\xt,\ut)
$
and the best cumulative cost achieved by a control policy from a given set of benchmark policies.

This general framework encapsulates numerous variations of learning in linear control that have been studied extensively in the literature.  When the system parameters are known ahead of time and the costs are fixed and known (convex) quadratics, this amounts to the classical ``planning'' formulation of linear-quadratic (LQ) stochastic control;
see \citep{bertsekas1995dynamic}.  The special case where the costs are fixed and known quadratics but the system parameters are unknown has been addressed much more recently~\citep{abbasi2011regret,cohen2019learning,mania2019certainty}.  This was recently extended to allow for a fixed and known convex cost~\citep{NEURIPS2020_565e8a41} and later for stochastic i.i.d.~costs~\citep{cassel2022efficient}.  On the other hand, the case where the system parameters are known but the quadratic costs are allowed to vary arbitrarily between rounds was first addressed in \citep{cohen2018online}, and has been later extended in various ways to allow for arbitrarily-varying convex costs~\citep{agarwal2019online,agarwal2019logarithmic,simchowitz2020improper,cassel2020bandit}.
In all of these special cases, we now know of efficient algorithms with rate-optimal \rootT regret guarantees.

For the online adaptive linear control problem in its full generality, however, no regret-optimal algorithms are presently known.  The state-of-the-art is due to \citep{simchowitz2020improper} that achieved $\smash{T^{2/3}}$-regret
using a simple explore-then-exploit strategy: in the exploration phase, their algorithm estimates the dynamics parameters by exciting the system with noise; then, in the exploitation phase it runs an online procedure for known dynamics using the estimated transitions.
This simple strategy has also been shown to achieve the optimal \rootT-regret when the online costs are additionally \emph{strongly convex}, demonstrating that the stringent strong convexity assumption allows one to circumvent the challenge of balancing exploration and exploitation in online adaptive linear control.

In this paper, we resolve this gap and give the first rate-optimal algorithm for the general online adaptive linear control problem, accommodating arbitrarily changing general convex (and Lipschitz) costs and unknown system parameters.
Our algorithm is computationally efficient and attains a \rootT regret guarantee with polynomial dependence on the natural parameters of the problem.

\paragraph{Techniques.}

\looseness=-1
Our approach builds upon a combination of recent techniques in online linear control.
First, we rely on the Disturbance Action Policies (DAPs) of \citet{agarwal2019online}:
our algorithm generates DAPs that choose the control at each time step as a linear transformation of past noise terms; 
the DAPs themselves are maintained by online convex optimization algorithms that generate slowly-changing decisions for guaranteeing the stability of the system throughout the learning process.
Moreover, since the dynamics are unknown, our algorithm estimates the noise terms on-the-fly, and uses these estimates in place of the true noise vectors (this is akin to a technique in~\citep{NEURIPS2020_565e8a41}).

Second, following the recent developments of~\citet{cassel2022efficient} for the case of stochastic costs, we perform regret minimization with respect to optimistic lower confidence bounds of the online costs.
However, these confidence bounds turn out to be inherently nonconvex. 
To maintain computational efficiency, we adapt a trick of \citet{dani2008stochastic} (in the context of stochastic linear bandits) for relaxing the nonconvex objectives so as to assume the form of a minimum of a small number of convex objectives; then, we hedge over multiple copies of online gradient descent as ``experts'' in a meta-algorithm, where each copy minimizes regret with respect to one of these convex objectives.

Even so, the decisions of the hedging meta-algorithm are random and can thus change abruptly, interfering with the slowly-moving nature of the DAPs that is crucial for the stability of the system.
We address this issue by using a lazy version of Follow the Perturbed Leader in place of the meta-algorithm (due to \citep{altschuler2018online}) that employs only a small number of switches between experts.  Overall, this results in a computationally efficient scheme that maintains the \rootT regret rate of the individual gradient-based experts.

\paragraph{Related work.}

The problem of adaptive linear-quadratic control has a long history \citep[e.g.,][]{bertsekas1995dynamic}. 
Recent years have seen a renewed interest in this problem through the modern view of regret minimization---building on classic asymptotic results to obtain finite-time guarantees \citep{abbasi2011regret,abeille2018improved,arora2018towards,dean2018regret,faradonbeh2017finite,ibrahimi2012efficient,ouyang2017control}. 
More recently \cite{cohen2019learning,mania2019certainty} provided polynomial-time algorithms obtaining an optimal \rootT regret rate.  The optimality of the \rootT rate was proved concurrently by \cite{cassel2020logarithmic,simchowitz2020naive}.

More recently, \cite{NEURIPS2020_565e8a41} gave an efficient algorithm with \rootT regret for learning the dynamics under a fixed known convex cost. 
\cite{NEURIPS2020_565e8a41} also observed that the problem of learning both dynamics and stochastic convex costs under bandit feedback is reducible
to an instance of stochastic bandit convex optimization for which complex, yet polynomial-time, generic algorithms exist~\citep{agarwal2011stochastic}. 
\citet{cassel2022efficient} later study the problem of learning the dynamics and stochastic convex costs under full-information feedback. Unlike the approach of~\citep{NEURIPS2020_565e8a41}, their algorithm is based on an ``optimism in the face of uncertainty'' principle and is thus conceptually simpler and more efficient to implement.

Our approach relies on the standard assumption that the controller is provided with some initial stabilizing policy.
First proposed in \cite{dean2018regret}, such an assumption yields regret that is polynomial in the problem dimensions, and was later shown to be necessary by \cite{chen2021black}.

Past work has also considered adaptive LQG control, namely linear-quadratic control under partial observability of the state \citep[for example,][]{simchowitz2020improper}. 
However, it turned out that in the stochastic setting, learning the optimal partial-observation linear controller is in a sense easier than learning the full-observation controller. It is in fact possible to obtain $\text{poly}(\log T)$ regret for adaptive LQG \citep{lale2020logarithmic}.
This result is facilitated by simplifying assumptions on both the noise distribution as well as the benchmark policy, assumptions which we do not make in this work.

Most works on regret minimization in adaptive control are model-based; meaning, the algorithm attempts to estimate the model parameters.
Previous literature also considered the alternative approach of model-free control \citep[e.g.,][]{abbasi2019model,cassel2021online,fazel2018global,malik2019derivative,tu2019gap}. These works, however, rely heavily on the assumption of quadratic strongly-convex costs and do not apply to general convex costs.

Lastly, \cite{cassel2020bandit,gradu2020non,NEURIPS2020_565e8a41} consider control under bandit feedback. 
These results are unfortunately impeded by the state-of-the-art in Bandit Convex Optimization, that is either not efficient in practice (namely, high-degree polynomial runtime) or requires further assumptions on the curvature of the cost functions. 
For this reason we focus here on full-information feedback, with the hope that our techniques can be adapted to bandit feedback in subsequent work, contingent on future advancements in BCO.

\section{Preliminaries}

\subsection{Linear control background}

A discrete-time linear control system is one whose dynamics are governed by the following rule:
\[
    x_{t+1} = \Astar x_t + \Bstar u_t + w_t,
\]
where $\Astar \in \mathbb{R}^{\dx \times \dx}$, $\Bstar \in \mathbb{R}^{\dx \times \du}$, and where $w_t \in \mathbb{R}^{\dx}$ is zero-mean i.i.d.
In the planning version of the problem the controller knows $\Astar, \Bstar$ and, at each time $t$, can choose $u_t$ as a function of $x_1,\ldots,x_t$. 
After choosing $u_t$, the controller incurs a known cost $c(x_t, u_t)$.
Classic results pertain to quadratic costs, and state that the control rule that minimizes the steady state cost $J(\pi) = \lim_{T \rightarrow \infty} \mathbb{E}_\pi [\frac{1}{T}\sum_{t=1}^T c(x_t,u_t)]$, chooses $u_t = K x_t$ for some matrix $K \in \mathbb{R}^{\du \times \dx}$.
Moreover, the optimal rule $\pi_\star$ stabilizes the system, implying that $J(\pi_\star)$ is finite and well-defined for any quadratic cost function.

We require the following notion of strong stability \citep{cohen2018online}, which is standard in the literature and whose purpose is to quantify the classic notion of (asymptotic) stability.
\begin{definition}[Strong stability]
    A controller $K$ for the system $(\Astar, \Bstar)$ is $(\kappa, \gamma)-$strongly stable ($\kappa \ge 1$, $0 < \gamma \le 1$) if there exist matrices $Q,L$ such that
    $\Astar +\Bstar K = Q L Q^{-1}$,
    $\norm{L} \le 1 - \gamma$,
    and
    $\norm{K}, \norm{Q}\norm{Q^{-1}} \le \kappa$.
\end{definition}

\subsection{Problem setup}

We address the problem of controlling an unknown linear dynamical system subject to general adversarial convex costs with full state and cost observation. 
In particular, the system parameters $\Astar$, $\Bstar$ are initially unknown and the learner repeatedly interacts with the system as follows:
\begin{enumerate}[label=(\arabic*),leftmargin=*]
    \item The player observes state $x_t$;
    \item The player chooses control $u_t$;
    \item The player observes the cost function $c_t : \RR[\dx] \times \RR[\du] \to \RR$, and incurs cost $c_t(x_t, u_t)$.
\end{enumerate}
Note that $(w_t)_{t=1}^\infty$ are unobserved, and the cost $c_t$ is revealed only after selecting $u_t$.
Our goal is to minimize regret with respect to any policy $\pi$ in a benchmark policy class $\Pi$. %
To that end, denote by $\xtpi, \utpi$ the state and action sequence resulting when following a policy $\pi$; then the regret compared to $\pi$ is defined as
\begin{align*}
    \mathrm{regret_T(\pi)}
    =
    \sum_{t=1}^{T} c_t(\xt,\ut)
    - c_t(\xtpi,\utpi)
    ,
\end{align*}
and we seek to bound this quantity with high probability for a fixed $\pi \in \Pi$.
We focus on the benchmark policy class of strongly stable linear policies that choose $u_t = K x_t$. i.e.,
\begin{align*}
    \Pi_{\mathrm{lin}}
    =
    \brk[c]{
    K \in \RR[\du \times \dx] 
    \; : \;
    \text{$K$ is $(\kappa,\gamma)$-strongly stable}
    }.
\end{align*}

We make the following assumptions on our learning problem:
\begin{itemize}[leftmargin=*]
    \item {\bfseries Non-stochastic convex and Lipschitz costs.} The costs $c_t$ are arbitrarily determined by an oblivious adversary\footnote{An oblivious adversary does not use past random choices of the learner to select its loss functions.} such that each $c_t(x,u)$ is convex in the pair $(x,u)$ and 
    for any $(x, u), (x', u')$ we have
    $
        \abs{c_t(x, u) - c_t(x', u')}
        \le
        \norm{(x - x', u - u')}
        ;
    $\footnote{In \cref{sec:extensions} we also explain how to accommodate quadratic losses via an appropriate choice of a normalizing constant.}
    \item {\bfseries i.i.d.\ Gaussian noise.} $(w_t)_{t=1}^{T}$ is a sequence of i.i.d.\ random variables such that $w_t \sim \mathcal{N}(0, \sigma^2 I)$;
    \item {\bfseries Stabilizable system.} $\Astar$ is $(\kappa,\gamma)-$strongly stable, and $\norm{\Bstar} \le \Bbound$.
\end{itemize}

Note the assumption that $\Astar$ is strongly stable is without loss of generality.
Otherwise, given access to a stabilizing controller $K_0$, we show in \cref{sec:unstable-reduction} a generic black-box reduction that takes any learning algorithm that assumes strongly-stable $\Astar$, augments its observations and adds $K_0 x_t$ to its predicted actions. 
This essentially replaces $\Astar$ with $\Astar + \Bstar K_0$, which is $(\kappa,\gamma)-$strongly stable as desired,
and only incurs a $2\kappa$ multiplicative factor in the regret. 

\subsection{Disturbance Action Policies} \label{sec:dap}

We use the, now standard, class of Disturbance Action Policies (DAPs) first proposed by \cite{agarwal2019online}. This class is parameterized by a sequence of matrices $\brk[c]{M^{[h]} \in \RR[\du \times \dx]}_{h=1}^{H}$. For brevity of notation, these are concatenated into a single matrix $M \in \RR[\du \times H \dx]$ defined as
$
    M
    =
    \brk1{
    M^{[1]}
    \cdots
    M^{[H]}
    }
    .
$
A DAP $\pi_M$ chooses actions
\begin{align*}
    u_t = \sum_{h=1}^{H} M^{[h]} w_{t-h}
    ,
\end{align*}
where recall that the $w_t$ are system disturbances.
Consider the benchmark policy class%
\footnote{We note that a more common definition uses $\sum_{h=1}^{H} \norm{M^{[h]}}$ to measure the size of the class. We chose the Frobenius norm for simplicity of the analysis, but replacing it would not change the results significantly.}
\begin{align*}
    \Pi_{\mathrm{DAP}}
    =
    \brk[c]*{
    \pi_M
    \;:\;
    \norm{M}_F \le \RM
    }
    .
\end{align*}
There are two main reasons for considering the DAP parameterization.
First, the loss functions are convex in $M$, a fact which is generally untrue for $K$ in a linear policy $u_t = K x_t$. This paves the way for tools from the online convex optimization literature.
Second, as shown in \cite[][Lemma 5.2]{agarwal2019online}, if 
$H \in \Omega(\gamma^{-1} \log T)$ 
and
$\RM \in \Omega(\kappa^2 \sqrt{\du / \gamma})$ 
then $\Pi_{\mathrm{DAP}}$ is a good approximation for $\Pi_{\mathrm{lin}}$ in the sense that
a regret guarantee with respect to $\Pi_{\mathrm{DAP}}$ gives the same guarantee with respect to $\Pi_{\mathrm{lin}}$ up to a constant additive factor. 
In light of the above, our regret guarantee will be given with respect to $\Pi_{\mathrm{DAP}}$.

\paragraph{Bounded memory representation.} 
As observed in recent literature, the linear dynamics have an infinitely long memory, i.e., all past actions have some effect on the current state, and as such on the losses. However, due to the stability of $\Astar$, the effective memory of the system, $H$, is essentially a constant. 
To see this, unroll the transition model to get that
\begin{align}
\label{eq:lqrUnrolling}
    x_t
    =
    \Astar^{H} x_{t-H}
    +
    \sum_{i=1}^{H} \brk*{
    \Astar^{i-1}\Bstar \ut[t-i]
    +
    \Astar^{i-1} {w}_{t-i}}
    =
    \Astar^{H} x_{t-H}
    +
    \model_\star \tilde{\obs}_{t-1} + {w}_{t-1}
    ,
\end{align}
where
$
\model_\star
=
\brk[s]{\Astar^{H-1}\Bstar, \ldots, \Astar\Bstar, \Bstar, \Astar^{H-1}, \ldots, \Astar} \in \RR[\dx \times \dmodel]
$
and
$
\tilde{\obs}_t
=
[
u_{t-H}\tran, \ldots, u_t\tran,
$
$
w_{t-H}\tran, \ldots, 
$
$
w_{t-1}\tran
]\tran
\in \RR[\dmodel]
,
$
where 
$
    \dmodel := H \du + (H-1) \dx
    .
$
Now, since $\Astar$ is strongly stable, the term $\Astar^{H} x_{t-H}$ quickly becomes negligible.
Following the notation set by \cite{cassel2022efficient}, this observation is combined with the DAP policy parameterization to define the following bounded memory representations. For an arbitrary sequence of disturbances $w = \brk[c]{w_t}_{t \ge 1}$ define
\begin{alignat}{2}
    &u_t(M; \ww)
    &&
    =
    \textstyle\sum_{h=1}^{H} M^{[h]} w_{t-h};
    \nonumber
    \\
    &\obsOp(M) 
    &&=
    \begin{pmatrix}
        M^{[H]} & M^{[H-1]} & \cdots & M^{[1]} \\
        & M^{[H]} & M^{[H-1]} & \cdots &  M^{[1]}  \\
         &  & \ddots & \ddots & &  \ddots &  \\
        &  & & M^{[H]} & M^{[H-1]} & \cdots &  M^{[1]} \\
          &  & &  & I &  \\
           &  & &  &  & \ddots \\
          &  & &  & & & I
    \end{pmatrix}
    \label{eq:obs-op-def}
    \\
    &\obs_t(M; \ww)
    &&
    =
    \brk{
    u_{t+1-H}(M; \ww)\tran
    ,
    \ldots
    u_{t}(M; \ww)\tran
    ,
    w_{t+1-H}
    ,
    \ldots
    ,
    w_{t-1}
    }\tran
    =
    \obsOp(M) w_{t+1-2H:t-1}
    ;\nonumber
    \\
    &x_t(M; \model, \ww)
    &&
    =
    \model \obs_{t-1}(M; \ww) + w_{t-1}
    .
    \label{eq:trunc-x-def}
\end{alignat}

Notice that $u_t, \obs_t, x_t$ do not depend on the entire sequence $\ww$, but only $w_{t-H:t-1}, w_{t+1-2H:t-1},$ and $w_{t-2H:t-1}$ respectively. Importantly, this means that we can compute these functions with knowledge of only the last (at most) $2H$ disturbances. While this notation does not reveal this fact explicitly, it helps with both brevity and clarity.

\begin{algorithm}[!ht]
	\caption{OCO in Adaptive Linear Control} \label{alg:lqr}
	\begin{algorithmic}[1]
		\State \textbf{input}:
		confidence parameter $\delta$, memory length $H$, optimism parameter $\pOptimism$, regularization parameters $\pRegTheta, \pRegW$, learning rate $\gdlr$, noise bound $\maxNoise$.
		
		\State \textbf{set} $i = 1, \tau = 1, V_1 = \pRegTheta I, M_1 = 0$ and $\hat{w}_t = 0, \wwTilde_t, u_t = 0$ for all $t < 1$
		.
		
		\State \textbf{define} loss scaling function:
		\begin{align*}
		    \lossScale(\model)
            :=
            \sqrt{8} \maxNoise \RM H \norm{\model}_F
            +
            {\pOptimism} \sqrt{2/H}\brk{2 + \RM^{-1}\sqrt{\dx}}
             .
		\end{align*}
		\label{ln:loss-scale}

	    \For{$t = 1,2,\ldots, T$}
	        
	        \State \textbf{play} 
	        $
	            \ut
	            =
	            \sum_{i=1}^{H} M_t^{[h]} \hat{w}_{t-h}.
	        $ \label{ln:choice-ut}
	        \State \textbf{set}
	        $V_{t+1} = V_t + \obs_t \obs_t\tran$
	        for 
	        $\obs_t = \brk{u_{t+1-H}\tran, \ldots, u_{t}\tran, \hat{w}_{t+1-H}\tran, \ldots, \hat{w}_{t-1}\tran}\tran$
	        .
	        \label{ln:obs-def}
	        \State \textbf{observe} $x_{t+1}$ and cost function $c_t$.
	        \State \textbf{calculate}
	        \begin{align*} 
	            (A_t \; B_t)
	            =
	            \argmin_{(A \; B) \in \mathbb{R}^{\dx \times (\dx + \du)}} 
	            \sum_{s=1}^{t} \norm{(A \; B) z_s - x_{s+1}}^2
	            +
	            \pRegW \norm{(A \; B)}_F^2
	            ,
	            \quad 
	            \text{where}
	            \;
	            z_s = \begin{pmatrix}x_s \\ u_s \end{pmatrix}. 
	        \end{align*} \label{ln:est-ab}
	        \State {\bf estimate noise}
	        $
	        \hat{w}_t 
	        =
	        \Pi_{\brk[c]{\norm{w} \le \maxNoise}}\brk[s]{x_{t+1} - A_t x_t - B_t u_t}
	        $.
	        \label{ln:est-noise}
	        \State
                \textbf{sample}
                $\wwTilde_t \sim \mathcal{N}(0, \sigma^2 I_{\dx})$.
                \label{ln:noise-generate}
	        \If{$\det(V_{t+1}) > 2 \det(V_{\tau_i})$} \label{ln:epoch-end}
	            \State \textbf{start new epoch}: $i = i + 1, \tau_i = t+1$.
	            \State \textbf{estimate system parameters}
	            \begin{align*} 
	                \model_{\tau_i}
	                =
	                \argmin_{\model \in \mathbb{R}^{\dx \times \dmodel}}
	                \brk[c]*{
	                    \sum_{s = 1}^{t} \norm{\model \obs_{s} - x_{s+1}}^2 + \pRegTheta \norm{\model}_F^2
	                }
	                .
                \end{align*} \label{ln:unrolled-ls}
                \State \textbf{initialize} $\mathcal{A} =$ new instance of $\textsc{BFPL}_{\delta/6}^\star$
                ,
                and set
                $M_{\tau_i} = \ldots = M_{\tau_i + 2H} = 0$. \label{ln:metaalg-init}
            \ElsIf{$t \ge \tau_i + 2H$}
                \State
                \textbf{define} expert loss functions:
                $
                \;\; \forall \kt[] \in [\dmodel] \times [(2H-1)\dx], \vSign \in \{\pm 1\}
                $
                \begin{align*}
                    \tilde{f}_t(M; \kt[], \vSign)
                    =
                    c_t(x_t(M; \model_{\tau_i}, \wwTilde), u_t(M; \wwTilde))
                    -
                    \pOptimism \sigma \vSign \cdot \brk*{V_{\tau_i}^{-1/2}  \obsOp(M)}_{\kt[]}
                    .
                \end{align*}
                \label{ln:expert-loss}
                \State \textbf{define} loss vector $\tilde{\ell}_t \in \RR[2(2H-1)\dx\dmodel^2]$ 
                s.t.\
                $
                \brk{\tilde{\ell}_t}_{\kt[], \vSign}
                =
                \tilde{f}_t(M_t(\kt[], \vSign); \kt[], \vSign)
                /
                \lossScale(\model_{\tau_i})
                .
                $
                \label{ln:metaalg-loss-vector}
                \State
                \textbf{update} experts: 
                $
                \;\; \forall \kt[] \in [\dmodel] \times [(2H-1)\dx], \vSign \in \{\pm 1\}
                $
                \[
                    M_{t+1}(\kt[], \vSign)
                    =
                    \Pi_{\mathcal{M}}\brk[s]*{
                    M_t(\kt[], \vSign)
                    -
                    \gdlr \nabla_{M} \tilde{f}_t(M_t(\kt[], \vSign); \kt[], \vSign)
                    }.
                \]
                \label{ln:update-experts}
                \State \textbf{update} prediction
                $
                    (\kt[t+1], \vSign_{t+1})
                    =
                    \mathcal{A}(\tilde{\ell}_t)
                $
                and \textbf{set} 
                $
                M_{t+1}
                =
                M_{t+1}(\kt[t+1], \vSign_{t+1})
                $
                \label{ln:update-metaalg}
            \EndIf
	    \EndFor
	\end{algorithmic}
\end{algorithm}

\section{Algorithm and Main Result}

In this section we present our algorithm for regret minimization in linear systems with unknown dynamics and adversarial convex costs; see \cref{alg:lqr}.
We denote by $\Pi_{\mathcal{M}}$ the projection onto $\mathcal{M}$.

The algorithm mediates between least squares estimation of the system dynamics (\cref{ln:est-ab,ln:unrolled-ls}), and optimizing the policy w.r.t.~adversarially-changing cost functions.
For OCO, the algorithm uses a combination of Online Gradient Descent \cite{zinkevich2003online} (\cref{ln:update-experts}) and $\textsc{BFPL}^\star_\delta$ \cite{altschuler2018online} (\cref{ln:metaalg-init})---an
experts algorithm that also guarantees an overall small number of switches with probability at least $1-\delta$. 
Our algorithm uses DAP parameterization (\cref{ln:choice-ut}; see \cref{sec:dap} and notations therein), and feeds the aforementioned online optimization algorithms with lower confidence bounds of the online costs.
See below for further details on the algorithm's operation.

We have the following guarantee for \cref{alg:lqr}. The proof is deferred to \cref{sec:proofOfLqrRegret}.

\begin{theorem}[Simplified version of \cref{thm:lqrRegret-full} in \cref{sec:proofOfLqrRegret}]
\label{thm:lqrRegret}
    Let $\delta \in (0,1)$ and suppose that we run \cref{alg:lqr} with parameters $\RM, \Bbound \ge 1$ and for proper choices of $\maxNoise, H, \pRegW, \pRegTheta, \gdlr, \pOptimism$.
    If 
    $
    T
    \ge
    8
    $
    then for any $\pi \in \Pi_{\mathrm{DAP}}$, with probability at least $1-\delta$,
    \begin{align*}
        \mathrm{regret}_T(\pi)
        \le
        \mathrm{poly}(\kappa, \gamma^{-1}, \sigma, \Bbound, \RM, \dx, \du, \log(T/\delta))
        \sqrt{ T}
        .
    \end{align*}
\end{theorem}

Our algorithm is comprised of multiple components working in tandem. We now give a brief overview of each of the components and how they play together.

\subsection{Prerequisites: system estimation and DAP parameterization}

\emph{Parameter estimation:}
The algorithm proceeds in epochs. At the beginning of each epoch, it estimates the unrolled model via least squares using all past observations (\cref{ln:unrolled-ls}), and the estimate $\model_{\tau_i}$ is then kept fixed throughout the epoch. 
The epoch ends when the determinant of $V_t$ is doubled (\cref{ln:epoch-end}); intuitively, when the confidence of the unrolled model increases substantially.\footnote{Concretely, the volume of the confidence ellipsoid around the unrolled model decreases by a constant factor.} 
Throughout the epoch, the algorithm maintains estimates of the transition noise $(\hat{w}_t)_{t=1}^T$ (\cref{ln:est-ab,ln:est-noise}). 
We observe that these noise estimates are essentially produced for ``free'' and no explicit exploration is needed.

\emph{DAP implementation:}
While the benefits of $\Pi_{\mathrm{DAP}}$ are clear, notice that it cannot be implemented as is since we do not have access to the system disturbances $w_t$ nor can we accurately recover them (due to the uncertainty in the transition model).
Similarly to previous works, our algorithm thus uses estimated disturbances $\hat{w}_t$ to compute its actions. 
At each time step $t$, the algorithm chooses $u_t$ as a linear function of the past $H$ noise estimates, and parameterized by $M_t$ (\cref{ln:choice-ut}).
$M_t$ itself is updated using OCO on surrogate cost functions that are formed as a composition between $c_t(x,u)$ and the bounded memory representations $u_t(M; \wwhat), x_t(M; \model_{\tau_i}, \wwhat)$, implicitly assuming that $M_t$ was kept fixed for the last $H$ time steps. It is therefore crucial that these representations closely reflect the state and action that are actually observed,
hence the OCO procedure has to make sure that the sequence $(M_t)_{t=1}^T$ changes slowly (more on this below).

\emph{Construction of lower confidence bounds:}
The algorithm uses the estimated unrolled model to minimize regret with respect to lower confidence bounds of the form:
\begin{equation}
    c_t(x_t(M; \model_{\tau_i}, \hat{w}), u_t(M; \hat{w}))
    -
    \pOptimism' \cdot \norm{V_{\tau_i}^{-1/2} \obs_{t-1}(M; \hat w)}.
    \label{eq:lcb}
\end{equation}
This lower confidence bound follows immediately by combining the Lipschitzness of $c_t$ and standard self-normalizing concentration bounds \cite{abbasi2011regret}.
In our analysis, we show that it indeed lower bounds $c_t(x_t, u_t)$.
Such lower confidence bounds are used extensively in multi-armed bandit and reinforcement learning literature to efficiently combine exploration and exploitation \cite{auer2002finite,auer2008near}.
Intuitively, their minimization steers the resulting policy towards state-action pairs that either yield low cost, or are insufficiently explored.

\subsection{Key idea: making the algorithm efficient}

The functions in \cref{eq:lcb} are, unfortunately, nonconvex (being a difference of two convex functions), and thus cannot be used in OCO algorithms in their current form.
However, we overcome this by relaxing the functions in \cref{eq:lcb}; we do so in two steps.
First, we move to an expected, amortized notion of optimism. We can do this since since $\wwhat \approx \ww$, which are i.i.d, and thus standard concentration arguments imply that the realized bonus term is close to its conditional expectation, which takes the form:
\begin{equation*}
    \sqrt{\EE[] \norm1{V_{\tau_i}^{-1/2} \obs_{t-1}(M; \ww)}^2}
    =
    \sigma \norm1{V_{\tau_i}^{-1/2} \obsOp(M)}_F.
\end{equation*}
Second, building on a trick from \cite{dani2008stochastic} in the context of linear bandit optimization, we further bound
$
\norm1{V_{\tau_i}^{-1/2} \obsOp(M)}_F
\le
\dmodel
\norm1{V_{\tau_i}^{-1/2} \obsOp(M)}_\infty
$
(where $\norm{\cdot}_\infty$ is the entry-wise matrix infinity norm).
Due to an adaptivity issue (more on this below), we also replace the estimated noises $\wwhat$ in the cost term with random simulated noises $\wwTilde \sim \mathcal{N}(0,\sigma^2 I)$.
After this relaxation, the resulting $\tilde{f_t}$ can be written as a minimum of convex function of the form 
\begin{equation} \label{eq:lcb-break}
    \tilde{f}_t(M; \kt[], \vSign)
    =
    c_t(x_t(M; \model_{\tau_i}, \hat{w}), u_t(M; \hat{w}))
    -
    \pOptimism \sigma \vSign \cdot \brk!{V_{\tau_i}^{-1/2}  \obsOp(M)}_{\kt[]}
    ,
\end{equation}
where $\pOptimism = \dmodel\pOptimism'$. Crucial to this trick is the fact that, unlike \cref{eq:lcb}, the linearized non-convex term is independent of the time index $t$.
This observation yields computationally-efficient regret minimization via a two-tier approach described as follows. 
We run a different copy of Online Gradient Descent \cite{zinkevich2003online} for each value of $\kt[], \vSign$, maintaining a different set of DAP parameters $M_t(\kt[], \vSign)$, and fed with $\tilde{f}_t(\cdot; \kt[], \vSign)$ (\cref{ln:update-experts}). 
On top of the OGD algorithms, we run an experts meta-algorithm to minimize $\tilde{f}_t(M_t(\kt[], \vSign); \kt[], \vSign)$ over $\kt[], \vSign$  (\cref{ln:update-metaalg}), treating the output of each OGD algorithm as an expert.

Observe that having initially taken expectation over the noises yields an exploration bonus term that, for fixed $M$, is fixed throughout each epoch. This makes sure that our OCO algorithms, that are restarted at every epoch, can compare against $M_\star$ (the best in hindsight) with $\kt[]$ and $\vSign$ being fixed at the start of the epoch.

\subsection{Additional challenges} \label{sec:additional-challenges}

\emph{Stabilizing the meta-algorithm:}
Our hedging approach nevertheless comes at a price. 
The choices of the meta-algorithm are inherently random, thus $M_t$ might change abruptly between consecutive rounds (recall that DAP require slowly-changing $M_t$). 
We therefore use a version of Follow the Lazy Leader ($\textsc{BFPL}^\star$; \cite{altschuler2018online}) that guarantees, with high probability, both no-regret and a small number of switches.
The small number of switches in conjunction with the fact that each of the expert algorithms generate slowly-changing decisions, guarantee that $M_t$ itself is slowly-changing overall.

\emph{Mitigating adaptivity in costs:}
Even so, the guarantees of $\textsc{BFPL}^\star$ hold only against \emph{oblivious} adversaries (and this limitation is inherent, as \cite{altschuler2018online} discuss extensively), yet the loss sequence constituting of the functions in \cref{eq:lcb-break} is unfortunately \emph{not} oblivious.
This is because the noise estimate $\hat{w}$ were generated using policies derived from previous choices of $\textsc{BFPL}^\star$.
We overcome this hindrance relying on the fact that the noise vectors are drawn from a known (Gaussian) distribution.
This allows to sample i.i.d.\ copies of the noise vectors $\wwTilde$ (\cref{ln:noise-generate}) that we use in $\wwTilde$ instead of $\hat w$, arriving at the functions defined in~\cref{ln:expert-loss}, and ensuring that $\textsc{BFPL}^\star$ receives obliviously-generated losses.

\section{Analysis} \label{sec:technical-overview}

In this section we give a (nearly) complete proof of \cref{thm:lqrRegret} in a simplified setup, inspired by \cite{NEURIPS2020_565e8a41}, where $\Astar = 0$. The analysis in the general case is significantly more technical and thus deferred from this extended abstract (see \cref{sec:proofOfLqrRegret} for full details).

Suppose that $\Astar = 0$ and thus $x_{t+1} = \Bstar u_t + w_t$, assume that $c_t(x, u) = c_t(x)$, i.e., the costs do not depend on $u$, and aim to minimize the pseudo regret, 
\begin{align*}
    \max_{u : \norm{u} \le R_u}
    \sum_{t=1}^{T} \brk[s]{
    J_t(\Bstar u_t) - J_t(\Bstar u)
    }
    ,
\end{align*}
where
$J_t(x) = \EE[w]c_t(x + w)$, is the expected instantaneous cost, which can be computed from $c_t(x)$ for a known noise distribution.
The resulting problem is an instance of the following variant of online convex optimization, which we now define with clean notation as to avoid confusion with our general setting.

\subsection{Simplified setting: OCO with a Hidden Linear Transform} \label{sec:simplified-setup}
Consider the following setting of online convex optimization. Let $\ocoSet \subseteq \RR[\din]$ be a convex decision set. (We denote by $\Pi_{\ocoSet}$ the projection onto $\ocoSet$.)
At round $t$ the learner:
\begin{enumerate}[label=(\roman*),leftmargin=*]
    \item predicts $\at \in \ocoSet$;
    \item observes cost function $\lt: \RR[\dout] \to \RR[]$ and state $\yt[t+1] = \ocoTrueModel \at + \ocowt$;
    \item incurs cost $\lt(\ocoTrueModel \at)$.
\end{enumerate}

We have that $\ocowt \in \RR[\dout]$ are i.i.d.\ noise terms, $\ocoTrueModel \in \RR[\dout \times \din]$ is an unknown linear transform, and $\yt \in \RR[\dout]$ are noisy observations. 
The cost functions are chosen by an oblivious adversary, and
we consider minimizing the regret, defined as
\begin{align*}
    \text{regret}_T
    =
    \max_{\at[] \in \ocoSet}
    \sum_{t=1}^{T} \brk[s]{
    \lt(\ocoTrueModel \at)
    -
    \lt(\ocoTrueModel \at[])
    }
    .
\end{align*}

\paragraph{Assumptions.}
We make the following assumptions:
\begin{itemize}[leftmargin=*]
    \item $\lt(\cdot)$ are convex and $1-$Lipschitz;
    \item There exist known $\ocoMaxNoise, \ocoModelDiam \ge 0$ such that $\norm{\ocowt} \le \ocoMaxNoise$, and $\norm{\ocoTrueModel} \le \ocoModelDiam$.
    \item 
    For all $\at[] \in \ocoSet$ we have $\norm{a} \le \ocoDiam / 2$.
\end{itemize}

\paragraph{Algorithm.}

\begin{algorithm}[!h]
	\caption{OCO with a hidden linear transform \label{alg:OCO}}
	\begin{algorithmic}[1]
		\State \textbf{input:}
		optimism parameter $\pOCOoptimism$, regularizer $\lambda$, learning rates $\gdlr, \mwlr$
		
		\State \textbf{set:} $V_1 = \lambda I, \ocoEstModel_1 = 0, i = 1, \tau_1 = 1$, and $\at[1](\kt[], \vSign) \in \ocoSet, \pt(\kt[], \vSign) = 1 / 2 \din \;\; \forall k\in[\din], \vSign\in\brk[c]{\pm 1}$.

	    \For{$t = 1,2,\ldots, T$}
	        
	        \State 
	        \textbf{draw} $(\kt, \vSign_t) \sim \pt$, and
	        \textbf{play} $\at = \at(\kt, \vSign_t)$. 
	        \State 
	        \textbf{observe} $\yt[t+1]=\ocoTrueModel a_t + w_t$ and cost function $\lt$, and \textbf{set}
	        $V_{t+1} = V_t + \at \at\tran$.
	        
	        \If{$\det(V_{t+1}) > 2 \det(V_{\tau_i})$}
	            \State 
	            \textbf{start new episode} $i = i+1, \tau_i = t+1$, and 
	            \textbf{set} $\pt[t+1](\kt[], \vSign) = 1 / 2 \din, \at[t+1](\kt[], \vSign) = \at(\kt[], \vSign)$.
	            \State
	            \textbf{estimate} parameters:
	            $
	                \ocoEstModel_{\tau_i}
	                =
	                \argmin_{\ocoModel \in \RR[\dout \times \din]} \sum_{s = 1}^{t} \brk[c]{ \norm{\ocoModel \at[s] - \yt[s+1]}^2 + \lambda \norm{\ocoModel}_F^2 }
	                .
	            $
            \Else
                \State
                \textbf{define} expert loss functions:
                $
                    \ltofu(\at[]; \kt[], \vSign)
                    =
                    \lt(\ocoEstModel_{\tau_i} \at[])
                    -
                    \pOCOoptimism \vSign \cdot \brk{V_{\tau_i}^{-1/2} \at[]}_{\kt[]}
                $.
                \State
                \textbf{update} experts:
                $
                    \at[t+1](\kt[], \vSign)
                    =
                    \Pi_{\ocoSet}\brk[s]*{
                    \at(\kt[], \vSign)
                    -
                    \gdlr \nabla_{\at[]} \ltofu(\at(\kt[], \vSign); \kt[], \vSign)
                    }
                $.
                \Comment{OGD}
                \State \textbf{update} prediction:
                $
                    \pt[t+1](\kt[], \vSign)
                    \propto
                    \pt(\kt[], \vSign) \exp\brk{-\mwlr \ltofu(\at(\kt[], \vSign); \kt[], \vSign)}
                $.
                \Comment{MW}
            \EndIf
	    \EndFor
	\end{algorithmic}
\end{algorithm}

Our algorithm for this simplified setup is detailed in \cref{alg:OCO}. Unlike the full control setting, the adversarial costs here have no memory, thus enable the following simplifications compared to \cref{alg:lqr}. First, we can forgo the DAP parameterization and directly optimize the prediction $\at$. This both removes the need to estimate the disturbances, and simplifies the construction of the lower confidence bound. Moreover, the lack of memory obviates the need to make our predictions change slowly over time, and we replace the $\textsc{BFPL}^\star$ sub-routine with Multiplicative Weights (MW) \cite[see][]{arora2012multiplicative}.

\subsection{Analysis}

The main result of this section bounds the regret of \cref{alg:OCO} with high probability. 

\begin{theorem}
\label{thm:ocoRegret}
    Let $\delta \in (0,1)$ and suppose that we run \cref{alg:OCO} with parameters
    \begin{align*}
        \gdlr = \frac{\ocoDiam}{(2\pOCOoptimism \ocoDiam^{-1} + \ocoModelDiam) \sqrt{T}}
        ,
        \mwlr = \frac{\sqrt{\log(2\din)}}{2 (2\pOCOoptimism + \ocoDiam \ocoModelDiam) \sqrt{T}}
        ,
        \lambda = \ocoDiam^2
        ,
        \pOCOoptimism
        =
        \sqrt{\din} \brk2{
        \ocoMaxNoise \dout\sqrt{8 \log\tfrac{2T}{\delta}}
        +
        \sqrt{2}\ocoDiam \ocoModelDiam}
        .
    \end{align*}
    If 
    $
    T
    \ge 8
    $
    then with probability at least $1-\delta$,
    \begin{align*}
        \text{regret}_T
        \le
        77 \din^{3/2} \brk*{
        \ocoMaxNoise \dout\sqrt{8 \log\tfrac{2T}{\delta}}
        +
        \ocoDiam \ocoModelDiam}
        \sqrt{T \log^2 \frac{4 \din T^2}{\delta}}
        .
    \end{align*}
\end{theorem}
The proof of \cref{thm:ocoRegret} is composed of two main lemmas. Similarly to the control setting, we first define an optimistic loss
\begin{align*}
    \ltofu(\at[])
    =
    \lt(\ocoEstModel_{\tau_{i(t)}} \at[])
    -
    \pOCOoptimism \norm{V_{\tau_{i(t)}}^{-1/2} \at[]}_\infty
    ,
\end{align*}
where $i(t) = \max\brk[c]{i : \tau_i \le t}.$
The following lemma shows that the optimistic loss lower bounds the true loss, and bounds the error between the two (See proof in \cref{sec:ocoAdditionalDetails}).
\begin{lemma}[optimism]
\label{lemma:ocoOptimismBound}
    Suppose that
    $
    \sqrt{\din} \norm{\ocoEstModel_{\tau_{i(t)}} - \ocoTrueModel}_{V_{\tau_{i(t)}}}
    \le
    \pOCOoptimism
    .
    $
    Then for any $a \in \RR[\din]$,
    \begin{align*}
        \ltofu(\at[])
        \le
        \lt(\ocoTrueModel \at[])
        &
        \le
        \ltofu(\at[]) + 2 \pOCOoptimism \sqrt{\at[]\tran V_{\tau_{i(t)}}^{-1} \at[]}
        .
    \end{align*}
\end{lemma}
\begin{proof}
The proof follows standard arguments (see e.g. Lemma 3 in \cite{cassel2022efficient}). 
We first use the Lipschitz assumption to get
\begin{align*}
    \abs{\lt(\ocoTrueModel \at[]) - \lt(\ocoEstModel_{\tau_{i(t)}} \at[])}
    &
    \le
    \norm{(\ocoTrueModel - \ocoEstModel_{\tau_{i(t)}}) \at[]}
    \\
    &
    \le
    \norm{\ocoTrueModel - \ocoEstModel_{\tau_{i(t)}}}_{V_{\tau_{i(t)}}} \norm{V_{\tau_{i(t)}}^{-1/2} \at[]}
    \\
    &\le
    \frac{\pOCOoptimism}{\sqrt{\din}} \norm{V_{\tau_{i(t)}}^{-1/2} \at[]}
    \\
    &\le
    {\pOCOoptimism} \norm{V_{\tau_{i(t)}}^{-1/2} \at[]}_\infty
    ,
\end{align*}
where the second and third transitions also used the estimation error and that $\norm{a} \le \sqrt{\din} \norm{a}_\infty$.
We thus have on one hand,
\begin{align*}
    \lt(\ocoTrueModel \at[])
    \ge
    \lt(\ocoEstModel_{\tau_{i(t)}} \at[])
    -
    {\pOCOoptimism} \norm{V_{\tau_{i(t)}}^{-1/2} \at[]}_\infty
    =
    \ltofu(\at[])
    ,
\end{align*}
and on the other hand we also have
\begin{align*}
    \lt(\ocoTrueModel \at[])
    &
    \le
    \lt(\ocoEstModel_{\tau_{i(t)}} \at[])
    +
    {\pOCOoptimism} \norm{V_{\tau_{i(t)}}^{-1/2} \at[]}_\infty
    =
    \ltofu(\at[])
    +
    2{\pOCOoptimism} \norm{V_{\tau_{i(t)}}^{-1/2} \at[]}_\infty
    \le
    \ltofu(\at[])
    +
    2{\pOCOoptimism} \sqrt{\at[]\tran V_{\tau_{i(t)}}^{-1} \at[]}
    ,
\end{align*}
where the last step also used $\norm{a}_\infty \le \norm{a}$.
\end{proof}
Next, the following result bounds the regret with respect to the optimistic cost functions.
\begin{lemma}
\label{lemma:ocoOptimisticRegret}
    Define
    $
        G_i
        =
        \norm{\ocoEstModel_{\tau_i}} + \pOCOoptimism \lambda^{-1/2}
        \text{ and }
        \bar{G}
        =
        2 \pOCOoptimism \lambda^{-1/2} + \ocoModelDiam
        .
    $
    With probability at least $1 - \delta$, for all epochs $i \ge 1$ simultaneously:
    \begin{align*}
        \sum_{t=\tau_i}^{\tau_{i+1}-1} \brk1{\ltofu(\at) - \ltofu(a)}
        \le
        3 \ocoDiam\brk*{\bar{G} + \bar{G}^{-1} G_i^2}
        \sqrt{T \log \frac{2 \din T^2}{\delta}}
        .
    \end{align*}
\end{lemma}
\begin{proof}
First, fix an epoch $i$ and notice that $\at(\kt[], \vSign)$ are the result of running Online Gradient Descent (OGD) on the functions $\ltofu(\cdot; \kt[], \vSign)$, which are $G_i$ Lipschitz. A classic regret bound for OGD (see \cref{lemma:OGD} in \cref{sec:oco-results}) then gives us that for all $\at[] \in \ocoSet$ and $\tau_i \le s \le T$
\begin{align*}
    \sum_{t = \tau_i}^{s}
    \ltofu(\at(\kt[], \vSign); \kt[], \vSign)
    -
    \ltofu(\at[]; \kt[], \vSign)
    \le
    \frac12 \ocoDiam \brk{\bar{G} + G_i^2 \bar{G}^{-1}} \sqrt{T}
    .
\end{align*}
Next, note that MW is invariant to a constant shift in the loss vectors. Letting $\at[0] \in \ocoSet$ be arbitrary, we have that $\pt$ is updated according to the MW rule with the loss of each expert being
$
\ltofu(\at(\kt[], \vSign); \kt[], \vSign)
-
\lt(\ocoEstModel_{\tau_i} \at[0])
.
$
Using the Lipschitz property of $\lt$, these are bounded as
\begin{align*}
\abs{
\ltofu(\at(\kt[], \vSign); \kt[], \vSign)
-
\lt(\ocoEstModel_{\tau_i} \at[0])
}
\le
\norm{\ocoEstModel_{\tau_i}} \norm{\at(\kt[], \vSign) - \at[0]}
+
\pOCOoptimism \lambda^{-1/2} \norm{\at(\kt[], \vSign)}
\le
G_i \ocoDiam
.
\end{align*}
A standard regret guarantee of MW (\cref{lemma:HedgeWHP} in \cref{sec:oco-results}) thus gives us that with probability at least $1 - \delta$,
\begin{align*}
    \sum_{t = \tau_i}^{s}
    \ltofu(\at(\kt, \vSign_t); \kt, \vSign_t)
    -
    \ltofu(\at(\kt[], \vSign); \kt[], \vSign)
    \le
    \ocoDiam\brk*{\bar{G} + \bar{G}^{-1} G_i^2}
    \sqrt{6 T \log \frac{2 \din T}{\delta}}
    ,
\end{align*}
for all $\kt[] \in \brk[s]{\din}, \vSign \in \brk[c]{-1,1}$, and $\tau_i \le s \le T$.

Now, let $\kt^*(\at[]), \vSign_t^*(\at[])$ be such that
$
\ltofu(\at[])
=
\ltofu(\at[]; \kt^*(\at[]), \vSign_t^*(\at[]))
$
for all $\tau_i \le t < \tau_{i+1}$
.
Importantly, notice that $\kt^*(\at[]), \vSign_t^*(\at[])$ are independent of the time index $t$. This is because the minimum in $\ltofu$ is taken over the optimism term, which is independent of $t$ inside a given epoch. For ease of notation, the following will omit the dependence of $\kt[]^*, \vSign^*$ on $\at[]$, which will be kept as a fixed (arbitrary) comparator. Combining the above, with probability $\geq 1 - \delta$ we have that for all $\at[] \in \ocoSet$:
\begin{align*}
    \tag{$\ltofu(\cdot) \le \ltofu(\cdot; \kt[], \vSign)$}
    \sum_{t=\tau_i}^{\tau_{i+1}-1}
    \ltofu(\at) - \ltofu(\at[])
    &
    \le
    \sum_{t=\tau_i}^{\tau_{i+1}-1}
    \ltofu(\at; \kt, \vSign_t) - \ltofu(\at[]; \kt[]^*, \vSign^*)
    \\
    &
    =
    \sum_{t=\tau_i}^{\tau_{i+1}-1}
    \brk*{
    \ltofu(\at(\kt,\vSign_t); \kt, \vSign_t)
    -
    \ltofu(\at(\kt[]^*, \vSign^*); \kt[]^*, \vSign^*)
    }
    \\
    &
    \quad +
    \sum_{t=\tau_i}^{\tau_{i+1}-1}
    \brk{
    \ltofu(\at(\kt[]^*, \vSign^*); \kt[]^*, \vSign^*) - \ltofu(\at[]; \kt[]^*, \vSign^*)
    }
    \\
    &
    \le
    3 \ocoDiam\brk*{\bar{G} + \bar{G}^{-1} G_i^2}
    \sqrt{T \log \frac{2 \din T}{\delta}}
    .
\end{align*}
Repeating the above with $\delta / T$ and taking a union bound over the epochs (of which there are at most $T$) concludes the proof.
\end{proof}

We are now ready to prove \cref{thm:ocoRegret}. We focus here on the main ideas, deferring some details to \cref{sec:ocoAdditionalDetails}.
\begin{proof}[of \cref{thm:ocoRegret}]
We decompose the regret as
\begin{align*}
    \text{regret}_T(\at[])
    \le
    \underbrace{
    \sum_{t=1}^{T} \lt(\ocoTrueModel \at) - \ltofu(\at)
    }_{\ocoR{1}}
    +
    \underbrace{
    \sum_{t=1}^{T} \ltofu(\at) - \ltofu(\at[])
    }_{\ocoR{2}}
    +
    \underbrace{
    \sum_{t=1}^{T} \ltofu(\at[]) - \lt(\ocoTrueModel \at[])
    }_{\ocoR{3}}
    ,
\end{align*}
and conclude the proof by bounding each term on the following good event. Suppose \cref{lemma:ocoOptimisticRegret} holds for all epochs with $\delta / 2T$, and that 
$
    \sqrt{\din} \norm{\ocoEstModel_{t} - \ocoTrueModel}_{V_{t}}
    \le
    \pOCOoptimism
$
for all $t \le T$, which follows from a standard least squares estimation bound (see \cref{lemma:ocoParameterEst}). Taking a union bound, this event holds with probability at least $1 - \delta$.
We conclude that \cref{lemma:ocoOptimismBound} holds and thus $\ocoR{3} \le 0$.
Moreover, we get that
\begin{align*}
    \ocoR{1}
    &
    \le
    \sum_{i=1}^{\nEpochs} \sum_{t=\tau_i}^{\tau_{i+1}-1} 2 \pOCOoptimism \sqrt{\at\tran V_{\tau_i}^{-1} \at}
    \le
    2 \pOCOoptimism \sum_{t=1}^{T} \sqrt{2 \at\tran V_t^{-1} \at}
    \le
    2 \pOCOoptimism \sqrt{2T \sum_{t=1}^{T}\at\tran V_t^{-1} \at}
    \le
    2 \pOCOoptimism \sqrt{10 T \din \log T}
    ,
\end{align*}
where the second inequality uses Lemma 27 of \cite{cohen2019learning}, which states that for $V_1 \succeq V_2 \succeq 0$ we have $V_1 \preceq V_2(\det(V_1) / \det(V_2))$,
the third is due to Jensen's inequality, and the fourth is a standard algebraic argument (see \cref{lemma:harmonicBound}).

Now, an immediate corollary (see \cref{eq:ocoEstDiam}) of the least square error bound is that
$
    \norm{\ocoEstModel_t}
    \le
    \pOCOoptimism \lambda^{-1/2} + \ocoModelDiam
    .
$
We thus have that $G_i \le \bar{G}$ for all $i \le \nEpochs$.
Next, notice that the number of epochs satisfies $\nEpochs \le 2 \din \log T$ (\cref{lemma:ocoNEpochs}).
We conclude that
\begin{align*}
    \ocoR{2}
    &
    =
    \sum_{i=1}^{\nEpochs} \sum_{t=\tau_i}^{\tau_{i+1}-1} 
    \brk1{\ltofu(\at) - \ltofu(\at[])}
    \\
    &\le
    \tag{\cref{lemma:ocoOptimisticRegret}}
    \sum_{i=1}^{\nEpochs} 
    3 \ocoDiam\brk*{\bar{G} + \bar{G}^{-1} G_i^2}
    \sqrt{T \log \frac{4 \din T^2}{\delta}}
    \\
    &
    \le
    12 \din \brk*{2 \pOCOoptimism + \ocoDiam \ocoModelDiam} \sqrt{T \log^2 \frac{4 \din T^2}{\delta}}
    .
    \qedhere
\end{align*}

\end{proof}

\begin{ack}
This work was partially supported by the Israeli Science Foundation (ISF) grant 2549/19, by the Len Blavatnik and the Blavatnik Family foundation, by the Yandex Initiative in Machine Learning, and by the Israeli VATAT data science scholarship. 
\end{ack}

\bibliography{bibliography}

\clearpage
\appendix

\section{Black-box reduction from unstable system} \label{sec:unstable-reduction}
In this section we show that any algorithm that works under the assumption that $\Astar$ is stable can be turned into one that instead receives as input a controller $K_0$ that stabilizes the system $(\Astar, \Bstar)$, and incurs the same regret up to a factor of $2 \kappa$.
Importantly, our result is not tailored to our specific algorithm and holds for any algorithm and a wide variety of benchmark policy classes.
This will show that our simplifying assumption that $\Astar$ is stable is indeed without loss of generality.
This has previously been considered to be true, but we could not find a formal proof in the literature; for completeness, we provide one here.

\paragraph{Formal setup.}

Formally, let $\alg$ be an online control algorithm such that
\begin{align*}
    u_t = \alg(x_{1:t}, u_{1:t-1}, c_{1:t-1}, \zeta)
    .
\end{align*}
We model $\alg$ as a deterministic function, and assume $\zeta$ to be its input random bits.
Next, let $\Pi$ be any benchmark policy class that any $\pi \in \Pi$ satisfies that
\begin{align*}
    u_t
    =
    \pi(w_{1:t-1}, x_0, \xi, t)
    ,
\end{align*}
i.e., a potentially stochastic (in $\xi$) time dependent policy that makes decisions solely based on past disturbances. Notice that this is almost without loss of generality since, given knowledge of the system, past state and actions can be recovered from the disturbances and used to compute the next action. The limitation of such classes is that they are non-adaptive in the sense that their policies choose the same actions regardless of the underlying system parameters and cost functions. By definition, $\Pi_{\text{DAP}}$ satisfies this assumption.

Instead of assuming that $\Astar$ is $(\kappa,\gamma)-$strongly stable, here we assume that
we are given a controller $K_0 \in \RR[\dx \times \du]$ such that $\Astar + \Bstar K_0$ is $(\kappa,\gamma)-$strongly stable.

\paragraph{The reduction.}

Given $\alg$ and $K_0$, we define a meta-algorithm that at each time $t$:
\begin{enumerate}[label=(\roman*)]
    \item \textbf{calculates}
    $
    \tilde{u}_t = \alg(x_{1:t}, \tilde{u}_{1:t-1}, \tilde{c}_{1:t-1}, \zeta)
    $
    where $\tilde{c}_t(x,u) = c_t(x, u + K_0 x) / 2 \kappa;$
    \item \textbf{plays}
    $
    u_t = K_0 x_t + \tilde{u}_t
    $
    and \textbf{observes} $x_{t+1}, c_{t}$.
\end{enumerate}

The following is our main result for the reduction to stable $\Astar$.
\begin{proposition}
    Suppose that $\alg$ has a regret upper bound of $C_T(\Pi)$ for $\Astar$ stable and benchmark policy class $\Pi$. Then given a stabilizing controller $K_0$, our meta algorithm has regret guarantee of $2 \kappa C_T(\Pi)$ against the benchmark class
    $
    \Pi_{K_0}
    =
    \brk[c]*{
    \pi_{K_0} :
    \pi \in \Pi
    }
    ,
    $
    where
    \begin{align*}
        \pi_{K_0}(x_t, w_{1:t-1}, x_0, \xi, t)
        =
        K_0 x_t + \pi(w_{1:t-1}, x_0, \xi, t)
        .
    \end{align*}
\end{proposition}

\begin{proof}
Notice that
\begin{align*}
    x_{t}
    =
    \Astar x_{t-1} + \Bstar u_{t-1} + w_{t-1}
    =
    (\Astar + \Bstar K_0) x_{t-1} + \Bstar \tilde{u}_{t-1} + w_{t-1}
    .
\end{align*}
This implies that from the perspective of $\alg$ the underlying system is $(\tilde{\Astar}, \Bstar)$ where $\tilde{\Astar} = \Astar + \Bstar K_0$ is $(\kappa,\gamma)-$strongly stable, and the disturbances $w_t$ are unchanged.
Next, note that $\tilde{c}_t$ is convex as it is formed as a composition of $c_t$ with an affine function.
Moreover, the construction of $\tilde{c}_t$ is purely deterministic, thus from the point of view of the algorithm, the loss sequence $\tilde{c}_t$ is oblivious.
Finally, we have that
\begin{align*}
    \abs{\tilde{c_t}(x,u) - \tilde{c}_t(x', u')}^2
    &
    =
    (2\kappa)^{-2} \abs{c_t(x, u + K_0 x) - c_t(x', u' + K_0 x')}^2
    \\
    \tag{$c_t$ Lipschitz}
    &
    \le
    (2\kappa)^{-2} \brk{
    \norm{x- x'}^2
    +
    \norm{(u - u') + K_0 (x-x')}^2
    }
    \\
    &
    \le
    (2\kappa)^{-2} \brk{
    2\norm{u - u'}^2
    +
    (1 + 2 \norm{K_0}^2) \norm{x-x'}^2
    }
    \\
    \tag{$\norm{K_0} \le \kappa$}
    &
    \le
    \norm{x-x'}^2 + \norm{u-u'}^2
    ,
\end{align*}
i.e., $\tilde{c}_t$ are $1$-Lipschitz. Thus, given these conditions, we can use the regret guarantee for $\alg$ to get that with probability at least $1-\delta$,
\begin{align*}
    \sum_{t=1}^{T} \brk[s]*{
    \tilde{c}_t(x_t,\tilde{u}_t) - \tilde{c}_t(x_t^{\pi}(\tilde{\Astar},\Bstar), u_t^{\pi})
    }
    \le
    C_T(\Pi)
    ,
    \quad
    \forall \pi \in \Pi
    ,
\end{align*}
where $x_t^{\pi}(A, B)$ is the state sequence that arises when following policy $\pi \in \Pi$ on the system $(A,B)$. We note that $u_t^\pi$ is the previously defined action sequence that results from following $\pi$, which does not depend on $(A,B)$ due to our assumption on the class $\Pi$.

Now, for any $\pi_{K_0} \in \Pi_{K_0}$ let $u_t^{\pi_{K_0}}$ be the action sequence when following $\pi_{K_0}$ on the system $(\Astar,\Bstar)$. For its underlying policy $\pi \in \Pi$ we thus have
\begin{align*}
    u_t^{\pi_{K_0}}
    &
    =
    \pi_{K_0}(x_t^{\pi_{K_0}}(\Astar,\Bstar), w_{1:t-1}, x_0, \xi, t)
    \\
    &
    =
    K_0 x_t^{\pi_{K_0}}(\Astar,\Bstar) + \pi(w_{1:t-1}, x_0, \xi, t)
    =
    K_0 x_t^{\pi_{K_0}}(\Astar,\Bstar) + u_t^{\pi}
    .
\end{align*}
Next, we prove by induction that 
$
    x_t^{\pi}(\tilde{\Astar}, \Bstar) 
    =
    x_t^{\pi_{K_0}}(\Astar, \Bstar)
$
for all $t \ge 1$. This holds trivially for the initial state $t=1$. Now, assume this is true up to $t-1$, then we have that
\begin{align*}
    x_t^{\pi}(\tilde{\Astar}, \Bstar) 
    &
    =
    \tilde{\Astar} x_{t-1}^{\pi}(\tilde{\Astar}, \Bstar) 
    +
    \Bstar u_{t-1}^\pi
    +
    w_{t-1}
    \\
    \tag{induction hypothesis}
    &
    =
    (\Astar + \Bstar K_0) x_{t-1}^{\pi_{K_0}}(\Astar, \Bstar)
    +
    \Bstar u_{t-1}^\pi
    +
    w_{t-1}
    \\
    &
    =
    \Astar x_{t-1}^{\pi_{K_0}}(\Astar, \Bstar)
    +
    \Bstar (u_{t-1}^\pi + K_0 x_{t-1}^{\pi_{K_0}}(\Astar, \Bstar))
    +
    w_{t-1}
    \\
    &
    =
    \Astar x_{t-1}^{\pi_{K_0}}(\Astar, \Bstar)
    +
    \Bstar u_{t-1}^{\pi_{K_0}}
    +
    w_{t-1}
    \\
    &
    =
    x_t^{\pi_{K_0}}(\Astar, \Bstar)
    ,
\end{align*}
thus proving the inductive claim.
We conclude that under the above event, for any $\pi_{K_0} \in \Pi_{K_0}$ we have that
\begin{align*}
    \sum_{t=1}^{T}
    &
    \brk[s]*{
    c_t(x_t, u_t)
    -
    c_t(x_t^{\pi_{K_0}}(\Astar, \Bstar), u_t^{\pi_{K_0}})
    }
    \\
    &
    =
    \sum_{t=1}^{T} \brk[s]*{
    c_t(x_t, \tilde{u}_t + K_0 x_t)
    -
    c_t(x_t^{\pi}(\tilde{\Astar}, \Bstar), u_t^{\pi} + K_0 x_t^{\pi}(\tilde{\Astar}, \Bstar))
    }
    \\
    &
    =
    2\kappa \sum_{t=1}^{T} \brk[s]*{
    \tilde{c}_t(x_t, \tilde{u}_t)
    -
    \tilde{c}_t(x_t^{\pi}(\tilde{\Astar}, \Bstar), u_t^{\pi})
    }
    \\
    &
    \le
    2\kappa C_T(\Pi)
    ,
\end{align*}
thus concluding the proof.
\end{proof}

\section{Extensions}
\label{sec:extensions}

In this section we elaborate on how to extend our results to Gaussian noise, and to handle quadratic costs.

\paragraph{Gaussian noise.}
In \cref{thm:lqrRegret-full} (\cref{sec:proofOfLqrRegret}) we will analyze a slight modification of \cref{alg:lqr}. Instead of $w_t \sim \mathcal{N}(0, \sigma^2 I)$ we will make a simplifying assumption that $w_t$ is zero-mean, has a known distribution with covariance $\wCov$, and is bounded as $\norm{w_t} \le \maxNoise$. This assumption will also modify the generated noise in \cref{ln:noise-generate} of \cref{alg:lqr}. 

Now, we claim that we can run \cref{alg:lqr} as is and obtain nearly the same guarantees. To that end, we follow the reduction proposed in \cite{cassel2021online}.
First, notice that \cref{alg:lqr} does not require an accurate estimate of $\wCov$. In fact, we can replace $\wCov$ in the lower confidence bound \cref{ln:expert-loss} with any $\hat{\wCov}$ satisfying
\begin{align*}
    \wCov
    \preceq
    \hat{\wCov}
    \preceq
    2\wCov,
\end{align*}
and the regret guarantee would change by at most a factor of 2. To see this, one needs to examine the proof of \cref{lemma:lqrR24} and in particular that of \cref{lemma:lqrOptimism}, and replace $\hat{\wCov}$ with either its lower or upper bounds appropriately.

Now, suppose that we run \cref{alg:lqr} with the noises $w_t, \tilde{w}_t$ replaced with
\begin{align*}
    \bar{w}_t = w_t \indEvent{\norm{\wCov^{-1/2}w_t}^2 \le 5 \dx \log (2T/\delta)}
    \\
    \bar{\tilde{w}}_t = \tilde{w}_t \indEvent{\norm{\wCov^{-1/2}\tilde{w}_t}^2 \le 5 \dx \log (2T/\delta)}
    ,
\end{align*}
and denote the covariance of this truncated noise $\bar{\wCov} = \EE \bar{w} \bar{w}\tran$.
As shown in \cite{cassel2021online}, we have that for $T \ge 12$
\begin{align*}
    \EE \bar{w}_t = \EE \bar{\tilde{w}}_t = 0
    \\
    \max{\norm{\bar{w}_t}, \norm{ \bar{\tilde{w}}_t}} \le \maxNoise
    \\
    \bar{\wCov} \preceq \wCov \preceq 2 \bar{\wCov}
    ,
\end{align*}
where $\maxNoise = \sqrt{5 \dx \norm{\wCov} \log (2T/\delta)}$. They also use standard tail inequalities to show that with probability at least $1 - \delta$, both
$
\bar{w}_t = w_t
$
and
$
\tilde{w}_t
=
\bar{\tilde{w}}_t
.
$
On this event we have that the regret of \cref{alg:lqr} with or without the truncation is the same. Since $\bar{w}_t$ satisfy the assumptions of \cref{thm:lqrRegret-full}, we can use it to bound the regret with respect to $\bar{w}_t$ with probability at least $1-\delta$. Using a union bound, we conclude that the regret of the original algorithm is bounded with probability at least $1-2\delta$ where the parameter $\maxNoise$ is set to $\sqrt{5 \dx \norm{\wCov} \log (2T/\delta)}$.

\paragraph{Quadratic costs.}
We now consider the case in which the cost functions are of the form:
\begin{align*}
    c_t(x,u) = x\tran Q_t x + u\tran R_t u
    ,
    \quad
    \text{where}
    \;\;
    \norm{Q_t}, \norm{R_t} \le 1.
\end{align*}
This could also be replaced with the notion of sub-quadratic Lipschitz costs \cite[see, e.g.,][]{simchowitz2020improper}).
Now, let $\RxuMax$ be an upper bound on $\norm{(x_t, u_t)}$ (see \cref{lemma:technicalParameters}).
Define the loss functions $\tilde{c}_t$ that coincide with $c_t$ on $\norm{(x_t, u_t)} \le \RxuMax$ and outside of the region they are extrapolated such that they are globally convex and $2 \RxuMax$ Lipschitz. By design, there is no difference between running \cref{alg:lqr} on either $c_t / 2 \RxuMax$ and $\tilde{c}_t / 2 \RxuMax$. Since $\tilde{c}_t / 2 \RxuMax$ satisfy the assumptions for \cref{thm:lqrRegret-full}, we get that its regret bound holds for quadratic (up to a $2 \RxuMax$ multiplicative factor).

\section{Proof of \cref{thm:lqrRegret}}
\label{sec:proofOfLqrRegret}

In this section we prove a regret bound for a slight modification of \cref{alg:lqr}. Concretely, we replace the assumption that $w_t \sim \mathcal{N}(0, \sigma^2 I)$ with the assumption that $w_t$ is zero-mean, has a known distribution $\mathcal{W}$ with covariance $\wCov$, and is bounded as $\norm{w_t} \le \maxNoise$. 
This results in the following modifications to \cref{alg:lqr}:
\begin{enumerate}
    \item The lower confidence bound
    $
        \pOptimism \sigma \vSign \cdot \brk*{V_{\tau_i}^{-1/2}  \obsOp(M)}_{\kt[]}
    $
    in \cref{ln:expert-loss} is replaced with
    \begin{align*}
        \pOptimism \vSign \cdot \brk*{V_{\tau_i}^{-1/2}  \obsOp(M) \wCov_{2H-1}}_{\kt[]}
        ,
    \end{align*}
    where $\wCov_{2H-1} = I_{2H-1} \otimes \wCov$, and $\otimes$ is the Kronecker product;
    \item The generated noises $\tilde{w}_t$ are sampled from $\mathcal{W}$.
\end{enumerate}
In \cref{sec:extensions} we explained how this modification is applicable to \cref{alg:lqr} without altering the Gaussian noise assumption.
The following is our main result, which bounds the regret of \cref{alg:lqr} under the above modifications and the bounded noise assumption. 

\begin{theorem}[restatement of \cref{thm:lqrRegret}] \label{thm:lqrRegret-full}
    Let $\delta \in (0,1)$ and suppose that we run \cref{alg:lqr} with parameters $\maxNoise, \RM, \Bbound \ge 1$ and
    \begin{align*}
        H 
        =
        \gamma^{-1} \log T
        ,
        &
        \quad
        \pRegW
        =
        5 \kappa^2 \maxNoise^2 \RM^2 \Bbound^2 H \gamma^{-1}
        ,
        \quad
        \pRegTheta
        =
        2 \maxNoise^2 \RM^2 H^2
        ,
        \quad
        \gdlr
        =
        \RM^2 \pOptimism^{-1} \sqrt{2H / T}
        ,
        \\
        &
        \pOptimism
        =
        21 \maxNoise \RM \Bbound \kappa^2 (\dx + \du)
        \sqrt{H^3 \gamma^{-3}  (\dx^2 \kappa^2 + \du \Bbound^2) \log \frac{24T^2}{\delta}}
        .
    \end{align*}
    If 
    $
    T
    \ge
    8
    $
    then for any $\pi \in \Pi_{\mathrm{DAP}}$, with probability at least $1-\delta$
    \begin{align*}
        \mathrm{regret}_T(\pi)
        \le
       43261 \maxNoise \RM^2 \Bbound^2 \kappa^3 \gamma^{-8} 
     (\dx^2 \kappa^2 + \du \Bbound^2) \log^6 \brk*{ \frac{48T^2}{\delta}}     \sqrt{T \dx (\dx + \du)^3 \log (6\dmodel^2)  }
        .
    \end{align*}
\end{theorem}

\paragraph{Structure.}
We begin with a preliminaries section (\cref{sec:lqrProofPrelims}) that states several results that will be used throughout, and are either technical or adaptated from existing results.
Next, in \cref{sec:lqrRegretDecomposition} we provide the body of the proof, decomposing the regret into logical terms, and stating the bound for each one.
Finally, in \cref{sec:truncationCostProof,sec:optimismCostProof,sec:excessRiskCostProof} we prove the bounds for each term.

\subsection{Preliminaries}
\label{sec:lqrProofPrelims}

\paragraph{Disturbance estimation.}
The success of our algorithm relies on the estimation of the system disturbances, which was first bounded in \cite{NEURIPS2020_565e8a41}.
Here we use the following statement, due to \cite{cassel2022efficient}.
\begin{lemma}
\label{lemma:disturbanceEstimation}
    Suppose that $\pRegW = 5 \kappa^2 \maxNoise^2 \RM^2 \Bbound^2 H \gamma^{-1}$, $H > \log T$, and $T \ge \dx$. With probability at least $1 - \delta$
    \begin{align*}
        \sqrt{\sum_{t=1}^{T} \norm{w_t - \hat{w}_t}^2}
        \le
        \wErr,
        \quad 
        \text{where}
        \quad 
        \wErr 
        = 
        10 \maxNoise \kappa \RM \Bbound \gamma^{-1} \sqrt{H (\dx + \du) (\dx^2 \kappa^2 + \du \Bbound^2) \log \frac{T}{\delta}}
        .
    \end{align*}
\end{lemma}
As noted by \cite{NEURIPS2020_565e8a41}, the quality of the disturbance estimation does not depend on the choices of the algorithm, i.e., we can recover the noise without any need for exploration. This is in stark contrast to the estimation of the system matrices $\Astar, \Bstar$, which requires exploration.

\paragraph{Estimating the unrolled model.}
Here we bound the difference between $\model_t$, the least squares estimate, and the real model $\model_\star$. Notice that some standard manipulations on \cref{eq:lqrUnrolling} yield that
\begin{align*}
    x_t
    =
    \model_\star \obs_{t-1} + {w}_{t-1} + e_{t-1}
    ,
\end{align*}
where
$
{\obs}_{t-1}
=
\brk[s]{u_{t-H}\tran, \ldots, u_{t-1}\tran, \hat{w}_{t-H}\tran, \ldots, \hat{w}_{t-2}\tran}
$
are the observations defined in \cref{ln:obs-def} of \cref{alg:lqr},
and
$
    e_{t-1}
    =
    \Astar^{H} x_{t-H}
    +
    \sum_{h=1}^{H} \Astar^{h-1} ({w}_{t-h} - \hat{w}_{t-h})
$
is a bias term, which, while small, is not negligible.
The following result bounds the least squares error for observations of this form. It takes the least squares estimation error bound of \cite{abbasi2011regret}, and augments it with a sensitivity analysis with respect to the biased observations.

\begin{lemma}[\cite{cassel2022efficient}] \label{lemma:lqrParameterEst}
Let 
$
\Delta_t
=
\model_\star - \model_t
,
$
and suppose that $\norm{\obs_t}^2 \le \pRegTheta, T \ge \dx$.
With probability at least $1 - \delta$, we have for all $1 \le t \le T$
\begin{align*}
    \norm{\Delta_t}_{V_t}^2
    \le
    \tr{\Delta_t\tran V_t \Delta_t}
    \le
    16 \maxNoise^2 \dx^2 \log \brk*{\frac{ T}{\delta}}
    +
    4\pRegTheta \norm{\model_\star}_F^2
    +
    2{\sum_{s=1}^{t-1}\norm{e_s}^2}.
\end{align*}
If we also have that
$
\pRegTheta
=
2 \maxNoise^2 \RM^2 H^2
,
$
and that 
$
\sum_{t=1}^{T} \norm{w_t - \hat{w}_t}^2
\le
\wErr^2
$
(see \cref{lemma:disturbanceEstimation})
then
\begin{align*}
    \norm{\Delta_t}_{V_t}
    \le
    \sqrt{\tr{\Delta_t\tran V_t \Delta_t}}
    \le
    21 \maxNoise \RM \Bbound \kappa^2 H 
    \sqrt{\gamma^{-3} (\dx + \du) (\dx^2 \kappa^2 + \du \Bbound^2) \log \frac{T}{\delta}}
    ,
\end{align*}
and
$
    \norm{\brk{\model_{t} \; I}}_F
    \le
    17 \Bbound \kappa^2 
    \sqrt{\gamma^{-3} (\dx + \du) (\dx^2 \kappa^2 + \du \Bbound^2) \log \frac{T}{\delta}}
    .
$
\end{lemma}

\paragraph{DAP bounds and properties.}
We need several properties that relate to the DAP parameterization and will be useful throughout. The following lemma is due to \cite[][Lemma 11]{cassel2022efficient}.

\begin{lemma}
\label{lemma:technicalParameters}
    We have that for all $\ww$ such that $\norm{w_t} \le \maxNoise$, $M \in \mathcal{M}$, and $t \le T$
    \begin{enumerate}
        \item 
        $
        \norm{\brk*{\model_\star \; I}}_F 
        \le
        \Bbound \kappa \sqrt{2 \dx / \gamma}
        ;
        $
        
        \item 
        $
        \norm{u_t(M; \ww)} 
        \le
        \maxNoise \RM \sqrt{H}
        ;
        $
        
        \item 
        $
        \norm{(\obs_t(M; \ww)\tran \; w_{t}\tran)\tran}
        \le
        \sqrt{2} \maxNoise \RM H
        ;
        $
        
        \item 
        $
        \max\brk[c]{\norm{x_t(M; \model_\star; \ww)}, \norm{x_t^{\pi_M}}, \norm{x_t}}
        \le
        2 \kappa \Bbound \maxNoise \RM \sqrt{H} / \gamma
        ;
        $
        
        \item 
        $
        \norm{u_t(M; \ww) - u_t(M; \ww')}
        \le
        \RM \norm{w_{t-H:t-1} - w_{t-H:t-1}'}
        ;
        $
        
        \item 
        $
        \sqrt{\norm{\obs_t(M; \ww) - \obs_t(M; \ww')}^2
        +
        \norm{w_t - w_t'}^2}
        \le
        \RM \sqrt{H} \norm{w_{t-2H:t-1} - w_{t-2H:t-1}'}
        .
        $
        
    \end{enumerate}
\end{lemma}

Recall that 
$
    \obs_t 
    = 
    \brk{u_{t+1-H}\tran, \ldots, u_{t}\tran, \hat{w}_{t+1-H}\tran, \ldots, \hat{w}_{t-1}\tran}\tran
$.
The following lemma complements the previous Lemma (see proof in \cref{sec:technicalProofs}).

\begin{lemma}
\label{lemma:oToReal}
We have that for all $\ww$ such that $\norm{w_t} \le \maxNoise$, $M \in \mathcal{M}$, and $t \le T$:
\begin{enumerate}
    \item
    $
    \norm{\obs_t}
    \le
    \sqrt{2} \maxNoise \RM H
    ;
    $
    \item
    $
        \norm{\obs_{t-1} - \obs_{t-1}(M_t; \ww)}^2
        \le
        2\RM^2 H \brk[s]*{
         \sum_{h=1}^{2H}\norm{w_{t-h} - \hat{w}_{t-h}}^2
         +
         \sum_{h=1}^{H} \norm{M_{t-h} - M_t}_F^2
         }
         .
    $
\end{enumerate}
\end{lemma}

\paragraph{Surrogate and optimistic costs.}
We summarize the useful properties of the surrogate costs.
To that end,
with some abuse of notation, we extend the definition of the surrogate and optimistic cost functions to include the dependence on their various parameters:
\begin{equation}
\label{eq:ft-gen-defs}
\begin{aligned}
        f_t(M; \ww)
        &=
        c_t(x_t(M; \model_\star, \ww), u_t(M; \ww))
        \\
        \bar{f}_t(M; \model, V, \ww)
        &=
        c_t(x_t(M; \model, \ww), u_t(M; \ww))
        -
        \pOptimism \norm{V^{-1/2} \obsOp(M) \wCov_{2H-1}^{1/2}}_{\infty}
        \\
        \bar{f}_t(M; \kt[], \vSign, \model, V, \ww)
        &=
        c_t(x_t(M; \model, \ww), u_t(M; \ww))
        -
        \pOptimism \vSign \cdot \brk*{V^{-1/2} \obsOp(M) \wCov_{2H-1}^{1/2}}_{\kt[]}
        ,
\end{aligned}
\end{equation}
where $\wCov_{2H-1} = I_{2H-1} \otimes \wCov$ and $\otimes$ is the Kronecker product of two matrices.
Recalling that $\ww, \wwTilde, \wwhat$ are the real, generated, and estimated noise sequences respectively, we use the following shorthand notations throughout:
\begin{equation}
\label{eq:ft-short-defs}
\begin{aligned}
    f_t(M)
    &=
    f_t(M; \ww)
    \\
    \bar{f}_t(M)
    &=
    \bar{f}_t(M; \model_{\taut}, V_{\taut}, \ww)
    \\
    \tilde{f}_t(M)
    &=
    \bar{f}_t(M; \model_{\taut}, V_{\taut}, \wwTilde)
    \\
    \tilde{f}_t(M; \kt[], \vSign)
    &=
    \bar{f}_t(M; \kt[], \vSign, \model_{\taut}, V_{\taut}, \wwTilde)
    ,
\end{aligned}
\end{equation}
where
$
    \taut
    =
    \tau_{i(t-2H)}
    =
    \max\brk[c]{\tau_i : \tau_i \le t - 2H}
    .
$
The following lemma characterizes the properties of $f_t, \bar{f}_t$ as a function of the various parameters (see proof in \cref{sec:technicalProofs}).
\begin{lemma}
\label{lemma:fbarProperties}
    Define the functions
    \begin{align*}
        \maxF(\model) = 5 \RM \maxNoise H \max\brk[c]{
        \norm{(\model \; I)}_F
        ,
        \kappa \gamma^{-1} \Bbound
        },
        \qquad
        \pLipF(\model)
        =
        \sqrt{2} \maxNoise H \norm{\model}_F
        +
        {\pOptimism}/\brk{\RM \sqrt{2H}}
        .
    \end{align*}
    For any $\ww, \ww'$ with $\norm{w_t}, \norm{w'_t} \le \maxNoise$ and $M, M'$ with $\norm{M}_F, \norm{M'}_F \le \RM$,
    we have:
    \begin{enumerate}
        \item 
        $
        \abs{f_t(M; \ww) - f_t(M; \ww')}
        \le
        \maxF(\model)
        ;
        $
        
        \item
        $
        \abs{
        \bar{f}_t(M; \model, V, \ww) 
        -
        \bar{f}_t(M; \model, V, \ww')
        }
        \le
        \maxF(\model)
        $;
    \end{enumerate}
    Additionally, if $V \succeq \pRegTheta I$ then
    \begin{enumerate}[resume]
        \item
        $
        \abs{
        \bar{f}_t(M; \model, V, \ww)
        -
        \bar{f}_t(M'; \model, V, \ww)
        }
        \le
        \pLipF(\model)
        \norm{M - M'}_F
        ;
        $
        \item
        $
        \abs{
        \bar{f}_t(M; \kt[], \vSign, \model, V, \ww)
        -
        \bar{f}_t(M'; \kt[], \vSign, \model, V, \ww)
        }
        \le
        \pLipF(\model)
        \norm{M - M'}_F
        ;
        $
        \item
        $
        \bar{f}_t(M; \kt[], \vSign, \model, V, \ww)
        \le
        \bar{f}_t(M; \model, V, \ww)
        +
        \pOptimism \sqrt{2/H}  \brk[s]{1  + \RM^{-1} \sqrt{\dx}}
        $
        .
    \end{enumerate}
    Moreover, if
    $
    \norm{\brk{\model \; I}}_F
    \le
    17 \Bbound \kappa^2 
    \sqrt{\gamma^{-3} (\dx + \du) (\dx^2 \kappa^2 + \du \Bbound^2) \log \frac{24T^2}{\delta}}
    ,
    $
    then:
    \begin{align*}
        \maxF(\model)
        \le
        5 \pOptimism / (H \sqrt{\dx(\dx+\du)})
        ,
        \quad
        \text{and}
        \;\;
        \pLipF(\model)
        \le
        {\pOptimism \sqrt{2}}/\brk{\RM \sqrt{H}}
        .
    \end{align*}
\end{lemma}

\subsection{Regret Decomposition}
\label{sec:lqrRegretDecomposition}

As seen in \cref{eq:lqrUnrolling}, the bounded state representation is such that it depends on the last $H$ decisions of the algorithm. This leads to an online convex optimization with memory problem, which complicates notation significantly. While the overall analysis is the same, we avoid the ``with-memory'' formulation using a regret decomposition that removes the memory dependence at an early stage, replacing it with a movement cost of the predictions.

Now, the following technical lemma bounds the number of epochs $\nEpochs$ (see proof in \cref{sec:lqr-side-lemmas}).
\begin{lemma}
\label{lemma:epochLengths}
We have that $\nEpochs \le 2(\dx+\du)H \log T$.
\end{lemma}
We are now ready to prove \cref{thm:lqrRegret-full}.
\begin{proof}[of \cref{thm:lqrRegret-full}]
Recall the surrogate cost and its optimistic version defined in \cref{eq:ft-gen-defs,eq:ft-short-defs}.
Letting $M_\star \in \mathcal{M}$ be the DAP approximation of $\pi \in \Pi_\text{lin}$, we have the following decomposition of the regret:
\begin{align*}
    \mathrm{Regret}_T(\pi)
    \tag{$\ocoR{1}$ - Truncation}
    &
    =
    \sum_{t=1}^{T}
    c_t(\xt,\ut) - f_t(M_{t})
    \\
    \tag{$\ocoR{2}$ - Optimism}
    &
    +
    \sum_{t=1}^{T} f_t(M_t) - \bar{f}_t(M_t)
    \\
    \tag{$\ocoR{3}$ - Excess Risk}
    &
    +
    \sum_{t=1}^{T} \bar{f}_t(M_t) - \bar{f}_t(M_\star)
    \\
    \tag{$\ocoR{4}$ - Optimism}
    &
    +
    \sum_{t=1}^{T} \bar{f}_t(M_\star) - f_t(M_\star)
    \\
    \tag{$\ocoR{5}$ - Truncation}
    &
    +
    \sum_{t=1}^{T} f_t(M_\star) - c_t(\xtpi,\utpi)
    .
\end{align*}
The proof of \cref{thm:lqrRegret} is concluded by taking a union bound over the following lemmas, which bound each of the terms (see proofs in \cref{sec:truncationCostProof,sec:optimismCostProof,sec:excessRiskCostProof}).
The technical derivation of the final regret bound is purely algebraic and may be found in \cref{lemma:regret-final-bound}.
\end{proof}

\begin{lemma}[Truncation cost]
\label{lemma:lqrR15}
    With probability at least $1 - \delta / 4$ we have that
    \begin{align*}
    \ocoR{1} + \ocoR{5}
    &
    \le
    24 \frac{\kappa^2}{\gamma^{2}} \maxNoise \Bbound^2 \RM^2 H \sqrt{T (\dx + \du) (\dx^2 \kappa^2 + \du \Bbound^2) \log \frac{4 T}{\delta}}
	+
	\frac{\kappa}{\gamma^{2}} \Bbound \maxNoise \sqrt{H} \sum_{t=1}^{T} \norm{M_t - M_{t-1}}
    .
\end{align*}
\end{lemma}

\begin{lemma}[Optimism cost]
\label{lemma:lqrR24}
    With probability at least $1 - \delta / 4$ we have that
    \begin{align*}
    \ocoR{2}+\ocoR{4}
    &
    \le
    65 \pOptimism \RM \Bbound \kappa \gamma^{-1} H \sqrt{T (\dx + \du) (\dx^2 \kappa^2 + \du \Bbound^2) \log \frac{48T^2}{\delta}}
     +
     \pOptimism\sqrt{\frac{8 H^3}{\maxNoise^2}}
     \sum_{t=1}^{T}
    \norm{M_t - M_{t-1}}_F
    .
    \end{align*} 
\end{lemma}

\begin{lemma}[Excess risk]
\label{lemma:lqrR3}
    With probability at least $1 - \delta / 2$ we have that
    \begin{align*}
    \ocoR{3}
    &
    \le
    4000
    \pOptimism (\dx+\du) 
    \sqrt{T H^3 \dx \log (6\dmodel^2) \log^3 \frac{48T^2}{\delta}}
    \\
    \sum_{t=1}^{T}
    \norm{M_t - M_{t-1}}
    &
    \le
    548 \RM (\dx + \du) H \sqrt{T \log (6\dmodel^2) \log^3\frac{48 T^2}{\delta}}
    .
\end{align*}
\end{lemma}

\subsection{Proof of \cref{lemma:lqrR15}}
\label{sec:truncationCostProof}

Bounding $\ocoR{1}$ and $\ocoR{5}$ is mostly standard in recent literature. Nonetheless, we give the details for completeness.
We start with the simpler $\ocoR{5}$. Recall from \cref{lemma:technicalParameters} that for
$
\RxuMax
=
2 \kappa \gamma^{-1} \Bbound \maxNoise \RM \sqrt{H}
$
\begin{align*}
    \max_{M \in \mathcal{M}, \norm{w} \le \maxNoise, t \le T} 
    \max\brk[c]{
    1
    ,
    \norm{x_t}
    ,
    \norm{x_t^{\pi_M}}
    ,
    \norm{x_t(M; \ww)}
    }
    \le
    \RxuMax
    .
\end{align*}
Notice that $\utpi = u_t(M_\star; \ww)$. We can thus use the Lipschitz assumption to get that
\begin{align*}
    f_t(M_\star) - c_t(\xtpi,\utpi)
    &
    =
    c_t(x_t(M_\star; \ww), u_t(M_\star; \ww))
    -
    c_t(\xtpi,\utpi)
    \\
    &
    \le
    \norm{{x}_t(M_\star; \ww) - \xtpi}
    \\
    \tag{\cref{eq:lqrUnrolling}}
    &
    =
    \norm{\Astar^H x_{t-H}^\pi}
    \\
    \tag{strong stability}
    &
    \le
    \RxuMax \kappa (1-\gamma)^H
    .
\end{align*}
Summing over $t$ and using that $(1-\gamma)^H \le e^{-\gamma H}$ we conclude that
\begin{align*}
    \ocoR{5}
    \le
    \RxuMax \kappa e^{-\gamma H} T
    .
\end{align*}
Moving to $\ocoR{1}$, notice that
$
u_t = u_t(M_{t}; \wwhat)
.
$
we thus have that
\begin{align*}
    \norm{x_t - {x}_t(M_{t}; \ww)}
    &
    =
    \norm*{
    \Astar^H x_{t-H}
    +
    \sum_{h=1}^{H} \Astar^{h-1}\Bstar \brk*{
    u_{t-h}(M_{t-h}; \wwhat)
    - 
    u_{t-h}(M_{t}; \ww)
    }
    }
    \\
    &
    \le
    \norm{\Astar^H} \norm{x_{t-H}}
    +
    \sum_{h=1}^{H} \norm{\Astar^{h-1}}\norm{\Bstar} \norm{
    u_{t-h}(M_{t-h}; \wwhat)
    - 
    u_{t-h}(M_{t}; \ww)
    }
    \\
    &
    \le
    \RxuMax \kappa e^{-\gamma H}
    +
    \kappa \Bbound \sum_{h=1}^{H}  (1-\gamma)^{h-1} \norm{
    u_{t-h}(M_{t-h}; \wwhat)
    - 
    u_{t-h}(M_{t}; \ww)
    }
    .
\end{align*}
Denoting $\brk[s]{x}_+ = \max\brk[c]{0, x}$ we further get that
\begin{align*}
    c_t(\xt,\ut) &- f_t(M_{t})
    =
    c_t(\xt,\ut) - c_t(x_t(M_{t}; \ww), u_t(M_{t}; \ww))
    \\
    &
    \le
    \norm{x_t - x_t(M_{\tauIJ}; \ww)}
    +
    \norm{u_t - u_t(M_{\tauIJ}; \ww)}
    \\
    &
    \le
    \kappa \brk[s]*{
    \RxuMax e^{-\gamma H}
    +
    \Bbound \sum_{h=0}^{H}  (1-\gamma)^{\brk[s]{h-1}_+} \norm{
    u_{t-h}(M_{t-h}; \wwhat) - u_{t-h}(M_{t}; \ww)}
    }
    .
\end{align*}
Now, we use the Lipschitz properties of $u_t$ (see \cref{lemma:technicalParameters}) to get that
\begin{align*}
	\norm{u_{t-h}(M_{t-h}; \wwhat) - u_{t-h}(M_t; \ww)}
	&
	\le
	\norm{u_{t-h}(M_{t-h}; \wwhat) - u_{t-h}(M_{t-h}; \ww)}
	\\
	&
	+
	\norm{u_{t-h}(M_{t-h}; \ww) - u_{t-h}(M_{t}; \ww)}
	\\
	&
	\RM \norm{w_{t-h-H:t-h-1} - \hat{w}_{t-h-H:t-h-1}}
	+
	\maxNoise \sqrt{H} \norm{M_t - M_{t-h}}
	,
\end{align*}
and summing over $t$ gives 
\begin{align*}
	\sum_{t=1}^{T}
	\norm{u_{t-h}(M_{t-h}; \wwhat) - u_{t-h}(M_t; \ww)}
	&
	\le
	\RM \sqrt{T H \sum_{t=1}^{T} \norm{w_{t} - \hat{w}_{t}}^2}
	+
	h \maxNoise \sqrt{H} \sum_{t=1}^{T} \norm{M_t - M_{t-1}}
	,
\end{align*}
Next, taking 
$H \ge \gamma^{-1} \log T$
we get that
\begin{align*}
    \ocoR{1}
    &
    +
    \ocoR{5}
    \le
    \kappa \brk[s]*{
    2\RxuMax e^{-\gamma H} T
    +
    \Bbound \sum_{h=0}^{H} \sum_{t=1}^{T} (1-\gamma)^{\brk[s]{h-1}_+} \norm{
    u_{t-h}(M_{t-h}; \wwhat) - u_{t-h}(M_{t}; \ww)}
    }
    \\
    &
    \le
    \kappa \brk[s]*{
    2\RxuMax
    +
    \Bbound \sum_{h=0}^{H} (1-\gamma)^{\brk[s]{h-1}_+}
    \brk*{
    \RM \sqrt{T H \sum_{t=1}^{T} \norm{w_{t} - \hat{w}_{t}}^2}
	+
	h \maxNoise \sqrt{H} \sum_{t=1}^{T} \norm{M_t - M_{t-1}}
    }
    }
    \\
    &
    \le
    2 \kappa \RxuMax
    +
    2 \kappa \gamma^{-1} \Bbound \RM \sqrt{T H \sum_{t=1}^{T} \norm{w_{t} - \hat{w}_{t}}^2}
	+
	\kappa \gamma^{-2} \Bbound \maxNoise \sqrt{H} \sum_{t=1}^{T} \norm{M_t - M_{t-1}}
    ,
\end{align*}
Finally, suppose that \cref{lemma:disturbanceEstimation} holds with $\delta/4$. Then we get that
\begin{align*}
    \ocoR{1}
    +
    \ocoR{5}
    &
    \le
    24 \kappa^2 \gamma^{-2} \maxNoise \Bbound^2 \RM^2 H \sqrt{T (\dx + \du) (\dx^2 \kappa^2 + \du \Bbound^2) \log \frac{4 T}{\delta}}
    \\
    &
	+
	\kappa \gamma^{-2} \Bbound \maxNoise \sqrt{H} \sum_{t=1}^{T} \norm{M_t - M_{t-1}}
	.
	&&&\blacksquare
\end{align*}

\subsection{Proof of \cref{lemma:lqrR24}}
\label{sec:optimismCostProof}
An optimistic cost function should satisfy two properties (in expectation). On the one hand, it is a global lower bound on the true cost function. On the other, it has a small error on the realized prediction sequence. Both of these properties are established in the following lemma.
Define the random variables
$
\Delta_t
=
\model_\star - \model_t
,
$
and
\begin{align*}
    \lseMax
    =
    \sqrt{\dx(\dx+\du) H^2} \max_{1 \le i \le \nEpochs}
    \norm{\Delta_{\tau_i} V_{\tau_i}^{1/2}}
    \\
    \taut
    =
    \tau_{i(t-2H)}
    =
    \max\brk[c]{\tau_i : \tau_i \le t - 2H}
    .
\end{align*}
Notice that $\pOptimism$ is chosen such that it bounds $\lseMax$ with high probability.
Let $\mathcal{F}_t$ be the filtration defined by the random variables
$
\brk[c]{
w_1, \ldots, w_{t-1},
M_1, \ldots, M_{t+2H}
}
.
$
Notice that this is a somewhat non-standard definition that contains variables from future time steps. This is done in order to satisfy the following properties:
\begin{itemize}
    \item Conditioning on $\mathcal{F}_{t-2H}$ does not change the distribution of $w_{t-2H:t-2}$, which are i.i.d random variables;
    \item $M_t, \taut, V_{\taut}, \model_{\taut}$ are $\mathcal{F}_{t-2H}$ measurable;
    \item $w_{1:t-1}$ is $\mathcal{F}_t$ measurable.
\end{itemize}
While the second and third requirements are trivially satisfied, the first only holds since the algorithm does not update $M_t$ during the first $2H$ rounds of each epoch.
\begin{lemma}[Optimism]
\label{lemma:lqrOptimism}
   $1 \le t \le T$ and $\mathcal{F}_{t-2H}$ measurable $M$ we have that
    \begin{align*}
        \EE\brk[s]*{\bar{f}_t(M) - f_t(M) \; \mid \; \mathcal{F}_{t-2H}}
        &
        \le
        \brk*{
        \lseMax
        -
        \pOptimism
        }
        \norm{V_{\taut}^{-1/2} \obsOp(M) \wCov_{2H-1}^{1/2}}_{\infty}
        \\
        \EE\brk[s]*{f_t(M) - \bar{f}_t(M) \; \mid \; \mathcal{F}_{t-2H}}
        &
        \le
        \brk*{
        \lseMax
        +
        \pOptimism
        }
        \sqrt{\EE \brk[s]*{\norm{V_{\taut}^{-1/2} \obs_{t-1}(M; \ww)}^2 \; \mid \; \mathcal{F}_{t-2H}}}
        .
    \end{align*}
\end{lemma}
\begin{proof}
Recall that
$
f_t(M)
=
c_t(x_t(M; {\model_\star}, \ww), u_t(M; \ww))
$
and thus using the Lipschitz property we have that
\begin{align*}
    \abs*{
    f_t(M)
    -
    c_t(x_t(M; \model_{\taut}, \ww), u_t(M; \ww))
    }
    &
    \le
    \norm{
    x_t(M; \model_{\taut}, \ww) 
    -
    x_t(M; \model_\star, \ww)}
    \\
    &
    =
    \norm{(\model_{\taut} - \model_\star) \obs_{t-1}(M; \ww)}
    \tag{\cref{eq:trunc-x-def}}
    \\
    &
    =
    \norm{\Delta_{\taut} \obs_{t-1}(M; \ww)}
    \\
    &
    \le
    \norm{\Delta_{\taut} V_{\taut}^{1/2}} \norm{V_{\taut}^{-1/2} \obs_{t-1}(M; \ww)}
    .
    \tag{Cauchy-Schwarz}
\end{align*}
Now, recall that $\obs_{t-1}(M; \ww) = \obsOp(M) w_{t-2H:t-2}$ where $\obsOp$ is as in \cref{eq:obs-op-def}. Notice that $\taut$ is $\mathcal{F}_{t-2H}$ measurable. Thus for any $M$ that is $\mathcal{F}_{t-2H}$ measurable
\begin{align*}
    \EE \brk[s]*{\norm{V_{\taut}^{-1/2} \obs_{t-1}(M; \ww)}^2 \; \mid \; \mathcal{F}_{t-2H}}
    &
    =
    {\EE \brk[s]*{\obs_{t-1}\tran(M; \ww) V_{\taut}^{-1} \obs_{t-1}(M; \ww) \; \mid \; \mathcal{F}_{t-2H}}}
    \\
    &
    =
    {\tr{V_{\taut}^{-1} \EE \brk[s]*{\obs_{t-1}(M; \ww)  \obs_{t-1}\tran(M; \ww) \; \mid \; \mathcal{F}_{t-2H}}}}
    \\
    &
    =
    {\tr{V_{\taut}^{-1} \obsOp(M) \EE \brk[s]*{w_{t-2H:t-2} w_{t-2H:t-2}\tran \; \mid \; \mathcal{F}_{t-2H}} \obsOp\tran(M)}}
    \\
    &
    =
    {\tr{V_{\taut}^{-1} \obsOp(M) \wCov_{2H-1} \obsOp\tran(M)}}
    \\
    &
    =
    \norm{V_{\taut}^{-1/2} \obsOp(M) \wCov_{2H-1}^{1/2}}_F^2
    ,
\end{align*}
where recall that $\wCov_{2H-1}$ is a block diagonal matrix with $2H-1$ blocks each containing $\wCov$.
Next, we use Jensen's inequality to get that
\begin{align*}
    \EE \brk[s]*{\norm{V_{\taut}^{-1/2} \obs_{t-1}(M; \ww)} \; \mid \; \mathcal{F}_{t-2H}}
    &
    \le
    \sqrt{\EE \brk[s]*{\norm{V_{\taut}^{-1/2} \obs_{t-1}(M; \ww)}^2 \; \mid \; \mathcal{F}_{t-2H}}}
    \\
    &
    =
    \norm{V_{\taut}^{-1/2} \obsOp(M) \wCov_{2H-1}^{1/2}}_F^2
    \\
    &
    \le
    \sqrt{2\dx(\dx+\du) H^2} \norm{V_{\taut}^{-1/2} \obsOp(M) \wCov_{2H-1}^{1/2}}_{\infty}
\end{align*}
where $\norm{Q}_{\infty}$ is the entry-wise infinity norm of a matrix $Q$ and the last inequality used the fact that for $x \in \RR[d]$ we have $\norm{x}_2 \le \sqrt{d}\norm{x}_\infty$. Noticing that $\model_{\taut}, V_{\taut}$ are $\mathcal{F}_{t-2H}$ measurable,
We get that
\begin{align*}
    \EE\brk[s]*{\bar{f}_t(M) - f_t(M) \; \mid \; \mathcal{F}_{t-2H}}
    &
    \le
    \norm{\Delta_{\taut} V_{\taut}^{1/2}} \cdot
    \EE\brk[s]*{
    \norm{V_{\taut}^{-1/2} \obs_{t-1}(M; \ww)} 
    \; \mid \;
    \mathcal{F}_{t-2H}
    }
    \\
    &
    \quad -
    \pOptimism \norm{V_{\taut}^{-1/2} \obsOp(M) \wCov_{2H-1}^{1/2}}_{\infty}
    \\
    &
    \le
    \brk*{
    \sqrt{2\dx(\dx+\du) H^2}
    \norm{\Delta_{\taut} V_{\taut}^{1/2}} 
    -
    \pOptimism
    }
    \norm{V_{\taut}^{-1/2} \obsOp(M) \wCov_{2H-1}^{1/2}}_{\infty}
    \\
    &
    \le
    \brk*{
    \lseMax
    -
    \pOptimism
    }
    \norm{V_{\taut}^{-1/2} \obsOp(M) \wCov_{2H-1}^{1/2}}_{\infty}
    ,
\end{align*}
and on the other hand
\begin{align*}
    \EE\brk[s]*{f_t(M_t) - \bar{f}_t(M_t) \; \mid \; \mathcal{F}_{t-2H}}
    &
    \le
    \norm{\Delta_{\taut} V_{\taut}^{1/2}} 
    \EE\brk[s]*{
    \norm{V_{\taut}^{-1/2} \obs_{t-1}(M; \ww)} 
    \; \mid \;
    \mathcal{F}_{t-2H}
    }
    \\
    &
    +
    \pOptimism \norm{V_{\taut}^{-1/2} \obsOp(M) \wCov_{2H-1}^{1/2}}_{\infty}
    \\
    &
    \le
    \brk*{
    \sqrt{2\dx(\dx+\du) H^2}
    \norm{\Delta_{\taut} V_{\taut}^{1/2}} 
    +
    \pOptimism
    }
    \norm{V_{\taut}^{-1/2} \obsOp(M) \wCov_{2H-1}^{1/2}}_{\infty}
    \\
    &
    \le
    \brk*{
    \lseMax
    +
    \pOptimism
    }
    \norm{V_{\taut}^{-1/2} \obsOp(M) \wCov_{2H-1}^{1/2}}_{F}
    \\
    &
    =
    \brk*{
    \lseMax
    +
    \pOptimism
    }
    \sqrt{\EE \brk[s]*{\norm{V_{\taut}^{-1/2} \obs_{t-1}(M; \ww)}^2 \; \mid \; \mathcal{F}_{t-2H}}}
    ,
\end{align*}
as desired.
\end{proof}
Next, we need to bound the additional cost incurred by summing over the confidence bound. This cost typically takes the form of a harmonic sum, yet here we have some additional terms that arise from the expected, amortized nature of our confidence bounds. This is summarized in the following lemma (see proof in \cref{sec:lqr-side-lemmas}).
\begin{lemma}
\label{lemma:expected-harmonic-sum}
With probability at least $1-\delta$
\begin{align*}
    \sum_{t=1}^{T}
    &
    \sqrt{\EE \brk[s]*{\norm{V_{\taut}^{-1/2} \obs_{t-1}(M_t; \ww)}^2 \; \mid \; \mathcal{F}_{t-2H}}}
    \\
    &
    \le
    13 H
    \sqrt{T (\dx+\du) \log \frac{4T^2 }{\delta}}
    +
    \sqrt{8 T \maxNoise^{-2} \sum_{t=1}^{T}
    \norm{w_{t} - \hat{w}_{t}}^2
    }
     +
     \sqrt{\frac{2 H^3}{\maxNoise^2}}
     \sum_{t=1}^{T}
    \norm{M_t - M_{t-1}}_F
    .
\end{align*}
\end{lemma}

We also need the following lemma that deals with the concentration of sums of variables that are independent when they are $2H$ apart in time (proof in \cref{sec:measure-concentration-proofs}).

\begin{lemma}[Block Concentration]
\label{lemma:blockConcentration}
Let $X_t$ be a sequence of random variables adapted to a filtration $\mathcal{F}_t$. Then we have the following
\begin{itemize}[leftmargin=*]
    \item 
    If $\abs{X_t - \EE\brk[s]{X_t \mid \mathcal{F}_{t-2H}}} \le C_t$ where $C_t \ge 0$ are $\mathcal{F}_{t-2H}$ measurable then w.p.\ at least $1 - \delta$
    \begin{align*}
        \sum_{t=1}^{T} \brk*{X_t - \EE\brk[s]{X_t \mid \mathcal{F}_{t-2H}}}
        \le
        2 \sqrt{\sum_{t=1}^{T} \brk{C_t^2} H \log \frac{T}{\delta}}
        ;
    \end{align*}
    \item If $0 \le X_t \le 1$ then with probability at least $1 - \delta$
    \begin{align*}
        \sum_{t=1}^{T} \EE\brk[s]{X_t \mid \mathcal{F}_{t-2H}}
        \le
        2 \sum_{t=1}^{T} (X_t)
        +
        8 H \log \frac{2 T}{\delta}
        .
    \end{align*}
\end{itemize}
\end{lemma}

We are now ready to prove \cref{lemma:lqrR24}. We begin with the slightly simpler $\ocoR{4}$. 
We want to apply \cref{lemma:blockConcentration} to $\bar{f}_t(M_\star) - f_t(M_\star)$, which are $\mathcal{F}_t$ measurable.
By \cref{lemma:fbarProperties} we have that for
$
\maxF(\model) 
=
5 \RM \maxNoise H \max\brk[c]{
        \norm{(\model \; I)}_F
        ,
        \kappa \gamma^{-1} \Bbound
        }
$
\begin{align*}
    \abs{
    \bar{f}_t(M_\star) - f_t(M_\star)
    -
    \EE\brk[s]*{\bar{f}_t(M_\star) - f_t(M_\star) \; \mid \; \mathcal{F}_{t-2H}}
    }
    \le
    2 \maxF(\model_{\taut})
\end{align*}
We thus use \cref{lemma:blockConcentration} to get that with probability at least $1 - \delta / 24$
\begin{align*}
    \sum_{t=1}^{T}
    &
    \bar{f}_t(M_\star) - f_t(M_\star)
    \\
    &
    \le
    \sum_{t=1}^{T}
    \EE\brk[s]*{\bar{f}_t(M_\star) - f_t(M_\star) \; \mid \; \mathcal{F}_{t-2H}}
    +
    4 \sqrt{\sum_{t=1}^{T} \maxF^2(\model_{\taut}) H \log\frac{24 T}{\delta}}
    \\
    &
    \le
    \sum_{t=1}^{T}
    \brk*{
        \lseMax
        -
        \pOptimism
        }
        \norm{V_{\taut}^{-1/2} \obsOp(M_\star) \wCov_{2H-1}^{1/2}}_{\infty}
    +
    4 \maxF^{\max}  \sqrt{T H \log\frac{24 T}{\delta}}
    ,
\end{align*}
where the second transition also used \cref{lemma:lqrOptimism}, and the definition
$
\maxF^{\max}
=
\max_{1 \le i \le \nEpochs} \maxF(\model_{\tau_i})
.
$
Next, we use a union bound on the events of \cref{lemma:disturbanceEstimation,lemma:lqrParameterEst} each with $\delta/24$ to bound 
$\norm{\Delta_{\taut} V_{\tau_i}^{1/2}}$
and
$\norm{(\model_{\taut} \; I)}_F$
for all $i \ge 1$. We conclude that with probability at least $1 - \delta / 12$
\begin{align*}
    \lseMax
    &
    \le
    \sqrt{\dx(\dx+\du)} H \max_{t \le T} \norm{\Delta_t}_{V_t}
    \\
    &
    \quad \le
    21 \maxNoise \RM \Bbound \kappa^2 H^2 (\dx + \du) \sqrt{\gamma^{-3} \dx (\dx^2 \kappa^2 + \du \Bbound^2) \log \frac{24 T^2}{\delta}}
    =
    \pOptimism
    ; \;\;
    \text{and}
    \\
    \tag{\cref{lemma:fbarProperties}}
    \maxF^{\max}
    &
    \le
    5 \pOptimism / (H \sqrt{\dx(\dx+\du)})
    .
\end{align*}
Plugging this back into the above bound we conclude that on the intersection of both events we have
\begin{align}
    \ocoR{4}
    &
    \le
    20 \pOptimism \sqrt{T H^{-1}  \dx^{-1} (\dx + \du)^{-1} \log\frac{24 T}{\delta}}
    .
    \label{eq:r4-final-bound}
\end{align}

Now, for $\ocoR{2}$ we start out similarly to $\ocoR{4}$. Since $M_t$ is independent of the noise terms $w_{t-2H:t-1}$,\footnote{Indeed, $M_t = 0$ in the first $2H$ steps since the start of an epoch; afterwards $V_{\taut[]}$ is fixed.}
we can use \cref{lemma:lqrOptimism,lemma:blockConcentration} to get that with probability at least $1 - \delta / 24$
 \begin{align}
    \ocoR{2}
    &
    =
    \sum_{t=1}^{T}
    \bar{f}_t(M_t) - f_t(M_t)
    \nonumber
    \\
    &
    \le
    \sum_{t=1}^{T}    \brk*{
    \EE\brk[s]*{\bar{f}_t(M_t) - f_t(M_t) \; \mid \; \mathcal{F}_{t-2H}}
    }
    +
   20 \pOptimism \sqrt{T H^{-1}  \dx^{-1} (\dx + \du)^{-1} \log\frac{24 T}{\delta}}
   \nonumber
    \\
    &
    \le
    20 \pOptimism \sqrt{T H^{-1}  \dx^{-1} (\dx + \du)^{-1} \log\frac{24 T}{\delta}}
    +
    \brk*{\lseMax + \pOptimism}
    \sum_{t=1}^{T}
    \sqrt{\EE \brk[s]*{\norm{V_{\taut}^{-1/2} \obs_{t-1}(M_t; \ww)}^2 \; \mid \; \mathcal{F}_{t-2H}}}
    .
    \label{eq:r2-first-bound}
\end{align}

It remains to bound the last term.
Using again the events of \cref{lemma:disturbanceEstimation,lemma:lqrParameterEst} (recall we've already taken a union bound over them when bounding $\ocoR{4}$), we have
\begin{align*}
    \lseMax \le \pOptimism
    ,
    \quad
    \sqrt{\sum_{t=1}^{T} \norm{w_t - \hat{w}_t}^2}
    \le
    10 \maxNoise \RM \Bbound \kappa \gamma^{-1} \sqrt{H (\dx + \du) (\dx^2 \kappa^2 + \du \Bbound^2) \log \frac{24T}{\delta}}
    .
\end{align*}
Combining this with \cref{lemma:expected-harmonic-sum} with $\delta/12$ we get that with probability at least $1-\delta/12$
\begin{align*}
    &
    \sum_{t=1}^{T}
    \sqrt{\EE \brk[s]*{\norm{V_{\taut}^{-1/2} \obs_{t-1}(M_t; \ww)}^2 \; \mid \; \mathcal{F}_{t-2H}}}
    \\
    &
    \qquad \le
    13 H
    \sqrt{T (\dx+\du) \log \frac{48T^2 }{\delta}}
    +
    \sqrt{8 T \maxNoise^{-2} \sum_{t=1}^{T}
    \norm{w_{t} - \hat{w}_{t}}^2
    }
     +
     \sqrt{\frac{2 H^3}{\maxNoise^2}}
     \sum_{t=1}^{T}
    \norm{M_t - M_{t-1}}_F
    \\
    &
    \qquad \le
    30 \RM \Bbound \kappa \gamma^{-1} H \sqrt{T (\dx + \du) (\dx^2 \kappa^2 + \du \Bbound^2) \log \frac{48T^2}{\delta}}
     +
     \sqrt{\frac{2 H^3}{\maxNoise^2}}
     \sum_{t=1}^{T}
    \norm{M_t - M_{t-1}}_F
    ,
\end{align*}
and we conclude by combining with \cref{eq:r2-first-bound,eq:r4-final-bound} that
\begin{align*}
    &\ocoR{2} + \ocoR{4}
    \le
    40 \pOptimism \sqrt{T H^{-1}  \dx^{-1} (\dx + \du)^{-1} \log\frac{24 T}{\delta}}
    \\
    &
    \qquad +
    60 \pOptimism \RM \Bbound \kappa \gamma^{-1} H \sqrt{T  (\dx + \du) (\dx^2 \kappa^2 + \du \Bbound^2) \log \frac{48 T^2}{\delta}}
     +
     \pOptimism\sqrt{\frac{8 H^3}{\maxNoise^2}}
     \sum_{t=1}^{T}
    \norm{M_t - M_{t-1}}_F
    \\
    &
    \le
    65 \pOptimism \RM \Bbound \kappa \gamma^{-1} H \sqrt{T (\dx + \du) (\dx^2 \kappa^2 + \du \Bbound^2) \log \frac{48 T^2}{\delta}}
     +
     \pOptimism\sqrt{\frac{8 H^3}{\maxNoise^2}}
     \sum_{t=1}^{T}
    \norm{M_t - M_{t-1}}_F
     .
     &&\blacksquare
\end{align*}

\subsection{Proof of \cref{lemma:lqrR3}}
\label{sec:excessRiskCostProof}

For the proof, we first need the following result (see proof in \cref{sec:oco-results}).

\begin{lemma}
\label{lemma:metaExpertsRegret}
    Let $f_t(M, \kt[], \vSign)$ be a sequence of oblivious loss functions that are convex and $G$ Lipschitz in $M$, and have a convex decision set $S$ with diameter $2R$. Let 
    $
    f_t(M) = \min_{\kt[] \in [d], \vSign \in \brk[c]{-1,1}} f_t(M, \kt[], \vSign)
    $
    and
    consider the update rule that at time $t$:
    \begin{enumerate}
        \item define loss vector $\ell_t$ such that
                $
                \brk{\ell_t}_{\kt[], \vSign}
                =
                f_t(M_t(\kt[], \vSign); \kt[], \vSign) / (2 G R + C)
                $
                \item
                update experts: 
                $
                    M_{t+1}(\kt[], \vSign)
                    =
                    \Pi_{\mathcal{M}}\brk[s]*{
                    M_t(\kt[], \vSign)
                    -
                    \eta \nabla_{M} f_t(M_t(\kt[], \vSign); \kt[], \vSign)
                    }
                $
                \item update prediction:
                $
                    (\kt[t+1], \vSign_{t+1})
                    =
                    \text{BFPL}_{\delta}^*(\ell_t)
                $
                and set 
                $
                M_{t+1}
                =
                M_{t+1}(\kt[t+1], \vSign_{t+1})
                $
    \end{enumerate}
    where $\eta = 2R / \bar{G} \sqrt{T}$ and $C \ge 0$. 
    Suppose that:
    \begin{enumerate}[label=\roman*]
        \item 
        $
        f_t(M; \kt[], \vSign)
        \le
        f_t(M)
        +
        C
        $
        for all $M \in S, \kt[] \in [d], \vSign \in \brk[c]{\pm 1}$;
        
        \item There exist $\kt[](M), \vSign(M)$ independent of $t$ such that $f_t(M) = f_t(M, \kt[](M), \vSign(M))$.
    \end{enumerate}
     Then with probability at least $1 - \delta$ we have that for all $\tau \le T$
    \begin{align*}
        \sum_{t=1}^{\tau} f_t(M_t) - f_t(M)
        &
        \le
        {151} \brk[s]{(\bar{G} + G^2 \bar{G}^{-1}) R + C} \sqrt{T \log (2d) \log \frac{2T}{\delta}}
        \\
        \sum_{t=1}^{\tau}
        \norm{M_t - M_{t-1}}
        &\le
        (270 + 2 G \bar{G}^{-1})
        R \sqrt{T \log (2d) \log \frac{2 T}{\delta}}
        .
    \end{align*}
\end{lemma}

We now proceed with the proof of \cref{lemma:lqrR3}.
Recall the definition (\cref{eq:ft-gen-defs,eq:ft-short-defs}):
\begin{align*}
    \tilde{f}_t(M)
    =
    c_t(x_t(M; \model_{\taut}, \wwTilde), u_t(M; \wwTilde))
    -
    \pOptimism
    \norm{V_{\taut}^{-1/2} \obsOp(M) \wCov_{2H-1}^{1/2}}_{\infty}
    ,
\end{align*}
where $\taut = \max\brk[c]{\tau_i \;:\; \tau_i \le t-2H}$ is the index of the episode to which $t-2H$ belongs. These are essentially identical to $\bar{f}$ but with the actual noise sequence $\ww$ replaced with our independently generated noise $\wwTilde$.
We begin by further decomposing $\ocoR{3}$ as
\begin{align*}
    \ocoR{3}
    =
    \underbrace{
    \sum_{t=1}^{T}
    \bar{f}_t(M_t) - \tilde{f}_t(M_t)
    }_{\ocoR{3,1}}
    +
    \underbrace{
    \sum_{t=1}^{T}
    \tilde{f}_t(M_t) - \tilde{f}_t(M_\star)
    }_{\ocoR{3,2}}
    +
    \underbrace{
    \sum_{t=1}^{T}
    \tilde{f}_t(M_\star) - \bar{f}_t(M_\star)
    }_{\ocoR{3,3}}
\end{align*}
We start with $\ocoR{3,1}, \ocoR{3,3}$.
Let $\mathcal{G}_t$ be the filtration defined by the random variables 
$
\{
w_1, \ldots, w_{t-1},
$
$
\tilde{w}_1, \ldots, \tilde{w}_{t-1},
M_1, \ldots, M_{t}
\}
$
and recall that $\bar{f}, \tilde{f}$ only depend on $w_{t-2H:t-1}, \tilde{w}_{t-2H:t-1}$ respectively and are thus $\mathcal{G}_t$ measurable. Furthermore, for any $\mathcal{G}_{t-2H}$ measurable $M$ we have
\begin{align*}
    \EE\brk[s]{\tilde{f}_t(M) - \bar{f}_t(M) \mid \mathcal{G}_{t-2H}}
    =
    0
    .
\end{align*}
Recalling from \cref{lemma:fbarProperties} that $2\maxF(\model_{\taut})$ bounds this sequence and are also $\mathcal{G}_{t-2H}$ measurable, we thus use \cref{lemma:blockConcentration} to get with probability at least $1 - \delta / 24$
\begin{align*}
    \ocoR{3,3}
    =
    \sum_{t =1}^{T}
    \tilde{f}_t(M_\star) - \bar{f}_t(M_\star)
    &
    \le
    4 \sqrt{\sum_{t=1}^{T} \maxF^2(\model_{\taut}) H \log \frac{24 T}{\delta}}
    \\
    &
    \le
    4 \maxF^{\max}\sqrt{T H \log \frac{24 T}{\delta}}
    ,
\end{align*}
where
$
\maxF^{\max}
=
\max_{1 \le i \le \nEpochs} \maxF(\model_{\tau_i})
.
$
Next, notice that $M_{t-2H}$ is $\mathcal{G}_{t-2H}$-measureable and thus
\begin{align*}
    \EE\brk[s]{
    \tilde{f}_t(M_{t-2H}) - \bar{f}_t(M_{t-2H})
    \mid
    \mathcal{G}_{t-2H}
    }
    =
    0
    .
\end{align*}
We thus use the same set of arguments as in $\ocoR{3,3}$ to get that with probability at least $1-\delta/24$
\begin{align*}
    \sum_{t = 1}^{T}
    \tilde{f}_t(M_{t-2H}) - \bar{f}_t(M_{t-2H})
    \le
    4 \maxF^{\max}\sqrt{T H \log \frac{24 T}{\delta}}
    .
\end{align*}
Now, let 
$
\pLipF^{\max}
=
\max_{1 \le i \le \nEpochs} \pLipF(\model_{\tau_i})
$
be the upper bound on the Lipschitz constant of $\bar{f}_t$ for all $t \le T$, as defined in \cref{lemma:fbarProperties}. Then we have that
\begin{align*}
    \abs*{
    \sum_{t = 1}^{T}
    \bar{f}_t(M_t) - \bar{f}(M_{t-2H})
    }
    &
    \le
    \pLipF^{\max}
    \sum_{t = 1}^{T}
    \norm{M_t - M_{t-2H}}_F
    \\
    &
    \le
    \pLipF^{\max}
    \sum_{t = 1}^{T}
    \sum_{h=0}^{2H-1}
    \norm{M_{t-h} - M_{t-(h+1)}}_F
    \\
    &
    \le
    2 \pLipF^{\max} H
    \sum_{t=1}^{T}
    \norm{M_t - M_{t-1}}_F
    .
\end{align*}
Combining with the previous inequalities, we get that
\begin{align*}
    \ocoR{3,1} + \ocoR{3,3}
    \le
    8 \maxF^{\max}\sqrt{T H \log \frac{24 T}{\delta}}
    +
    4 \pLipF^{\max} H
    \sum_{t=1}^{T}
    \norm{M_t - M_{t-1}}_F
    .
\end{align*}

Moving to $\ocoR{3,2}$, we split the analysis into epochs, and combine the results via a union bound. First, we fix some $1 \le i \le \nEpochs$ and define for all $t \ge 1$, $\kt[] \in [\dmodel]\times[(2H-1)\dx], \vSign \in \brk[c]{\pm 1}$ the functions 
\begin{align*}
    &\ftildeti(M; \kt[], \vSign)
    \\
    &\qquad=
    c_{\tau_i +2H +t-1}(x_{\tau_i +2H +t-1}(M; \model_{\tau_i}, \wwTilde), u_{\tau_i +2H +t-1}(M; \wwTilde))
    -
    \pOptimism \vSign \cdot \brk*{V_{\tau_i}^{-1/2}  \obsOp(M) \wCov_{2H-1}^{1/2}}_{\kt[]};
    \\
    &\ftildeti(M)
    \\
    &\qquad
    =
    c_{\tau_i +2H +t-1}(x_{\tau_i +2H +t-1}(M; \model_{\tau_i}, \wwTilde), u_{\tau_i +2H +t-1}(M; \wwTilde))
    -
    \pOptimism \norm*{V_{\tau_i}^{-1/2}  \obsOp(M) \wCov_{2H-1}^{1/2}}_{\infty}
    .
\end{align*}

Let $\Mti$ be the iterates that result from running the procedure described in \cref{lemma:metaExpertsRegret} on the functions $\ftildeti$ starting with an arbitrary$\Mti[1] \in \mathcal{M}$.
Then we have the following observations:
\begin{enumerate}
    \item $\ftildeti$ are oblivious with respect to the iterates $\Mti$. This is because $c_t$ are oblivious and $\tau_i, \model_{\tau_i}, V_{\tau_i}, \wwTilde$ can be determined independently of the iterates $\Mti$;
    \item $\ftildeti(M; \kt[], \vSign)$ are convex as a sum of a linear function with a composition of a convex and affine functions;
    \item $\ftildeti(M; \kt[], \vSign)$ are $\pLipF(\model_{\tau_i})$ Lipschitz (see \cref{lemma:fbarProperties});
    \item 
    $
    \ftildeti(M; \kt[], \vSign)
    \le
    \ftildeti(M)
    +
    C
    $ 
    for all $M \in \mathcal{M}, \kt[] \in [\dmodel] \times [(2H-1)\dx], \vSign \in \brk[c]{\pm 1}$ where 
    $
    C
    =
    \pOptimism \sqrt{2/H}  \brk[s]{(1  + \RM^{-1} \sqrt{\dx}}
    $
    (see \cref{lemma:fbarProperties});
    \item $\gdlr = 2 \RM / \bar{G} \sqrt{T}$ where $\bar{G} = \pOptimism \sqrt{2/H} \RM^{-1}$.
\end{enumerate}
Next, let
\[
    \kt[](M), \vSign(M)
    \in
    \argmax_{\substack{\kt[] \in [\dmodel]\times[(2H-1)\dx], \\ \vSign \in \brk[c]{\pm 1}}} 
    \vSign \cdot \brk*{V_{\tau_i}^{-1/2} \obsOp(M) \wCov_{2H-1}^{1/2}}_{\kt[]}
    .
\]
Since the term being maximized is independent of $t$ then so are $\kt[](M), \vSign(M)$. We thus get that for all $t \ge 1$
\begin{align*}
    \ftildeti(M)
    &
    =
    c_t(x_t(M; \model_{\tau_i}, \wwTilde), u_t(M; \wwTilde))
    -
    \pOptimism \norm*{V_{\tau_i}^{-1/2}  \obsOp(M) \wCov_{2H-1}^{1/2}}_{\infty}
    \\
    &
    =
    c_t(x_t(M; \model_{\tau_i}, \wwTilde), u_t(M; \wwTilde))
    -
    \pOptimism \vSign(M) \cdot \brk*{V_{\tau_i}^{-1/2} \obsOp(M) \wCov_{2H-1}^{1/2}}_{\kt[](M)}
    \\
    &
    =
    \ftildeti(M; \kt[](M), \vSign(M))
    .
\end{align*}
We thus use \cref{lemma:metaExpertsRegret} with $\delta/24T$ and take a union bound over the epochs to get that with probability at least $1-\delta/24$ simultaneously for all $1 \le \tau \le T, 1 \le i \le \nEpochs$
\begin{align*}
    \sum_{t = 1}^{\tau}
    \ftildeti(\Mti) - \ftildeti(M_\star)
    &\le
    151 \brk[s]{(\bar{G} + \pLipF(\model_{\tau_i})^2 \bar{G}^{-1}) \RM + C} \sqrt{T \log (6\dmodel^2) \log \frac{48T^2}{\delta}}, \quad \text{and}
    \\
    \sum_{t=1}^{\tau}
    \norm{\Mti - \Mti[t-1]}
    &\le
    (270 + 2 \pLipF(\model_{\tau_i}) \bar{G}^{-1}) \RM \sqrt{T \log (6\dmodel^2) \log\frac{48 T^2}{\delta}}
    .
\end{align*}
Next, suppose that the events of \cref{lemma:disturbanceEstimation,lemma:lqrParameterEst} each holds with $\delta/24$. This occurs with probability at least $1-\delta/12$ and implies 
$
\norm{(\model_{\tauI} \;  I)}_F
\le
17 \Bbound \kappa^2 
    \sqrt{\gamma^{-3} (\dx + \du) (\dx^2 \kappa^2 + \du \Bbound^2) \log \frac{24T^2}{\delta}}
    .
$
Pluuging this into \cref{lemma:fbarProperties} we get that
\begin{align*}
    \maxF^{\max} 
    \le
    5 \RM \maxNoise H \max\brk[c]{\max_{t \le T}
    \norm{(\model_t \; I)}_F
    ,
    \kappa \gamma^{-1} \Bbound
    }
    \le
    5 \pOptimism / (H \sqrt{\dx(\dx+\du)})
    \\
    \pLipF(\model_{\tau_i})
    \le
    \pLipF^{\max}
    \le
    2 \maxNoise H \max_{t \le T}\norm{(\model \; I)}_F
    \le
    {\pOptimism \sqrt{2}}/\brk{\RM \sqrt{H}}
    ,
\end{align*}
and therefore
\begin{align*}
    \brk[s]{(\bar{G} + \pLipF(\model_{\tau_i})^2 \bar{G}^{-1}) \RM + C}
    \le
    2 \bar{G} \RM + C
    &
    \le
    2 \pOptimism \sqrt{2/H} \RM^{-1} \RM
    +
    \pOptimism \sqrt{2/H}  \brk[s]{1  + \RM^{-1} \sqrt{\dx}}
    \\
    &
    \le
    \pOptimism \sqrt{2/H}  \brk[s]{3  + \RM^{-1} \sqrt{\dx}}
    \\
    &
    \le
    \pOptimism \sqrt{8 \dx /H}
    .
\end{align*}
Now, notice that for $\tau_i + 2H \le t \le \tau_{i+1}-1$ we have that $\tilde{f}_t$ in \cref{alg:lqr} coincide with $\ftildeti[t+1-(\tau_i + 2H)]$. Moreover, $\gdlr = 2\RM / (\bar{G} \sqrt{T})$ and the scaling factor $\lossScale$ in \cref{alg:lqr} satisfies
\begin{align*}
    \lossScale(\model_{\tau_i})
    &
    =
    \sqrt{8} \maxNoise \RM H \norm{\model}_F
    +
    \pOptimism \sqrt{2/H}  \brk[s]{(2  + \RM^{-1} \sqrt{\dx}}
    \\
    &
    =
    2 \pLipF(\model_{\tau_i}) \RM
    +
    \pOptimism \sqrt{2/H}  \brk[s]{(1  + \RM^{-1} \sqrt{\dx}}
    =
    2 G R + C
    ,
\end{align*}
which implies that \cref{alg:lqr} runs the same procedure as in \cref{lemma:metaExpertsRegret} and thus 
$
M_t = \Mti[t+1-(\tau_i + 2H)]
.
$
We conclude that
\begin{align*}
    \ocoR{3,2}
    &
    \le
    2 \pLipF^{\max} \RM \nEpochs
    +
    \sum_{i=1}^{\nEpochs}
    \sum_{t = \tau_i + 2H}^{\tau_{i+1}-1}
    \tilde{f}_t(M_t) - \tilde{f}_t(M_\star)
    \\
    &
    =
    2 \pLipF^{\max} \RM \nEpochs
    +
    \sum_{i=1}^{\nEpochs}
    \sum_{t = 1}^{\tau_{i+1}-(\tau_i+1)}
    \ftildeti(\Mti) - \ftildeti(M_\star)
    \\
    &
    \le
    \pOptimism \sqrt{8/H} \nEpochs
    \brk[s]*{
    1
    +
    151
    \sqrt{T \dx \log (6\dmodel^2) \log \frac{48T^2}{\delta}}
    }
    \\
    &
    \le
    860
    \pOptimism (\dx+\du) \log(T)
    \sqrt{T H \dx \log (6\dmodel^2) \log \frac{48T^2}{\delta}}
    \\
    &
    \le
    860
    \pOptimism (\dx+\du) 
    \sqrt{T H \dx \log (6\dmodel^2) \log^3 \frac{48T^2}{\delta}}
    ,
\end{align*}
where in the first inequality we bounded the loss at the first $2H$ rounds of each epoch using the Lipschitz property of $\tilde{f}_t$ (see \cref{lemma:fbarProperties}), and in the third we bounded $\nEpochs \le 2 H (\dx + \du) \log(T)$ using \cref{lemma:epochLengths}.
We also get that
\begin{align*}
    \sum_{t=1}^{T}
    \norm{M_t - M_{t-1}}
    &
    \le
    2 \RM \nEpochs
    +
    \sum_{i=1}^{\nEpochs}
    \sum_{t = \tau_i + 2H}^{\tau_{i+1}-1}
    \norm{M_t - M_{t-1}}
    \\
    &
    \le
    \nEpochs \brk[s]*{
    2 \RM
    +
    272 \RM \sqrt{T \log (6\dmodel^2) \log\frac{48 T^2}{\delta}}
    }
    \\
    &
    \le
    548 \RM (\dx + \du) H \log(T) \sqrt{T \log (6\dmodel^2) \log\frac{48 T^2}{\delta}}
    \\
    &
    \le
    548 \RM (\dx + \du) H \sqrt{T \log (6\dmodel^2) \log^3\frac{48 T^2}{\delta}}
     .
\end{align*}
Plugging this back into the bound for $\ocoR{3,1} + \ocoR{3,3}$ we get
\begin{align*}
    \ocoR{3,1} + \ocoR{3,3}
    &
    \le
    8 \maxF^{\max}\sqrt{T H \log \frac{24 T}{\delta}}
    +
    4 \pLipF^{\max} H
    \sum_{t=1}^{T}
    \norm{M_t - M_{t-1}}_F
    \\
    &
    \le
    40 \pOptimism \sqrt{T H^{-1} \dx^{-1} (\dx+\du)^{-1} \log \frac{24 T}{\delta}}
    +
    \pOptimism \sqrt{32 H}\RM^{-1}
    \sum_{t=1}^{T}
    \norm{M_t - M_{t-1}}_F
    \\
    &
    \le
    40 \pOptimism \sqrt{T H^{-1} \dx^{-1} (\dx+\du)^{-1} \log \frac{24 T}{\delta}}
    +
    3100\pOptimism
    (\dx + \du) \sqrt{T H^3 \log (6\dmodel^2) \log^3\frac{48 T^2}{\delta}}
    \\
    &
    \le
    3140\pOptimism
    (\dx + \du) \sqrt{T H^3 \log (6\dmodel^2) \log^3\frac{48 T^2}{\delta}}
    .
\end{align*}
Combining the bounds for $\ocoR{3,1}, \ocoR{3,2}, \ocoR{3,3}$ we conclude that
\begin{align*}
    \ocoR{3}
    &
    \le
    3140\pOptimism
    (\dx + \du) \sqrt{T H^3 \log (6\dmodel^2) \log^3\frac{48 T^2}{\delta}}
    \\
    &
    +
    860
    \pOptimism (\dx+\du) 
    \sqrt{T H \dx \log (6\dmodel^2) \log^3 \frac{48T^2}{\delta}}
    \\
    &
    \le
    4000
    \pOptimism (\dx+\du) 
    \sqrt{T H^3 \dx \log (6\dmodel^2) \log^3 \frac{48T^2}{\delta}}
     .
\end{align*}
Taking a union bound over all the events throughout the lemma we have that their intersection occurs with probability at least $1 - \delta / 4$.
\hfill$\blacksquare$

\subsection{Side lemmas}
\label{sec:lqr-side-lemmas}

\begin{proof}[of \cref{lemma:epochLengths}]
    The algorithm ensures that 
    \begin{align*}
    \det(V_T)
    \ge
    \det(V_{\tauI[\nEpochs]}) 
    \ge
    2 \det(V_{\tauI[\nEpochs-1]})
    \ldots
    \ge
    2^{\nEpochs-1} \det{V_1}
    ,
    \end{align*}
    and changing sides, and taking the logarithm we conclude that
    \begin{align*}
    \nEpochs
    &
    \le
    1
    +
    \log \brk*{\det(V_{T}) / \det(V)}
    \\
    &
    =
    1
    +
    \log \det(V^{-1/2} V_{T+1} V^{-1/2})
    \\
    \tag{$\det(A) \le \norm{A}^d$}
    &
    \le
    1
    +
    (\dx+\du)H \log \norm{V^{-1/2} V_{T} V^{-1/2}}
    \\
    \tag{triangle inequality}
    &
    \le
    1
    +
    (\dx+\du)H \log \brk*{1 + \frac{1}{\pRegTheta}\sum_{t=1}^{T-1} \norm{\obs_t}^2}
    \\
    &
    \le
    1
    +
    (\dx+\du)H \log T
    \\
    \tag{$T \ge 3$}
    &
    \le
    2(\dx+\du)H \log T
    ,
\end{align*}
where the second to last inequality holds since $\norm{\obs_t}^2 \le \pRegTheta$ by \cref{lemma:technicalParameters}.
\end{proof}

\begin{proof}[of \cref{lemma:expected-harmonic-sum}]
First, we use \cref{lemma:oToReal} and our choice of $\pRegTheta = 2 \maxNoise^2 \RM^2 H^2$ to get that
\begin{align*}
    &
    \sum_{t=1}^{T}
    \sqrt{2 \EE \brk[s]*{
    \norm{V_{\taut}^{-1/2} (\obs_{t-1}(M_t; \ww) - \obs_{t-1})}^2 \; \mid \; \mathcal{F}_{t-2H}}}
    \\
    &
    \qquad \le
    \sqrt{\frac{2}{\pRegTheta}}
    \sum_{t=1}^{T}
    \sqrt{\EE \brk[s]*{
    \norm{\obs_{t-1}(M_t; \ww) - \obs_{t-1}}^2 \; \mid \; \mathcal{F}_{t-2H}}}
    \\
    &
    \qquad \le
    (\maxNoise \RM H)^{-1}
    \sum_{t=1}^{T}
    \sqrt{\EE \brk[s]*{
    2\RM^2 H \brk[s]*{
     \sum_{h=1}^{2H}\norm{w_{t-h} - \hat{w}_{t-h}}^2
     +
     \sum_{h=1}^{H} \norm{M_{t-h} - M_t}_F^2
     } \; \mid \; \mathcal{F}_{t-2H}}}
    \\
    &
    \qquad \le
    \sqrt{\frac{2}{H \maxNoise^2}}
    \sum_{t=1}^{T}
    \sqrt{\EE \brk[s]*{
     \sum_{h=1}^{2H}\norm{w_{t-h} - \hat{w}_{t-h}}^2
     \; \mid \; \mathcal{F}_{t-2H}}
     +
     \sum_{h=1}^{H} \norm{M_{t-h} - M_t}_F^2
     }
    \\
    &
    \qquad \le
    \sqrt{\frac{2}{H \maxNoise^2}}
    \brk[s]*{
    \sqrt{
   T
    \sum_{h=1}^{2H}
    \sum_{t=1}^{T}
    \EE \brk[s]*{
    \norm{w_{t-h} - \hat{w}_{t-h}}^2
     \mid \mathcal{F}_{t-2H}}}
     +
     \sum_{t=1}^{T}
     \sum_{h=1}^{H} \norm{M_{t-h} - M_t}_F
      }
      ,
\end{align*}
where the second to last transition also used the fact that for $h = 0, \ldots H$ we have $M_{t-h}$ is $\mathcal{F}_{t-2H}$ measurable, and the last transition used both Jensen's inequality and $\norm{x}_2 \le \norm{x}_1$.
Now, we seek to apply \cref{lemma:blockConcentration} for the disturbances. Indeed, we have that $\norm{w_{t-h} - \hat{w}_{t-h}}^2$ are non-negative, bounded by $4 \maxNoise^2$, and $\mathcal{F}_t$ measurable for all $1 \le t - h \le T$. Using \cref{lemma:blockConcentration} with $\delta / 2$ we get that with probability at least $1-\delta / 2$
\begin{align*}
    \sum_{t=1}^{T}
    \EE \brk[s]*{
    \norm{w_{t-h} - \hat{w}_{t-h}}^2
    \mid \mathcal{F}_{t-2H}}
    \le
    2\sum_{t=1}^{T}
    \norm{w_{t-h} - \hat{w}_{t-h}}^2
    +
    32 \maxNoise^2 H \log \frac{4 T}{\delta}
    .
\end{align*}
Next, we also have that
\begin{align*}
    \sum_{t=1}^{T} \sum_{h=1}^{H} \norm{M_{t-h} - M_t}_F
    \le
    \sum_{t=1}^{T}
    \sum_{h=1}^{H} \sum_{h'=0}^{h-1}
    \norm{M_{t-h'} - M_{t-(h'+1)}}_F
    \le
    H^2 \sum_{t=1}^{T}
    \norm{M_t - M_{t-1}}_F
    ,
\end{align*}
and thus plugging both of these into the above we get that
\begin{align*}
    &
    \sum_{t=1}^{T}
    \sqrt{2 \EE \brk[s]*{
    \norm{V_{\taut}^{-1/2} (\obs_{t-1}(M_t; \ww) - \obs_{t-1})}^2 \; \mid \; \mathcal{F}_{t-2H}}}
    \\
    &
    \le
    \sqrt{\frac{2}{H \maxNoise^2}}
    \brk[s]*{
    \sqrt{
    T 
    \sum_{h=1}^{2H}
    \brk[s]*{
    32 \maxNoise^2 H \log \frac{4 T}{\delta}
    +
    2\sum_{t=1}^{T}
    \norm{w_{t-h} - \hat{w}_{t-h}}^2
    }}
     +
     H^2 \sum_{t=1}^{T}
    \norm{M_t - M_{t-1}}_F
      }
      \\
      &
    \le
    \sqrt{
    128 T
    H \log \frac{4 T}{\delta}
    }
    +
    \sqrt{8 T \maxNoise^{-2} \sum_{t=1}^{T}
    \norm{w_{t} - \hat{w}_{t}}^2
    }
     +
     \sqrt{\frac{2 H^3}{\maxNoise^2}}
     \sum_{t=1}^{T}
    \norm{M_t - M_{t-1}}_F
      .
\end{align*}
Next, we seek to apply \cref{lemma:blockConcentration} to $\norm{V_{\taut}^{-1/2} \obs_{t-1}}^2$. Indeed they are non-negative, $\mathcal{F}_t$ measurable and by \cref{lemma:oToReal} satisfy
\begin{align*}
    \norm{V_{\taut}^{-1/2} \obs_{t-1}}^2
    \le
    \pRegTheta^{-1} \norm{o_{t-1}}^2
    \le
    (2 \maxNoise^2 \RM^2 H^2)^{-1} 2 \maxNoise^2 \RM^2 H^2
    =
    1
\end{align*}
Applying \cref{lemma:blockConcentration} with $\delta / 2$ we get that with probability at least $1-\delta/2$
\begin{align*}
    \sum_{t=1}^{T}
    \EE \brk[s]*{\norm{V_{\taut}^{-1/2} \obs_{t-1}}^2 \; \mid \; \mathcal{F}_{t-2H}}
    &
    \le
    2 \sum_{t=1}^{T}
    \brk*{
    \norm{V_{\taut}^{-1/2} \obs_{t-1}}^2}
    +
    8 H \log \frac{4 T}{\delta}
    \\\
    &
    \le
    2 \sum_{t=1}^{T}
    \brk*{
    \norm{V_{\taut[t+2H]}^{-1/2} \obs_{t-1}}^2}
    +
    4 H \nEpochs
    +
    8 H \log \frac{4 T}{\delta}
    \\
    &
    \le
    4 \sum_{t=1}^{T}
    \brk*{
    \norm{V_{t-1}^{-1/2} \obs_{t-1}}^2}
    +
    4 H \nEpochs
    +
    8 H \log \frac{4 T}{\delta}
    \\
    &
    \le
    (20 + 8H) H (\dx+\du) \log(T)
    +
    8 H \log \frac{4 T}{\delta}
    \\
    &
    \le
    18 H^2 (\dx+\du) \log \frac{4 T^2 }{\delta}
    ,
\end{align*}
where the second transition used the fact that there are at most $2H$ times per epoch for which $\taut \neq \taut[t+2H]$, i.e., is not the start of the current epoch, the third transition used \cite[][Lemma 27]{cohen2019learning}, which states that for $V_1 \succeq V_2 \succeq 0$ we have $o\tran V_1 o \le (o\tran V_2 o) \det(V_1) / \det(V_2)$, and the fourth transition also used \cref{lemma:harmonicBound} to bound the harmonic sum. Using Jensen's inequality, we get that
\begin{align*}
    \sum_{t=1}^{T}
    &
    \sqrt{\EE \brk[s]*{
    2 \norm{V_{\taut}^{-1/2} \obs_{t-1}}^2  \mid  \mathcal{F}_{t-2H}}}
    \le
    \sqrt{2 T \sum_{t=1}^{T}
    \EE \brk[s]*{
     \norm{V_{\taut}^{-1/2} \obs_{t-1}}^2  \mid  \mathcal{F}_{t-2H}}}
    \\
    &
    \le
    \sqrt{36 T H^2 (\dx+\du) \log \frac{4 T^2}{\delta}}
    .
\end{align*}
Taking a union bound and combining the last two inequalities, we conclude that with probability at least $1-\delta$
\begin{align*}
    &
    \sum_{t=1}^{T}
    \sqrt{\EE \brk[s]*{\norm{V_{\taut}^{-1/2} \obs_{t-1}(M_t; \ww)}^2 \; \mid \; \mathcal{F}_{t-2H}}}
    \\
    &
    \le
    \sum_{t=1}^{T}
    \sqrt{\EE \brk[s]*{
    2 \norm{V_{\taut}^{-1/2} \obs_{t-1}}^2  \mid  \mathcal{F}_{t-2H}}}
    +
   \sum_{t=1}^{T}
    \sqrt{2 \EE \brk[s]*{
    \norm{V_{\taut}^{-1/2} (\obs_{t-1}(M_t; \ww) - \obs_{t-1})}^2 \; \mid \; \mathcal{F}_{t-2H}}}
    \\
    &
    \le
    13 H
    \sqrt{T (\dx+\du) \log \frac{4T^2 }{\delta}}
    +
    \sqrt{8 T \maxNoise^{-2} \sum_{t=1}^{T}
    \norm{w_{t} - \hat{w}_{t}}^2
    }
     +
     \sqrt{\frac{2 H^3}{\maxNoise^2}}
     \sum_{t=1}^{T}
    \norm{M_t - M_{t-1}}_F
    .
    \qedhere
\end{align*}
\end{proof}

The following lemma combines the bounds in \cref{lemma:lqrR15,lemma:lqrR24,lemma:lqrR3} to complete the final regret bound in \cref{thm:lqrRegret-full}.
\begin{lemma}
\label{lemma:regret-final-bound}
We have that with probability at least $1-\delta$
\begin{align*}
    \mathrm{Regret}_T(\pi)
    \le
     37483  \maxNoise \RM^2 \Bbound^2 \kappa^3 \gamma^{-6}
     (\dx^2 \kappa^2 + \du \Bbound^2)
       \log^{6} \frac{48T^2}{\delta}
     \sqrt{T (\dx + \du)^3 \log (2 \dmodel)}
    .
\end{align*}
\end{lemma}
\begin{proof}
Suppose that the events of \cref{lemma:lqrR15,lemma:lqrR24,lemma:lqrR3} hold. By a union bound, this holds with probability at least $1-\delta$.
Now, we simplify each of the terms before deriving the final bound.
Recall from \cref{thm:lqrRegret-full} that
\begin{align*}
    \pOptimism
    &
    =
    21 \maxNoise \RM \Bbound \kappa^2 (\dx + \du)
    \sqrt{H^3 \gamma^{-3}  (\dx^2 \kappa^2 + \du \Bbound^2) \log \frac{24T^2}{\delta}}
    \\
    &
    \qquad \le
   21 \maxNoise \RM \Bbound \kappa^2 \gamma^{-3} (\dx + \du)
    \sqrt{ (\dx^2 \kappa^2 + \du \Bbound^2) \log^4 \frac{24T^2}{\delta}} 
       .
\end{align*}
Now, plugging the movement cost bound from \cref{lemma:lqrR3} into \cref{lemma:lqrR15}, we have
\begin{align*}
\ocoR{1} + \ocoR{5}
&
\le
24 \frac{\kappa^2}{\gamma^{2}} \maxNoise \Bbound^2 \RM^2 H \sqrt{T (\dx + \du) (\dx^2 \kappa^2 + \du \Bbound^2) \log \frac{4 T}{\delta}}
+
\frac{\kappa}{\gamma^{2}} \Bbound \maxNoise \sqrt{H} \sum_{t=1}^{T} \norm{M_t - M_{t-1}}
\\
&
\le
24 \frac{\kappa^2}{\gamma^{2}} \maxNoise \Bbound^2 \RM^2 H \sqrt{T (\dx + \du) (\dx^2 \kappa^2 + \du \Bbound^2) \log \frac{4 T}{\delta}}
\\
&
\qquad +
548 \frac{\kappa}{\gamma^{2}} \maxNoise \RM \Bbound (\dx + \du) \sqrt{T H^3 \log (6\dmodel^2) \log^3\frac{48 T^2}{\delta}}
\\
&
\le
554 \frac{\kappa^2}{\gamma^{2}} \maxNoise \Bbound^2 \RM^2 \sqrt{T H^3 (\dx + \du) (\dx^2 \kappa^2 + \du \Bbound^2) \log (6\dmodel^2) \log^3\frac{48 T^2}{\delta}}
\\
&
\le
554 \maxNoise \Bbound^2 \RM^2 \kappa^2 \gamma^{-4} 
\log^3 \brk*{\frac{48 T^2}{\delta}}
\sqrt{T (\dx + \du) (\dx^2 \kappa^2 + \du \Bbound^2) \log (6\dmodel^2)}
\\
&
\le
\maxNoise \RM^2 \Bbound^2 \kappa^3 \gamma^{-8} 
     (\dx^2 \kappa^2 + \du \Bbound^2) \log^6 \brk*{ \frac{48T^2}{\delta}}     \sqrt{T \dx (\dx + \du)^3 \log (6\dmodel^2)  }
,
\end{align*}
where the third inequality used the fact that $T \ge 8$ and thus $\log 48T^2 \ge 8$.
Next, we do the same for $\ocoR{2}, \ocoR{4}$ to get that
\begin{align*}
\ocoR{2}+\ocoR{4}
&
\le
65 \pOptimism \RM \Bbound \kappa \gamma^{-1} H \sqrt{T (\dx + \du) (\dx^2 \kappa^2 + \du \Bbound^2) \log \frac{48T^2}{\delta}}
 +
 \pOptimism\sqrt{\frac{8 H^3}{\maxNoise^2}}
 \sum_{t=1}^{T}
\norm{M_t - M_{t-1}}_F
\\
&
\le
65 \pOptimism \RM \Bbound \kappa \gamma^{-2} \sqrt{T (\dx + \du) (\dx^2 \kappa^2 + \du \Bbound^2) \log^3 \frac{48T^2}{\delta}}
\\
&
\qquad  +
1550\pOptimism \maxNoise^{-1}
  \RM (\dx + \du)  \sqrt{T H^5 \log (6\dmodel^2) \log^3\frac{48 T^2}{\delta}}
\\
&
\le
1560 \pOptimism \RM \Bbound \kappa \gamma^{-2} \sqrt{T H^5 (\dx + \du) (\dx^2 \kappa^2 + \du \Bbound^2)\log (6\dmodel^2)  \log^3 \frac{48T^2}{\delta}}
\\
&
\le
32760 \maxNoise \RM^2 \Bbound^2 \kappa^3 \gamma^{-5} 
     (\dx^2 \kappa^2 + \du \Bbound^2)     \sqrt{T H^5 (\dx + \du)^3 \log (6\dmodel^2)  \log^7 \frac{48T^2}{\delta}}
\\
&
\le
32760 \maxNoise \RM^2 \Bbound^2 \kappa^3 \gamma^{-8} 
     (\dx^2 \kappa^2 + \du \Bbound^2) \log^6 \brk*{ \frac{48T^2}{\delta}}     \sqrt{T \dx (\dx + \du)^3 \log (6\dmodel^2)  }
.
\end{align*} 
Next, we plug in $\pOptimism$ into $\ocoR{3}$ to get that

\begin{align*}
    \ocoR{3}
    &
    \le
    4000
    \pOptimism (\dx+\du) 
    \sqrt{T H^3 \dx \log (6\dmodel^2) \log^3 \frac{48T^2}{\delta}}
    \\
    &
    \le
    84000
    \maxNoise \RM \Bbound \kappa^2 \gamma^{-5} (\dx + \du)^2
    \log^5 \brk*{\frac{48T^2}{\delta}}
    \sqrt{T  \dx (\dx^2 \kappa^2 + \du \Bbound^2) \log (6\dmodel^2) }
    \\
    &
    \le
    10500
    \maxNoise \RM^2 \Bbound^2 \kappa^3 \gamma^{-8} (\dx^2 \kappa^2 + \du \Bbound^2)
    \log^6 \brk*{\frac{48T^2}{\delta}}
    \sqrt{T \dx (\dx + \du)^3 \log (6\dmodel^2) }
    .
\end{align*}
Combining the above, we conclude  that
\begin{align*}
    &
    \mathrm{Regret}_T(\pi)
    \le
     43261 \maxNoise \RM^2 \Bbound^2 \kappa^3 \gamma^{-8} 
     (\dx^2 \kappa^2 + \du \Bbound^2) \log^6 \brk*{ \frac{48T^2}{\delta}}     \sqrt{T \dx (\dx + \du)^3 \log (6\dmodel^2)  }
    .
    \qedhere
\end{align*}
\end{proof}

\section{\cref{thm:ocoRegret}: Deferred details}
\label{sec:ocoAdditionalDetails}

Here we complete the deferred details in the proof of \cref{thm:ocoRegret}.

We start by stating the following high-probability error bound for least squares estimation, that bounds the error of our estimates $\ocoEstModel_t$ of $\ocoTrueModel$, and as such also satisfies the condition of \cref{lemma:ocoOptimismBound}.
\begin{lemma}[\cite{abbasi2011regret}] 
\label{lemma:ocoParameterEst}
Let $\Delta_t = \ocoTrueModel - \ocoEstModel_t$, and suppose that $\norm{\at}^2 \le \lambda = \ocoDiam^2$, $T \ge \dout$. 
With probability at least $1 - \delta$, we have for all $t \ge 1$
\begin{align*}
    \norm{\Delta_t}_{V_t}^2
    \le
    \tr{\Delta_t\tran V_t \Delta_t}
    \le
    8 \ocoMaxNoise^2 \dout^2 \log \frac{T}{\delta}
    +
    2\ocoDiam^2 \ocoModelDiam^2
    \le
    \frac{\pOCOoptimism^2}{\din}
    .
\end{align*}
\end{lemma}
As an immediate corollary, when \cref{lemma:ocoParameterEst} holds we also get that
\begin{align}
\label{eq:ocoEstDiam}
    \norm{\ocoEstModel_t}
    \le
    \norm{\ocoEstModel_t}_F
    \le
    \lambda^{-1/2} \norm{\Delta_t}_{V_t}
    +
    \norm{\ocoTrueModel}_F
    \le
    \pOCOoptimism (\lambda \din)^{-1/2} + \ocoModelDiam
\end{align}
Next, is a well-known bound on harmonic sums \citep[see, e.g.,][]{cohen2019learning}. This is used to show that the optimistic and true losses are close on the realized predictions (proof in \cref{sec:technicalProofs}).
\begin{lemma}
    \label{lemma:harmonicBound}
    Let $\at \in \RR[\din]$ be a sequence such that $\norm{\at}^2 \le \lambda$, and define $V_t = \lambda I + \sum_{s=1}^{t-1} \at[s] \at[s]\tran$. Then
    $
        \sum_{t=1}^{T} \at\tran V_t^{-1} \at
        \le
        5 \din \log T
        .
    $
\end{lemma}
\begin{proof}%
Notice that $\at\tran V_t^{-1} \at \le \norm{\at}^2 / \lambda^2 \le 1$, and so by \cite[][Lemma 26]{cohen2019learning} we get that $\at V_t^{-1} \at \le \log \brk*{\det(V_{t+1}) / \det(V_t)}$. We conclude that
\begin{align*}
    \sum_{t=1}^{T} \at\tran V^{-1} \at
    &\le
    2 \sum_{t=1}^{T} \at\tran V_t^{-1} \at
    \\
    &
    \le
    4 \sum_{t=1}^{T} \log \brk*{\det(V_{t+1}) / \det(V_t)}
    \\
    \tag{telescoping sum}
    &
    =
    4 \log \brk*{\det(V_{T+1}) / \det(V)}
    \\
    &
    =
    4 \log \det(V^{-1/2} V_{T+1} V^{-1/2})
    \\
    \tag{$\det(A) \le \norm{A}^d$}
    &
    \le
    4 \din \log \norm{V^{-1/2} V_{T+1} V^{-1/2}}
    \\
    \tag{triangle inequality}
    &
    \le
    4 \din \log \brk*{1 + \frac{1}{\lambda^2}\sum_{s=1}^{T} \norm{\at[s]}^2}
    \\
    &
    \le
    4 \din \log (T+1)
    \\
    \tag{$T \ge 4$}
    &
    \le
    5 \din \log T
    .
\end{align*}
\end{proof}

Next, the following lemma bounds the number of epochs.
\begin{lemma}
\label{lemma:ocoNEpochs}
We have that $\nEpochs \le 2 \din \log T$.
\end{lemma}
\begin{proof}
    The algorithm ensures that 
    \begin{align*}
    \det(V_T)
    \ge
    \det(V_{\tau_{\nEpochs}}) 
    \ge
    2 \det(V_{\tau_{\nEpochs-1}})
    \ldots
    \ge
    2^{\nEpochs-1} \det{V_1}
    ,
    \end{align*}
    and changing sides, and taking the logarithm we conclude that
    \begin{align*}
    \nEpochs
    &
    \le
    1
    +
    \log \brk*{\det(V_{T}) / \det(V)}
    \\
    &
    =
    1
    +
    \log \det(V^{-1/2} V_{T+1} V^{-1/2})
    \\
    \tag{$\det(A) \le \norm{A}^d$}
    &
    \le
    1
    +
    \din \log \norm{V^{-1/2} V_{T} V^{-1/2}}
    \\
    \tag{triangle inequality}
    &
    \le
    1
    +
    \din \log \brk*{1 + \frac{1}{\lambda}\sum_{t=1}^{T-1} \norm{\at}^2}
    \\
    &
    \le
    1
    +
    \din \log T
    \\
    \tag{$T \ge 3$}
    &
    \le
    2 \din \log T
    ,
\end{align*}
where the second to last inequality holds since $\norm{\at}^2 \le \ocoDiam^2 = \lambda$.
\end{proof}

\section{Technical Lemmas and Proofs}
\label{sec:technicalProofs}

\subsection{OCO Results} \label{sec:oco-results}
Let $\brk[c]{f_t}_{t \ge 1}$ be a sequence of functions and $\ocoSet \in \RR[d]$ be a convex set. The following lemma states the regret guarantee of the Online Gradient Descent (OGD), \cite{zinkevich2003online} update rule, given by:
\begin{align*}
	x_{t+1}
	=
	\Pi_{\ocoSet}\brk[s]{x_t - \eta \nabla f_t(x_t)}
	,
\end{align*}
where $\Pi_{\ocoSet}$ is the $\ell_2$ projection onto $\ocoSet$, and $x_0 \in \ocoSet$ is be chosen arbitrarily.

\begin{lemma}
\label{lemma:OGD}
    Running OGD with $\eta = R / (\bar{G} \sqrt{T})$ on a decision set $S$ with diameter $R$, and $G-$Lipschitz convex loss functions $f_t$ gives
    \begin{align*}
        \sum_{t=1}^{T} f_t(x_t) - f_t(x)
        \le
        \frac12 R \brk{\bar{G} + G^2 \bar{G}^{-1}} \sqrt{T}
    \end{align*}
    for all $T \ge 1$, $x \in S$.
\end{lemma}
\begin{proof}
We use a standard result for OGD \cite{zinkevich2003online} to get that for all $x \in S$
\[
    \sum_{t=1}^{T} \brk[s]{f_t(x_t) - f_t(x)}
    \le
    \frac{R^2}{2\eta} + \frac12 \eta G^2 T
    =
    \frac12 R \brk{\bar{G} + G^2 \bar{G}^{-1}} \sqrt{T}
    .
    \qedhere
\]
\end{proof}
Let $\brk[c]{\lt}_{t \ge 1}$ be a sequence of loss vectors in $\RR[d]$ with $\lti$ the $i-$th coordinate of $\lt$. The following lemma gives a high probability regret guarantee for the Hedge algorithm (see e.g.,\cite{chernov2010prediction}), also known as Multiplicative Weights (MW), which draws $i_t \sim \pt$ where:
\begin{align*}
	\pti[t+1,i]
	\propto
	\pti
	e^{-\eta \lti}
	\propto
	e^{- \eta \sum_{s=1}^{t} \lti[s,i]}
	,
\end{align*}
and $\pt[1]$ is uniform.
\begin{lemma}
    \label{lemma:HedgeWHP}
    Suppose that we play Hedge over loss vectors $\ell_t \in \RR[d]$ chosen by an oblivious adversary, and that satisfy $\norm{\ell_t}_\infty \le C$. If $\eta = \sqrt{\log (d) / 4(T \bar{C}^2)}$ then with probability at least $1 - \delta$
    \begin{align*}
        \sum_{t=1}^{T} \ell_{t,i(t)} - \ell_{t,i^*}
        \le
        \brk*{\bar{C} + \bar{C}^{-1} C^2}
        \sqrt{6 T \log \frac{d}{\delta}}
        .
    \end{align*}
\end{lemma}
\begin{proof}
    Let $\mathcal{F}_t$ be the filtration defined by all random variables up to time $t$, not including the randomized choice of expert. Then we have that
    $
        \EE\brk[s]{\ell_{t,i(t)} \mid \mathcal{F}_t}
        =
        \sum_{i=1}^{2d} \pti \lti
        .
    $
    Moreover, $Z_t = \ell_{t, i(t)} - \EE\brk[s]{\ell_{t,i(t)} \mid \mathcal{F}_t}$ is a martingale difference sequence with $\abs{Z_t} \le 2C$. We thus invoke the Azuma–Hoeffding inequality to get with probability at least $1 - \delta$
    \begin{align*}
        \sum_{t=1}^{T} \ell_{t,i(t)} - \ell_{t,i^*}
        &
        =
        \sum_{t=1}^{T} \ell_{t,i(t)} - \EE\brk[s]*{\ell_{t,i(t)} \big| \mathcal{F}_t}
        +
        \sum_{t=1}^{T} \sum_{i=1}^{2d} \pti \lti - \ell_{t,i^*}
        \\
        &
        \le
        \sqrt{8 T C^2 \log \frac{1}{\delta}}
        +
        \sum_{t=1}^{T} \sum_{i=1}^{2d} \pti \lti - \ell_{t,i^*}
        \tag{Azuma-Hoeffding}
        \\
        &
        \le
        \sqrt{8 T C^2 \log \frac{1}{\delta}}
        +
        \frac{\log d}{\eta}
        +
        4 \eta C^2 T
        \tag{Hedge regret bound \cite[e.g.,][]{chernov2010prediction}}
        \\
        &
        =
        2 C \sqrt{2 T \log \frac{1}{\delta}}
        +
        2 \brk*{\bar{C} + \bar{C}^{-1} C^2}\sqrt{T \log d}
        \\
        &
        \le
        \brk*{\bar{C} + \bar{C}^{-1} C^2}
        \brk*{\sqrt{\log \frac{1}{\delta}}
        +
        \sqrt{2\log d}} \sqrt{2T}
        \tag{$2 C \le \brk*{\bar{C} + \bar{C}^{-1} C^2}$ for all $C, \bar{C} > 0$}
        \\
        &
        \le
        \brk*{\bar{C} + \bar{C}^{-1} C^2}
        \sqrt{6 T \log \frac{d}{\delta}}
        .
        \tag{$\sqrt{2x} + \sqrt{y} \le \sqrt{3(x+y)}$ by AM-GM inequality}
    \end{align*}
\end{proof}

The following lemma states the guarantees for the $\text{BFPL}_\delta^\star$ algorithm\cite{altschuler2018online}, which is essentially a batched version of Follow the Perturbed Leader (FPL) \cite{kalai2005efficient}, which restarts with fresh randomness every time that a set number of leader switches occur.

\begin{lemma}[\citet{altschuler2018online}]
\label{lemma:BFPL}
    Suppose that we run $\text{BFPL}_\delta^\star$ with losses bounded in $[0,1]^n$. Then with probability at least $1 - \delta$
    \begin{align*}
        \sum_{t=1}^{T} \ell_{t,i(t)} - \ell_{t,i^*}
        \le
        150 \sqrt{T \log n \log \frac{2}{\delta}}
        \qquad
        \text{and}
        \qquad
        \#\brk[c]{\text{switches}}
        \le
        135 \sqrt{T \log n \log \frac{2}{\delta}}
    \end{align*}
\end{lemma}

The following result describes the guarantees of our meta-algorithm for computationally-efficient regret minimization of a particular non-convex structure.

\begin{lemma*}[restatement of \cref{lemma:metaExpertsRegret}]
    Let $f_t(M, \kt[], \vSign)$ be a sequence of oblivious loss functions that are convex and $G$ Lipschitz in $M$, and have a convex decision set $S$ with diameter $2R$. Let 
    $
    f_t(M) = \min_{\kt[] \in [d], \vSign \in \brk[c]{-1,1}} f_t(M, \kt[], \vSign)
    $
    and
    consider the update rule that at time $t$:
    \begin{enumerate}
        \item define loss vector $\ell_t$ such that
                $
                \brk{\ell_t}_{\kt[], \vSign}
                =
                f_t(M_t(\kt[], \vSign); \kt[], \vSign) / (2 G R + C)
                $
                \item
                update experts: 
                $
                    M_{t+1}(\kt[], \vSign)
                    =
                    \Pi_{\mathcal{M}}\brk[s]*{
                    M_t(\kt[], \vSign)
                    -
                    \eta \nabla_{M} f_t(M_t(\kt[], \vSign); \kt[], \vSign)
                    }
                $
                \item update prediction:
                $
                    (\kt[t+1], \vSign_{t+1})
                    =
                    \text{BFPL}_{\delta}^*(\ell_t)
                $
                and set 
                $
                M_{t+1}
                =
                M_{t+1}(\kt[t+1], \vSign_{t+1})
                $
    \end{enumerate}
    where $\eta = 2R / \bar{G} \sqrt{T}$ and $C \ge 0$. 
    Suppose that:
    \begin{enumerate}[label=\roman*]
        \item 
        \label{item:a1}
        $
        f_t(M; \kt[], \vSign)
        \le
        f_t(M)
        +
        C
        $
        for all $M \in S, \kt[] \in [d], \vSign \in \brk[c]{\pm 1}$;
        
        \item There exist $\kt[](M), \vSign(M)$ independent of $t$ such that $f_t(M) = f_t(M, \kt[](M), \vSign(M))$.
    \end{enumerate}
     Then with probability at least $1 - \delta$ we have that for all $\tau \le T$
    \begin{align*}
        \sum_{t=1}^{\tau} f_t(M_t) - f_t(M)
        &
        \le
        {151} \brk[s]{(\bar{G} + G^2 \bar{G}^{-1}) R + C} \sqrt{T \log (2d) \log \frac{2T}{\delta}}
        \\
        \sum_{t=1}^{\tau}
        \norm{M_t - M_{t-1}}
        &\le
        (270 + 2 G \bar{G}^{-1})
        R \sqrt{T \log (2d) \log \frac{2 T}{\delta}}
        .
    \end{align*}
\end{lemma*}
\begin{proof}
$M_t(\kt[], \vSign)$ are exactly the iterates of running Online Gradient Descent (OGD) on the functions $f_t(\cdot; \kt[], \vSign)$, which are convex and $G$-Lipschitz. A classic result (see \cref{lemma:OGD}) then gives us that for all $M \in \mathcal{M}$ and $\tau \le T$
\begin{align*}
    \sum_{t = 1}^{\tau}
    f_t(M_t(\kt[], \vSign); \kt[], \vSign)
    -
    f_t(M; \kt[], \vSign)
    \le
    (\bar{G} + G^2 \bar{G}^{-1}) R \sqrt{T}
    ,
\end{align*}
Next, we verify that $\ell_t$ satisfy the conditions for BFPL (\cref{lemma:BFPL}). First, since OGD is deterministic, the iterates $M_t(\kt[], \vSign)$ and thus the losses $\ell_t$ are deterministic functions of the loss functions $f_1, \ldots, f_{t}$. Since the latter are oblivious, so are the loss vectors $\ell_t$. Next, notice that BFPL is invariant to a constant shift in the loss vectors. The procedure described in the lemma is equivalent to using the losses
\begin{align*}
    \brk{\ell_t}_{\kt[], \vSign}
    =
    \frac{
    f_t(M_t(\kt[], \vSign); \kt[], \vSign)
    -
    f_t(\bar{M})
    }
     {2 G R + C}
     ,
\end{align*}
where $\bar{M} \in \argmin_{M \in S} f_t(M)$. We show that $\ell_t \in [0,1]$. By definition of $f_t$ and $\bar{M}$ we have
\begin{align*}
    \brk{\ell_t}_{\kt[], \vSign}
    =
    \frac{
    f_t(M_t(\kt[], \vSign); \kt[], \vSign)
    -
    f_t(\bar{M})
    }
     {2 G R + C}
     \ge
     \frac{
    f_t(M_t(\kt[], \vSign))
    -
    f_t(\bar{M})
    }
     {2 G R + C}
     \ge
     0
     .
\end{align*}
On the other hand, using that $f_t$ is $G$ Lipschitz and the assumption in \ref{item:a1} we have
\begin{align*}
    \brk{\ell_t}_{\kt[], \vSign}
    =
    \frac{
    f_t(M_t(\kt[], \vSign); \kt[], \vSign)
    -
    f_t(\bar{M})
    }
     {2 G R + C}
     \le
     \frac{
    2GR + f_t(\bar{M}; \kt[], \vSign)
    -
    f_t(\bar{M})
    }
     {2 G R + C}
     \le
     \frac{
    2GR + C
    }
     {2 G R + C}
     \le 1
     .
\end{align*}
We thus use \cref{lemma:BFPL} with $\delta / T$, and a union bound to get that with probability at least $1 - \delta$
\begin{align*}
    \sum_{t = 1}^{\tau}
    f_t(M_t(\kt, \vSign_t); \kt, \vSign_t)
    -
    f_t(M_t(\kt[], \vSign); \kt[], \vSign)
    &
    \le
    150(2 G R + C)\sqrt{T \log (2d) \log \frac{2T}{\delta}}
    \\
    \# \brk[c]{\text{switches}}_T
    &
    \le
    135 \sqrt{T \log (2d) \log \frac{2 T}{\delta}}
    ,
\end{align*}
for all $\kt[] \in \brk[s]{d}, \vSign \in \brk[c]{\pm 1}$, and $\tau \le T$.
Now, for ease of notation denote $\kt[]^*, \vSign^* = \kt[](M), \vSign(M)$ where these are taken from the lemma's assumptions. Then we conclude that with probability at least $1-\delta$ we have that for all $M \in \mathcal{M}$
\begin{align*}
    \tag{$f_t(\cdot) \le f_t(\cdot; \kt[], \vSign)$}
    \sum_{t=1}^{\tau}
    f_t(M_t) - f_t(M)
    &
    \le
    \sum_{t=1}^{\tau}
    f_t(M_t; \kt, \vSign_t) - f_t(M)
    \\
    &
    =
    \sum_{t=1}^{\tau}
    f_t(M_t; \kt, \vSign_t) - f_t(M; \kt[]^*, \vSign^*)
    \\
    &
    =
    \sum_{t=1}^{\tau}
    \brk*{
    f_t(M_t(\kt,\vSign_t); \kt, \vSign_t)
    -
    f_t(M_t(\kt[]^*, \vSign^*); \kt[]^*, \vSign^*)
    }
    \\
    &
    +
    \sum_{t=1}^{\tau}
    \brk*{
    f_t(M_t(\kt[]^*, \vSign^*); \kt[]^*, \vSign^*) - f_t(M; \kt[]^*, \vSign^*)
    }
    \\
    \tag{$G \le \frac12 (\bar{G} + G^2 \bar{G}^{-1})$}
    &
    \le
    {151} \brk[s]{(\bar{G} + G^2 \bar{G}^{-1}) R + C} \sqrt{T \log (2d) \log \frac{2T}{\delta}}
    .
\end{align*}
Next, notice that if there is no expert change (switch) then 
$
\norm{M_t - M_{t-1}}
\le
G \eta
=
2 G R / (\bar{G} \sqrt{T})
,
$
and otherwise, if there is a switch, then
$
\norm{M_t - M_{t-1}}
\le
2 R
.
$
We thus get that under the above event
\[
    \sum_{t=1}^{\tau}
    \norm{M_t - M_{t-1}}
    \le
    (270 + 2 G \bar{G}^{-1})
    R \sqrt{T \log (2d) \log \frac{2 T}{\delta}}
    .
    \qedhere
\]
\end{proof}

\subsection{Concentration of Measure} \label{sec:measure-concentration-proofs}
First, we give the following Bernstein type tail bound \cite[see e.g.,][Lemma D.4]{rosenberg2020near}.
\begin{lemma}
\label{lemma:multiplicative-concentration}
Let $\brk[c]{X_t}_{t \ge 1}$
be a sequence of random variables with expectation adapted to a filtration
$\mathcal{F}_t$.
Suppose that $0 \le X_t \le 1$ almost surely. Then with probability at least $1-\delta$
\begin{align*}
    \sum_{t=1}^{T} \EE \brk[s]{X_t \mid \mathcal{F}_{t-1}}
    \le
    2 \sum_{t=1}^{T} X_t
    +
    4 \log \frac{2}{\delta}
\end{align*}
\end{lemma}
\begin{lemma*}[restatement of \cref{lemma:blockConcentration}]
Let $X_t$ be a sequence of random variables adapted to a filtration $\mathcal{F}_t$. Then we have the following
\begin{itemize}[leftmargin=*]
    \item 
    If $\abs{X_t - \EE\brk[s]{X_t \mid \mathcal{F}_{t-2H}}} \le C_t$ where $C_t \ge 0$ are $\mathcal{F}_{t-2H}$ measurable then with probability at least $1 - \delta$
    \begin{align*}
        \sum_{t=1}^{T} \brk*{X_t - \EE\brk[s]{X_t \mid \mathcal{F}_{t-2H}}}
        \le
        2 \sqrt{\sum_{t=1}^{T} \brk{C_t^2} H \log \frac{T}{\delta}}
        ;
    \end{align*}
    \item If $0 \le X_t \le 1$ then with probability at least $1 - \delta$
    \begin{align*}
        \sum_{t=1}^{T} \EE\brk[s]{X_t \mid \mathcal{F}_{t-2H}}
        \le
        2 \sum_{t=1}^{T} (X_t)
        +
        8 H \log \frac{2 T}{\delta}
        .
    \end{align*}
\end{itemize}
\end{lemma*}
\begin{proof}
For $h = 1, \ldots, 2H$, and $k \ge 0$ define the time indices
\begin{align*}
    t_{k}^{(h)}
    =
    h + 2H k
    =
    t_{k-1}^{(h)} + 2H
    ,
\end{align*}
and the filtration 
$
    \bar{\mathcal{F}}_k^{(h)} 
    =
    \mathcal{F}_{t_{k}^{(h)}}
    .
$
Denoting $X_k^{(h)} = X_{t_k^{(h)}}$ we have that $X_k^{(h)}$ is $\bar{\mathcal{F}}_k^{(h)}$ measurable and that
\begin{align*}
\abs{X_k^{(h)} - \EE\brk[s]{X_k^{(h)} \mid \bar{\mathcal{F}}_{k-1}^{(h)}}}
=
\abs{X_{t_k^{(h)}} - \EE\brk[s]{X_{t_k^{(h)}} \mid \mathcal{F}_{t_k^{(h)}-2H}}}
\le 
C_{t_k^{(h)}}
.
\end{align*}
We can thus invoke the Azuma–Hoeffding inequality with a union bound over all $h = 1, \ldots, 2H$
to get that with probability at least $1 - \delta$
\begin{align*}
    \sum_{k=1}^{K(h)} 
    \brk*{X_k^{(h)} - \EE\brk[s]{X_k^{(h)} \mid \bar{\mathcal{F}}_{k-1}^{(h)}}}
    \le
    \sqrt{2 \sum_{k=1}^{K(h)} (C_{t_k^{(h)}}^2) \log \frac{2H}{\delta}}
    \le
    \tag{$2H \le T$}
    \sqrt{2 \sum_{k=1}^{K(h)} (C_{t_k^{(h)}}^2) \log \frac{T}{\delta}}
    ,
\end{align*}
where we denoted $K(h) = \floor{(T-h)/2H}$. Now, notice that \begin{align*}
\brk[c]{t_k : k = 1, \ldots, K(h), h = 1, \ldots, 2H} 
=
\brk[c]{1, \ldots, T}
.
\end{align*}
We conclude that
\begin{align*}
    \sum_{t=1}^{T} \brk*{X_t - \EE\brk[s]{X_t \mid \mathcal{F}_{t-2H}}}
    &
    =
    \sum_{h=1}^{2H}
    \sum_{k=1}^{K(h)}
    \brk*{X_k^{(h)} - \EE\brk[s]{X_k^{(h)} \mid \bar{\mathcal{F}}_{k-1}^{(h)}}}
    \\
    &
    \le
    \sum_{h=1}^{2H}
    \sqrt{2 \sum_{k=1}^{K(h)} (C_{t_k^{(h)}}^2) \log \frac{T}{\delta}}
    \\
    &
    \le
    2 \sqrt{H \sum_{h=1}^{2H} \sum_{k=1}^{K(h)} (C_{t_k^{(h)}}^2) \log \frac{T}{\delta}}
    \\
    &
    =
    2 \sqrt{\sum_{t=1}^{T} (C_{t}^2) H \log \frac{T}{\delta}}
    .
\end{align*}
Moving on to the second claim of the lemma, $X_k^{(h)}$ satisfies \cref{lemma:multiplicative-concentration}, and we thus invoke it with $\delta / 2H$ for all  $h = 1, \ldots, 2H$. Taking a union bound, we get that with probability at least $1 - \delta$ for all $h = 1, \ldots, 2H$
\begin{align*}
    \sum_{k=1}^{K(h)} 
    \EE\brk[s]{X_k^{(h)} \mid \bar{\mathcal{F}}_{k-1}^{(h)}}
    \le
    2 \sum_{k=1}^{K(h)} (X_k^{(h)})
    +
    4 \log \frac{4H}{\delta}
    \le
    2 \sum_{k=1}^{K(h)} (X_k^{(h)})
    +
    4 \log \frac{2T}{\delta}
    \tag{$2H \le T$}
    ,
\end{align*}
and thus finally
\begin{align*}
    \sum_{t=1}^{T} \EE\brk[s]{X_t \mid \mathcal{F}_{t-2H}}
    &
    =
    \sum_{h=1}^{2H}
    \sum_{k=1}^{K(h)}
    \EE\brk[s]{X_k^{(h)} \mid \bar{\mathcal{F}}_{k-1}^{(h)}}
    \\
    &
    \le
    \sum_{h=1}^{2H}
    \brk[s]*{
    2 \sum_{k=1}^{K(h)} \brk{X_k^{(h)}}
    +
    4 \log \frac{2T}{\delta}
    }
    \\
    &
    \le
    2 \sum_{t=1}^{T} (X_t)
    +
    8 H \log \frac{2T}{\delta}
    .
    \qedhere
\end{align*}
\end{proof}

\subsection{Surrogate functions}

\begin{lemma*}[restatement of \cref{lemma:oToReal}]
For all $\ww$ such that $\norm{w_t} \le \maxNoise$, $M \in \mathcal{M}$, and $t \le T$, we have:
\begin{enumerate}
    \item
    $
    \norm{\obs_t}
    \le
    \sqrt{2} \maxNoise \RM H
    ;
    $
    \item
    $
        \norm{\obs_{t-1} - \obs_{t-1}(M_t; \ww)}^2
        \le
        2\RM^2 H \brk[s]*{
         \sum_{h=1}^{2H}\norm{w_{t-h} - \hat{w}_{t-h}}^2
         +
         \sum_{h=1}^{H} \norm{M_{t-h} - M_t}_F^2
         }
         .
    $
\end{enumerate}
\end{lemma*}

\begin{proof}
Recall:
\begin{align*}
    \obs_{t-1}
    &=
    (u_{t-H}(M_{t-H}; \wwhat)\tran
    ,
    \ldots
    u_{t-1}(M_{t-1}; \wwhat)\tran
    ,
    \hat{w}_{t-H}\tran
    ,
    \ldots
    \hat{w}_{t-2}\tran
    )\tran
    , 
    \quad \text{and}
    \\
    \obs_{t-1}(M_t; \ww)
    &=
    (u_{t-H}(M_{t}; \ww)\tran
    ,
    \ldots
    u_{t-1}(M_{t}; \ww)\tran
    ,
    {w}_{t-H}\tran
    ,
    \ldots
    {w}_{t-2}\tran
    )\tran
    .
\end{align*}
First, $\norm{u_t(M; \ww)} \le \maxNoise \RM \sqrt{H}$ by \cref{lemma:technicalParameters}. Thus
\begin{align*}
    \norm{\obs_{t-1}}
    \le
    \sqrt{\sum_{h=1}^{H} \brk[s]*{\norm{u_{t-H}(M_{t-h})}^2 + \norm{w_{t-h}}^2}}
    \le
    \sqrt{H\brk[s]{\maxNoise^2 \RM^2 H + \maxNoise^2}}
    \le
    \sqrt{2} \maxNoise \RM H
    ,
\end{align*}
concluding the first part of the lemma.

For the second part, we begin by using \cref{lemma:technicalParameters} to get 
\begin{align*}
    &
    \norm{u_{t-h}(M_{t-h}; \wwhat) - u_{t-h}(M_t; \ww)}^2
    \\
    &
    \qquad \le
    2\norm{u_{t-h}(M_{t-h}; \wwhat) - u_{t-h}(M_t; \wwhat)}^2
    +
    2\norm{u_{t-h}(M_{t}; \wwhat) - u_{t-h}(M_t; \ww)}^2
    \\
    &
    \qquad \le
    2\RM^2 \norm{w_{t-(h+H):t-(h+1)} - \hat{w}_{t-(h+H):t-(h+1)}}^2
    +
    2\RM^2 H \norm{M_{t-h} - M_t}_F^2
    \\
    &
    \qquad \le
    2\RM^2 \sum_{h'=1}^{H}\norm{w_{t-(h+h')} - \hat{w}_{t-(h+h')}}^2
    +
     2\RM^2 H \norm{M_{t-h} - M_t}_F^2
    .
\end{align*}
We thus get
\begin{align*}
    \norm{\obs_{t-1} - \obs_{t-1}(M_t; \ww)}^2
    &
    =
    \sum_{h=1}^{H}
    \norm{u_{t-h}(M_{t-h}; \wwhat) - u_{t-h}(M_t; \ww)}^2
    +
    \sum_{h=2}^{H} \norm{w_{t-h}-\hat{w}_{t-h}}^2
    \\
    &
    \le
    2\RM^2 \sum_{h=1}^{H} \sum_{h'=0}^{H}\norm{w_{t-(h+h')} - \hat{w}_{t-(h+h')}}^2
    +
     2\RM^2 H \sum_{h=1}^{H} \norm{M_{t-h} - M_t}_F^2
     \\
     &
     \le
     2\RM^2 H \brk[s]*{
     \sum_{h=1}^{2H}\norm{w_{t-h} - \hat{w}_{t-h}}^2
     +
     \sum_{h=1}^{H} \norm{M_{t-h} - M_t}_F^2
     }
     .
     \qedhere
\end{align*}
\end{proof}

\begin{lemma*}[restatement of \cref{lemma:fbarProperties}]
    Define the functions
    \begin{align*}
        \maxF(\model) = 5 \RM \maxNoise H \max\brk[c]{
        \norm{(\model \; I)}_F
        ,
        \kappa \gamma^{-1} \Bbound
        },
        \qquad
        \pLipF(\model)
        =
        \sqrt{2} \maxNoise H \norm{\model}_F
        +
        {\pOptimism}/\brk{\RM \sqrt{2H}}
        .
    \end{align*}
    For any $\ww, \ww'$ with $\norm{w_t}, \norm{w'_t} \le \maxNoise$ and $M, M'$ with $\norm{M}_F, \norm{M'}_F \le \RM$,
    we have:
    \begin{enumerate}
        \item 
        $
        \abs{f_t(M; \ww) - f_t(M; \ww')}
        \le
        \maxF(\model)
        ;
        $
        
        \item
        $
        \abs{
        \bar{f}_t(M; \model, V, \ww) 
        -
        \bar{f}_t(M; \model, V, \ww')
        }
        \le
        \maxF(\model)
        $;
    \end{enumerate}
    Additionally, if $V \succeq \pRegTheta I$ then
    \begin{enumerate}[resume]
        \item
        $
        \abs{
        \bar{f}_t(M; \model, V, \ww)
        -
        \bar{f}_t(M'; \model, V, \ww)
        }
        \le
        \pLipF(\model)
        \norm{M - M'}_F
        ;
        $
        \item
        $
        \abs{
        \bar{f}_t(M; \kt[], \vSign, \model, V, \ww)
        -
        \bar{f}_t(M'; \kt[], \vSign, \model, V, \ww)
        }
        \le
        \pLipF(\model)
        \norm{M - M'}_F
        ;
        $
        \item
        $
        \bar{f}_t(M; \kt[], \vSign, \model, V, \ww)
        \le
        \bar{f}_t(M; \model, V, \ww)
        +
        \pOptimism \sqrt{2/H}  \brk[s]{(1  + \RM^{-1} \sqrt{\dx}}
        $
        .
    \end{enumerate}
    Moreover, if
    $
    \norm{\brk{\model \; I}}_F
    \le
    17 \Bbound \kappa^2 
    \sqrt{\gamma^{-3} (\dx + \du) (\dx^2 \kappa^2 + \du \Bbound^2) \log \frac{24T^2}{\delta}}
    ,
    $
    then:
    \begin{align*}
        \maxF(\model)
        \le
        5 \pOptimism / (H \sqrt{\dx(\dx+\du)})
        ,
        \quad
        \text{and}
        \;\;
        \pLipF(\model)
        \le
        {\pOptimism \sqrt{2}}/\brk{\RM \sqrt{H}}
        .
    \end{align*}
\end{lemma*}

\begin{proof}
First, recalling the definition of $x_t(M; \model, \ww)$ in \cref{eq:trunc-x-def}, we have
\begin{align*}
    x_t(M; \model, \ww)
    =
    \model_\star \obs_{t-1}(M; \ww) + w_{t-1}
    =
    \sum_{h=1}^{H} \Astar^{h-1}\brk[s]*{
    \Bstar u_{t-h}(M; \ww) + w_{t-h}}
    .
\end{align*}
Also noticing that $u_t(M;\ww)$ is $\RM$ Lipschitz in $w_{t-H:t-1}$ (\cref{lemma:technicalParameters}), we get
\begin{align*}
    &
    \norm{
    x_t(M; \model_\star, \ww)
    -
    x_t(M; \model_\star, \ww')
    }
    =
    \norm*{
    \sum_{h=1}^{H} \Astar^{h-1}\brk[s]*{
    \Bstar (u_{t-h}(M; \ww) - u_{t-h}(M; \ww'))
    +
    (w_{t - h} - w'_{t-h})}
    }
    \\
    &
    \le
    \sum_{h=1}^{H} \kappa (1-\gamma)^{h-1}\brk[s]*{
    \Bbound \norm{u_{t-h}(M; \ww) - u_{t-h}(M; \ww')}
    +
    \norm{w_{t - h} - w'_{t-h}}}
    \\
    &
    \le
    \sum_{h=1}^{H} \kappa (1-\gamma)^{h-1}\brk[s]*{
    \Bbound \RM \norm{w_{t-(h+H):t-(h+1)} - w_{t-(h+H):t-(h+1)}'}
    +
    \norm{w_{t - h} - w'_{t-h}}}
    \\
    \tag{${x} + {y} \le \sqrt{2(x^2 + y^2)}$}
    &
    \le
    \sqrt{2} \kappa \sum_{h=1}^{H}  (1-\gamma)^{h-1}
    \Bbound \RM \norm{w_{t-(h+H):t-h} - w_{t-(h+H):t-h}'}
    \\
    &
    \le
    \sqrt{2} \kappa \gamma^{-1}
    \Bbound \RM \norm{w_{t-2H:t-1} - w_{t-2H:t-1}'}
    ,
\end{align*}
where in the third inequality notice that $\norm{w_{1:t-1}}^2 + \norm{w_t}^2 = \norm{w_{1:t}}^2$. used 

Next, also using the Lipschitz assumption on $c_t$, and that $u_t$ is $\RM-$Lipschitz with respect to $w_{t-H:t-1}$ (\cref{lemma:technicalParameters}) we get that
\begin{align*}
    \abs{f_t(M; \ww) - f_t(M; \ww')}
    &=
    \abs{
    c_t(x_t(M; \model_\star, \ww), u_t(M; \ww))
    -
    c_t(x_t(M; \model_\star, \ww'), u_t(M; \ww'))
    }
    \\
    &
    \le
    \norm{
    (
    x_t(M; \model_\star, \ww) - x_t(M; \model_\star, \ww')
    ,
    u_t(M; \ww) - u_t(M; \ww')
    )
    }
    \\
    &
    \le
    \sqrt{3} \kappa \gamma^{-1}
        \Bbound \RM \norm{w_{t-2H:t-1} - w_{t-2H:t-1}'}
        .
\end{align*}
Moreover, since 
$
\norm{w_{t-2H:t-1} - w'_{t-2H:t-1}}
\le
\maxNoise \sqrt{8H}
$
we also get that
\begin{align*}
    \abs{f_t(M; \ww) - f_t(M; \ww')}
    \le
    5 \kappa \gamma^{-1} \Bbound \RM \maxNoise \sqrt{H}
    .
\end{align*}

Now, also recall
$
x_t(M; \model, \ww)
= 
\model \obs_{t-1}(M; \ww) + w_{t-1}
$,
thus by \cref{lemma:technicalParameters} we have
\begin{alignat*}{2}
    &
    \norm{x_t(M; \model, \ww)}
    &&
    \le
    \sqrt{2} \maxNoise \RM H \norm{(\model \; I)}
    \\
    &
    \norm{\brk*{x_t(M; \model, \ww) \; u_t(M; \ww)}}
    &&
    \le
    \sqrt{3} \maxNoise \RM H \norm{(\model \; I)}
    \\
    &
    \norm{x_t(M; \model, \ww) - x_t(M; \model, \ww')}
    &&
    \le
    \sqrt{2H} \RM \norm{(\model \; I)} \norm{w_{t-2H:t-1} - w'_{t-2H:t-1}}
    .
\end{alignat*}
Then we get
\begin{align*}
    &
    \abs{
    c_t(x_t(M; \model, \ww), u_t(M; \ww))
    -
    c_t(x_t(M; \model, \ww'), u_t(M; \ww'))
    }
    \\
    &
    \le
    \norm{
    (
    x_t(M; \model, \ww) - x_t(M; \model, \ww')
    ,
    u_t(M; \ww) - u_t(M; \ww')
    )
    }
    \\
    &
    \le
    \sqrt{3H} \RM \norm{(\model \; I)} \norm{w_{t-2H:t-1} - w'_{t-2H:t-1}}
    ,
\end{align*}
and since the second term of $\bar{f}_t$ does not depend on $\ww$ we also have
\begin{align*}
    \abs{
    \bar{f}_t(M; \model, V, \ww)
    -
    \bar{f}_t(M; \model, V, \ww')
    }
    \le
    \sqrt{3H} \RM \norm{(\model \; I)} \norm{w_{t-2H:t-1} - w'_{t-2H:t-1}}
    ,
\end{align*}
Moreover, since 
$
\norm{w_{t-2H:t-1} - w'_{t-2H:t-1}}
\le
\maxNoise \sqrt{8H}
$
we also get
\begin{align*}
    \abs{
    \bar{f}_t(M; \model, V, \ww)
    -
    \bar{f}_t(M; \model, V, \ww')
    }
    \le
    5 \RM \maxNoise H \norm{(\model \; I)}
    \le
    \maxF(\model)
    .
\end{align*}
This concludes the first two parts of the proof.

Next, we have that
\begin{align*}
    \norm{u_t(M; \ww) - u_t(M'; \ww)}
    &
    =
    \norm*{\sum_{h=1}^{H} (M^{[h]} - M'^{[h]}) w_{t-h}}
    \\
    &
    \le 
    \maxNoise \sum_{h=1}^{H} \norm*{M^{[h]} - M'^{[h]}}
    \\
    \tag{Cauchy-Schwarz}
    &
    \le
    \maxNoise \sqrt{H} \norm{M - M'}_F
    ,
\end{align*}
thus we get
\begin{align*}
    \norm{\obs_t(M; \ww) - \obs_t(M'; \ww)}
    &
    =
    \sqrt{\sum_{h=1}^{H} \norm{u_{t+1-h}(M; \ww) - u_{t+1-h}(M'; \ww)}^2}
    \\
    &
    \le
    \maxNoise H \norm{M - M'}_F
    ,
\end{align*}
and
\begin{align*}
    \norm{x_t(M; \model, \ww) - x_t(M'; \model, \ww)}
    \le
    \maxNoise H \norm{\model} \norm{M - M'}_F
    .
\end{align*}
We thus have that
\begin{align*}
    &
    \abs{
    c_t(x_t(M; \model, \ww), u_t(M; \ww))
    -
    c_t(x_t(M'; \model, \ww), u_t(M'; \ww))
    }
    \\
    &
    \le
    \norm{
    (
    x_t(M; \model, \ww) - x_t(M'; \model, \ww)
    ,\;
    u_t(M; \ww) - u_t(M'; \ww)
    )
    }
    \\
    &
    \le
    \sqrt{2} \maxNoise H \norm{\model} \norm{M - M'}_F
    .
\end{align*}
Next, denote $\tilde{M} = M - M'$ and notice that
\begin{align*}
    \norm{\obsOp(M) -\obsOp(M')}_F
    &
    =
    \norm*{\;
    \begin{pmatrix}
        \tilde{M}^{[H]} & \tilde{M}^{[H-1]} & \cdots & \tilde{M}^{[1]} \\
        & \tilde{M}^{[H]} & \tilde{M}^{[H-1]} & \cdots &  \tilde{M}^{[1]}  \\
         &  & \ddots & \ddots & &  \ddots &  \\
        &  & & \tilde{M}^{[H]} & \tilde{M}^{[H-1]} & \cdots &  \tilde{M}^{[1]} \\
          &  & &  & 0 &  \\
           &  & &  &  & \ddots \\
          &  & &  & & & 0
    \end{pmatrix} \;
    }_F
    \\
    &
    =
    \sqrt{H} \norm{M - M'}_F
\end{align*}
We thus have
\begin{align*}
    \abs*{
    \norm{V^{-1/2} \obsOp(M) \wCov^{1/2}_{2H-1}}_{\infty}
    -
    \norm{V^{-1/2} \obsOp(M') \wCov^{1/2}_{2H-1}}_{\infty}
    }
    &
    \le
    \norm{
    V^{-1/2}
    \brk{\obsOp(M) - \obsOp(M')}
    \wCov^{1/2}_{2H-1}
    }_{\infty}
    \\
    &
    \le
    \norm{
    V^{-1/2}
    \brk{\obsOp(M) - \obsOp(M')}
    \wCov^{1/2}_{2H-1}
    }_{F}
    \\
    &
    \le
    \norm{V^{-1/2}}
    \norm{\wCov^{1/2}}
    \norm{\obsOp(M) - \obsOp(M')}_F
    \\
    &
    \le
    \pRegTheta^{-1/2}
    \maxNoise
    \sqrt{H}
    \norm{M - M'}_F
    ,
\end{align*}
which yields
\begin{align*}
    \abs{
    \bar{f}_t(M; \model, V, \ww)
    -
    \bar{f}_t(M'; \model, V, \ww)
    }
    &
    \le
    (
    \sqrt{2} \maxNoise H \norm{\model}
    +
    \pOptimism
    \pRegTheta^{-1/2} \maxNoise \sqrt{H}
    )
    \norm{M - M'}_F
    \\
    &
    \le
    \brk[s]{
    \sqrt{2} \maxNoise H \norm{\model}
    +
    {\pOptimism}/\brk{\RM \sqrt{2H}}
    }
    \norm{M - M'}_F
    \\
    &
    \le
    \pLipF(\model) \norm{M-M'}_F
  .
\end{align*}
Identical arguments show that
\begin{align*}
    \abs{
    \bar{f}_t(M; \kt[], \vSign, \model, V, \ww)
    -
    \bar{f}_t(M'; \kt[], \vSign, \model, V, \ww)
    }
    \le
    \pLipF(\model)
    \norm{M - M'}_F   
    .
\end{align*}
Next, we have
\begin{align*}
    \bar{f}_t(M; \kt[], \vSign, \model, V, \ww)
    -
    \bar{f}_t(M; \model, V, \ww)
    &
    =
    \pOptimism\brk*{
    \norm{V^{-1/2} \obsOp(M) \wCov^{1/2}_{2H-1}}_{\infty}
    -
    \vSign \cdot \brk{V^{-1/2} \obsOp(M) \wCov^{1/2}_{2H-1}}_{\kt[]}
    }
    \\
    &
    \le
    2\pOptimism
    \norm{V^{-1/2} \obsOp(M) \wCov^{1/2}_{2H-1}}_{\infty}
    \\
    &
    \le
    2\pOptimism
    \norm{V^{-1/2}}
    \norm{\wCov^{1/2}}
    \norm{\obsOp(M)}_F
    \\
    &
    \le
    2\pOptimism \pRegTheta^{-1/2} W \sqrt{H \RM^2 + H \dx}
    \\
    &
    \le
    \pOptimism  \sqrt{(2  + 2\RM^{-2} \dx)/H}
    \\
    &
    \le
    \pOptimism \sqrt{2/H}  \brk[s]{(1  + \RM^{-1} \sqrt{\dx}}
    .
\end{align*}

Finally, if 
$
    \norm{\brk{\model \; I}}_F
    \le
    17 \Bbound \kappa^2 
    \sqrt{\gamma^{-3} (\dx + \du) (\dx^2 \kappa^2 + \du \Bbound^2) \log \frac{24T^2}{\delta}}
    ,
$
then we have that
\begin{align*}
    \maxF(\model)
    &
    \le
    85 \maxNoise \RM \Bbound \kappa^2 H \sqrt{\gamma^{-3} (\dx + \du) (\dx^2 \kappa^2 + \du \Bbound) \log \frac{24 T^2}{\delta}}
    \le
    5 \pOptimism / (H \sqrt{\dx(\dx+\du)})
    \\
    \pLipF(\model)
    &
    \le
    \frac{\sqrt{2}}{5 \RM} \maxF(\model)
    +
    {\pOptimism}/\brk{\RM \sqrt{2H}}
    \le
    {\pOptimism \sqrt{2}}/\brk{\RM \sqrt{H}}
    ,
\end{align*}
where the last transition assumed that $H \ge 2$.
\end{proof}

\end{document}